\documentclass[nohyperref]{article}

\usepackage{microtype}
\usepackage{graphicx}
\usepackage{subfigure}
\usepackage{booktabs} 

\usepackage{hyperref}

\usepackage[accepted]{icml2022a}

\usepackage{amsmath}
\usepackage{amssymb}
\usepackage{mathtools}
\usepackage{amsthm}

\usepackage[capitalize,noabbrev]{cleveref}

\theoremstyle{plain}
\newtheorem{theorem}{Theorem}[section]

\newtheorem{lemma}[theorem]{Lemma}
\newtheorem{corollary}[theorem]{Corollary}
\theoremstyle{definition}
\newtheorem{definition}[theorem]{Definition}
\newtheorem{assumption}[theorem]{Assumption}
\theoremstyle{remark}
\newtheorem{remark}[theorem]{Remark}

\usepackage[textsize=tiny]{todonotes}

\icmltitlerunning{AI-SARAH: Adaptive and Implicit Stochastic Recursive Gradient Methods}

\usepackage{amsmath,amsfonts,bm}

\def\eqref#1{(\ref{#1})}

\def\1{\bm{1}}

\def\ve{{\bm{e}}}

\DeclareMathAlphabet{\mathsfit}{\encodingdefault}{\sfdefault}{m}{sl}
\SetMathAlphabet{\mathsfit}{bold}{\encodingdefault}{\sfdefault}{bx}{n}

\newcommand{\R}{\mathbb{R}}

\newcommand{\reg}{\lambda}

\DeclareMathOperator*{\argmax}{arg\,max}

\usepackage[utf8]{inputenc} 
\usepackage[T1]{fontenc}    
\usepackage{hyperref}       
\usepackage{url}            
\usepackage{booktabs}       
\usepackage{amsfonts}       
\usepackage{nicefrac}       
\usepackage{microtype}      
\usepackage{xcolor}         
\usepackage{algorithm}
\usepackage{algorithmic}
\usepackage{microtype}
\usepackage{graphicx}
\usepackage{subfigure}
\usepackage{booktabs} 
\usepackage{amsmath}
\usepackage{amsfonts}
\usepackage{amsthm}
\newcommand{\cR}{{\cal R}}
\newcommand{\cC}{{\cal C}}
\usepackage{threeparttable}

\usepackage{mathtools}

\newcommand*{\defeq}{\stackrel{\text{def}}{=}}

\newcommand\tagthis{\addtocounter{equation}{1}\tag{\theequation}}

\usepackage{xargs}                      

\usepackage{url} 

\newcommand{\rchi}{\raisebox{2pt}{$\chi$}}
\usepackage{enumitem}
\newcommand*\zheng[1]{\textcolor{blue}{Zheng: #1}}
\newcommand*\zs[1]{\todo[color=green,inline]{Zheng: #1}}

\newcommand*\nicolas[1]{\todo[color=pink]{\textbf{Nicolas:} #1}}

\newcommand*\scon[1]{\mbox{#1-strongly convex}}

\newcommand*\reglize{\mbox{$\ell^2$-regularized}}
\newcommand*\Exp{\ensuremath{\mathbb{E}}}

\newcommand{\Lew}[0]{{\mathcal{W}}}

\def\EE{\mathbb E}

\newcommand{\n}{[n]}

\title{AI-SARAH: Adaptive and Implicit Stochastic Recursive Gradient Methods}

\author{%
  David S.~Hippocampus\thanks{Use footnote for providing further information
    about author (webpage, alternative address)---\emph{not} for acknowledging
    funding agencies.} \\
  Department of Computer Science\\
  Cranberry-Lemon University\\
  Pittsburgh, PA 15213 \\
  \texttt{hippo@cs.cranberry-lemon.edu} \\
}

\begin{document}
\setlength{\abovedisplayskip}{3pt}
\setlength{\belowdisplayskip}{3pt}

\twocolumn[
\icmltitle{AI-SARAH: Adaptive and Implicit Stochastic Recursive Gradient Methods}



\icmlsetsymbol{equal}{*}

\begin{icmlauthorlist}
\icmlauthor{Zheng Shi}{lehigh,ibm}
\icmlauthor{Abdurakhmon Sadiev}{mbzuai,mipt}
\icmlauthor{Nicolas Loizou}{jh}
\icmlauthor{Peter Richt\'arik}{kaust}
\icmlauthor{Martin Tak\'a\v{c}}{mbzuai}
\end{icmlauthorlist}

\icmlaffiliation{lehigh}{Industrial and Systems Engineering, Lehigh University, Bethlehem, USA}
\icmlaffiliation{ibm}{IBM Corporation, Armonk, USA}
 \icmlaffiliation{mbzuai}{Mohamed bin Zayed University of Artificial Intelligence (MBZUAI),  Abu Dhabi,
United Arab Emirates}
 \icmlaffiliation{mipt}{Moscow Institute of Physics and Technology, Dolgoprudny, Russia}
\icmlaffiliation{jh}{Johns Hopkins University, Baltimore, MD, USA}
\icmlaffiliation{kaust}{Computer Science, King Abdullah University of Science and Technology, Thuwal, Saudi Arabia}

\icmlcorrespondingauthor{Zheng Shi}{shi.zheng.tfls@gmail.com}
\icmlcorrespondingauthor{Martin Tak\'a\v{c}}{takac.MT@gmail.com}

\icmlkeywords{Machine Learning, ICML}

\vskip 0.3in
]



\printAffiliationsAndNotice{} 
 
\begin{abstract}
We present \textit{AI-SARAH}, a practical variant of \textit{SARAH}. As a variant of \textit{SARAH}, this algorithm employs the stochastic recursive gradient yet adjusts step-size based on local geometry. \textit{AI-SARAH} implicitly computes step-size and efficiently estimates local Lipschitz smoothness of stochastic functions. It is fully adaptive, tune-free, straightforward to implement, and computationally efficient. We provide technical insight and intuitive illustrations on its design and convergence. We conduct extensive empirical analysis and demonstrate its strong performance compared with its classical counterparts and other state-of-the-art first-order methods in solving convex machine learning problems.
\end{abstract}

\section{Introduction}
\label{intro}
We consider the unconstrained finite-sum optimization problem
\begin{align}
     \min _{w \in \cR^d}  \left[ P(w) \defeq \tfrac{1}{n}\ \sum _{i=1}^n f_i(w) \right].\label{MainProb}
\end{align}
This problem is prevalent in machine learning tasks where $w$ corresponds to the model parameters, $f_i(w)$ represents the loss on the training point $i$, and the goal is to minimize the average loss $P(w)$ across the training points. In machine learning applications, (\ref{MainProb}) is often considered the loss function of Empirical Risk Minimization (ERM) problems. For instance, given a classification or regression problem, $f_i$ can be defined as logistic regression or least square by $(x_i,y_i)$ where $x_i$ is a feature representation and $y_i$ is a label.
Throughout the paper, we assume that each function $f_i$, $i \in \n \defeq \{1,...,n\}$, is smooth and convex, and there exists an optimal solution $w^*$ of (\ref{MainProb}).
\subsection{Main Contributions}
\label{contribution}
We propose \textit{\textit{AI-SARAH}}, a practical variant of stochastic recursive gradient methods \citep{sarah17} to solve (\ref{MainProb}). This practical algorithm explores and adapts to local geometry. It is adaptive at full scale yet requires zero effort of tuning hyper-parameters. The extensive numerical experiments demonstrate that our tune-free and fully adaptive algorithm is capable of delivering a consistently competitive performance on various datasets, when comparing with \textit{\textit{SARAH}}, \textit{\textit{SARAH}+} and other state-of-the-art first-order method, all equipped with fine-tuned hyper-parameters (which are selected from $\approx 5,000$ runs for each problem). This work provides a foundation on studying adaptivity (of stochastic recursive gradient methods) and demonstrates that a \textbf{truly adaptive  stochastic recursive algorithm can be developed in practice.}

\subsection{Related Work}

Stochastic gradient descent (\textit{\textit{SGD}}) \citep{robbins1951stochastic, NemYudin1983book, Pegasos, Nemirovski-Juditsky-Lan-Shapiro-2009, gower2019sgd} is the workhorse for training supervised machine learning problems that have the generic form (\ref{MainProb}).  
\\ 
In its generic form, \textit{\textit{SGD}} defines the new iterate by subtracting a multiple of a stochastic gradient $g(w_t)$ from the
current iterate $w_t$. That is, $$w_{t+1}=w_t-\alpha_t g(w_t).$$ 
In most algorithms, $g(w)$ is an unbiased estimator of the gradient (i.e., a stochastic gradient), $\Exp[g(w)]=\nabla P(w), \forall w \in \cR^d$. However, in several algorithms (including the ones from this paper), $g(w)$ could be a biased estimator, and convergence guarantees can still be well obtained.
\\
\textbf{Adaptive step-size selection.} The main parameter to guarantee the convergence of \textit{SGD} is the \emph{step-size}. In recent years, several ways of selecting the step-size have been proposed. For example, an analysis of \textit{SGD} with constant step-size ($\alpha_t=\alpha$) or decreasing step-size has been proposed in \citet{moulines2011non,ghadimi2013stochastic, needell2014stochastic, pmlr-v80-nguyen18c, bottou2018optimization, gower2019sgd, gower2020sgd} under different assumptions on the properties of problem~ \eqref{MainProb}. 
\\[2pt]
 More recently, \emph{adaptive / parameter-free} methods \citep{duchi2011adaptive, kingma2014adam, bengio2015rmsprop, li2018convergence, Schmidt19, liu2019variance, ward2019adagrad, loizou2020stochastic} that adapt the step-size as the algorithms progress have become popular and are particularly beneficial when training deep neural networks. Normally, in these algorithms, the step-size does not depend on parameters that might be unknown in practical scenarios, like the smoothness
 or the strongly convex parameter. 
  \\[2pt]
\textbf{Random vector $g(w_t)$ and variance reduced methods.}
One of the most remarkable algorithmic breakthroughs in recent years was the development of variance-reduced stochastic gradient algorithms for solving finite-sum optimization problems.  These algorithms, by reducing the variance of the stochastic gradients, are able to guarantee convergence to the exact solution of the optimization problem with faster convergence than classical \textit{SGD}. In the past decade, many efficient variance-reduced methods have been proposed. Some popular examples of variance reduced algorithms are \textit{SAG} \citep{schmidt2017minimizing}, \textit{SAGA} \citep{saga14}, \textit{SVRG} \citep{svrg13} and \textit{SARAH} \citep{sarah17}. For more examples of variance reduced methods, 
see \citet{defazio2016simple, mS2GD, GowerRichBach2018, khaled2020unified, horvth2020adaptivity, cutkosky2020momentumbased}.
\\[2pt]
Among the variance reduced methods, \textit{SARAH} is of our interest in this work. Like the popular \textit{SVRG}, \textit{SARAH} algorithm is composed of two nested loops. In each outer loop $k \geq 1$, the gradient estimate $v_0 = \nabla P(w_{k-1})$ is set to be the full gradient. 
Subsequently, in the inner loop, at $t \geq 1$, a biased estimator $v_t$ is used and defined recursively~as  
\begin{align}
\label{TheVt}
    v_t = \nabla f_i(w_t) - \nabla f_i(w_{t-1}) + v_{t-1},
\end{align}
where $i \in \n$ is a random sample selected at $t$.
\\[2pt]
A common characteristic of the popular variance reduced methods is that the step-size $\alpha$ in their update rule $w_{t+1}=w_t -\alpha v_t$ is  constant (or diminishing with predetermined rules) and that depends on the characteristics of  problem~(\ref{MainProb}). An exception to this rule are the variance reduced methods with Barzilai-Borwein step size, named  \textit{BB-SVRG} and \textit{BB-SARAH} proposed in \citet{tan2016barzilai} and \citet{li2019adaptive} respectively. These methods allow to use Barzilai-Borwein (\textit{BB}) step size rule to update the step-size once in every epoch; for more examples, see \citet{pmlr-v119-li20n, YANG2021}. There are also methods proposing approach of using local Lipschitz smoothness to derive an adaptive step-size \citep{LIU2019} with additional tunable parameters or leveraging \textit{BB} step-size with averaging schemes to automatically determine the inner loop size \citep{pmlr-v119-li20n}. However, these methods do not fully take advantage of the local geometry, and \textbf{a truly adaptive algorithm: adjusting step-size at every (inner) iteration and eliminating need of tuning any hyper-parameters, is yet to be developed in the stochastic variance reduced framework.} This is exactly the main contribution of this work, as we mentioned in previous section.
\section{Motivation}
\label{motivation}
With our primary focus on the design of a stochastic recursive algorithm with adaptive step-size, we discuss our motivation in this chapter.
\\[2pt]
A standard approach of tuning the step-size involves the painstaking grid search on a wide range of candidates. While more sophisticated methods can design a tuning plan, they often struggle for efficiency and/or require a considerable amount of computing resources.
\\[2pt]
More importantly, tuning step-size requires knowledge that is not readily available at a starting point $w_0 \in \cR^d$, and choices of step-size could be heavily influenced by the curvature provided $\nabla^2 P(w_0)$. \textit{What if a step-size has to be small due to a "sharp" curvature initially, which becomes "flat" afterwards?}
\\[2pt]
To see this is indeed the case for many machine learning problems, let us consider logistic regression for a binary classification problem, i.e., $f_i(w) = \log (1+\exp (-y_ix_i^T w)) + \frac{\lambda}{2}\|w\|^2$, where $x_i \in \cR^d$ is a feature vector, $y_i \in \{-1,+1\}$ is a ground truth, and the ERM problem is in the form of (\ref{MainProb}). It is easy to derive the local curvature of $P(w)$, defined by its Hessian in the form
\begin{align}
\nabla^2 P(w)
 &= \tfrac1n \textstyle{\sum}_{i=1}^n
 \underbrace{\tfrac{\exp(-y_i x_i^T w)}{[1+\exp(-y_i x_i^T w)]^2}}_{s_i(w)}
 x_ix_i^T + \lambda I. \label{eq:motivation}
\end{align}
Given that $\frac{a}{(1+a)^2} \leq 0.25$ for any $a\geq 0$, one can immediately obtain the global bound on Hessian, i.e. $\forall w\in \cR^d$ we have 
$
\nabla^2 P(w)
 \preceq  \tfrac14 \frac1n \sum_{i=1}^n
 x_ix_i^T + \lambda I$.
 Consequently, the parameter of global Lipschitz smoothness is $L =\tfrac14 \lambda_{\max}(\frac1n \sum_{i=1}^n
 x_ix_i^T)+\lambda$. It is well known that, with a constant step-size less than (or equal to) $\frac{1}{L}$, a convergence is guaranteed by many algorithms.
\\ 
 However, suppose the algorithm starts at a random $w_0$ (or at $\mathbf{0} \in \cR^d$), this bound can be very tight. With more progress being made on approaching an optimal solution (or reducing the training error), it is likely that, for many training samples, $-y_i x_i^T w_t \ll 0$. An immediate implication is that $s_i(w_t)$ defined in (\ref{eq:motivation}) becomes smaller and hence the local curvature will be smaller as well. It suggests that, although a large initial step-size could lead to divergence, with more progress made by the algorithm, the parameter of local Lipschitz smoothness tends to be smaller and a larger step-size can be used. That being said, such a dynamic step-size cannot be well defined in the beginning, and a \textbf{fully adaptive approach needs to be developed.} 
\\
 For illustration, we present the inspiring results of an experiment on \textit{real-sim} dataset \cite{chang2011libsvm}
 with \reglize{} logistic regression. Figures~\ref{fig:a} and \ref{fig:b} compare the performance of classical \textit{\textit{SARAH}} with \textit{\textit{AI-SARAH}} in terms of the evolution of the optimality gap and the squared norm of recursive gradient. As is clear from the figure, \textit{\textit{AI-SARAH}} displays a significantly faster convergence per effective pass\footnote{The effective pass is defined as a complete pass on the training dataset. Each data sample is selected once per effective pass on average.}.
\\
 Now, let us discuss why this could happen. The distribution of $s_i$ as shown in Figured \ref{fig:d} indicates that: initially, all $s_i$'s are concentrated at $0.25$; the median continues to reduce within a few effective passes on the training samples; eventually, it stabilizes somewhere below $0.05$. Correspondingly, as presented in Figure \ref{fig:c}, \textit{\textit{AI-SARAH}} starts with a conservative step-size dominated by the global Lipschitz smoothness, i.e., $1/\lambda_{max}(\nabla^2 P(w_0))$ (red dots); however, within $5$ effective passes, the moving average (magenta dash) and upper-bound (blue line) of the step-size start surpassing the red dots, and eventually stablize above the conservative step-size. 
\\ 
 For classical \textit{\textit{SARAH}}, we configure the algorithm with different values of the fixed step-size, i.e., $\{2^{-2},2^{-1},...,2^4\}$, and notice that $2^5$ leads to a divergence. On the other hand, \textit{\textit{AI-SARAH}} starts with a small step-size, yet achieves a faster convergence per effective pass with an eventual (moving average) step-size larger than $2^5$.
\begin{figure}
     \centering
     \subfigure[]{\label{fig:a}\includegraphics[width=0.22\textwidth]{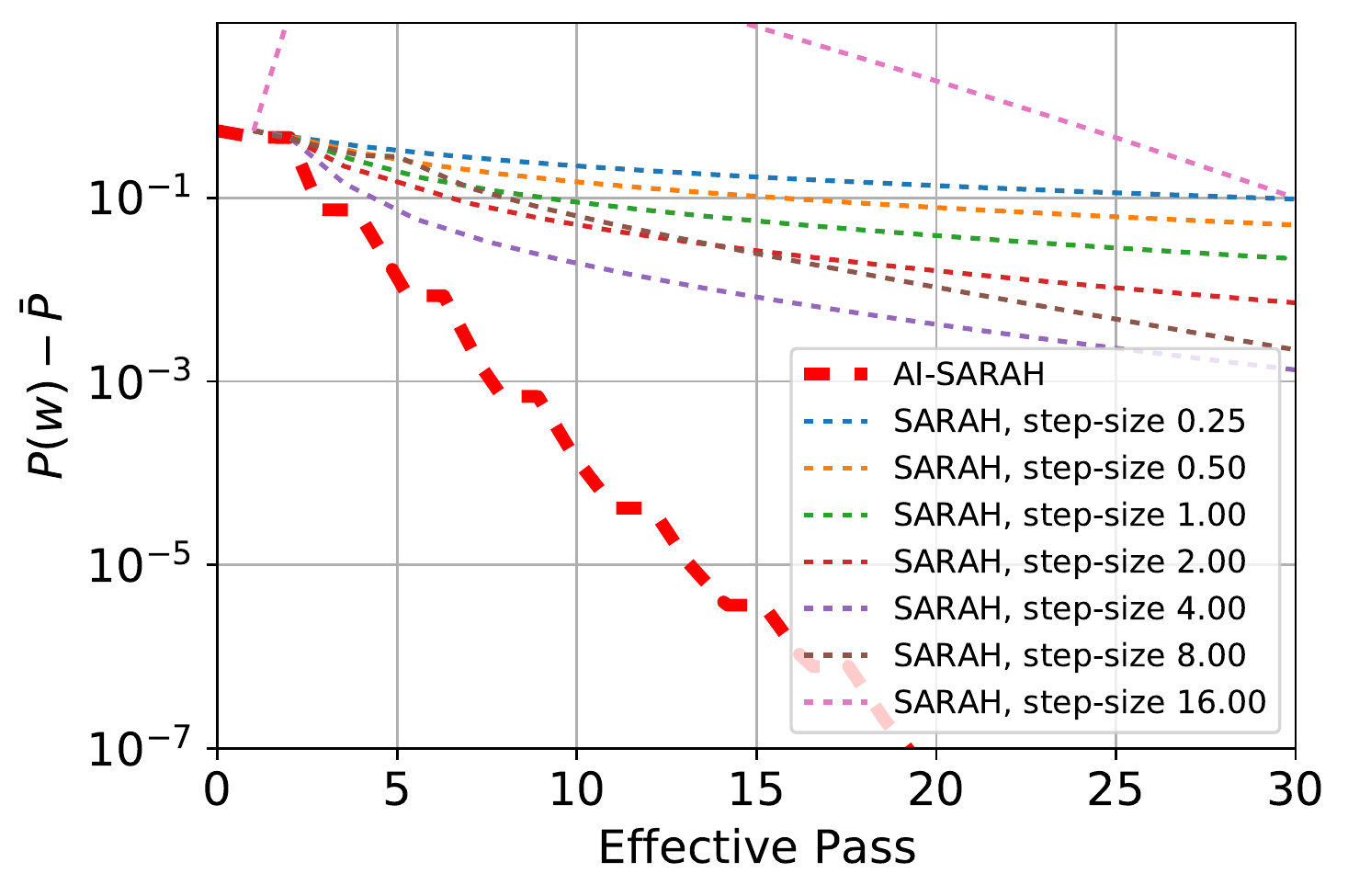}}
     \subfigure[]{\label{fig:b}\includegraphics[width=0.22\textwidth]{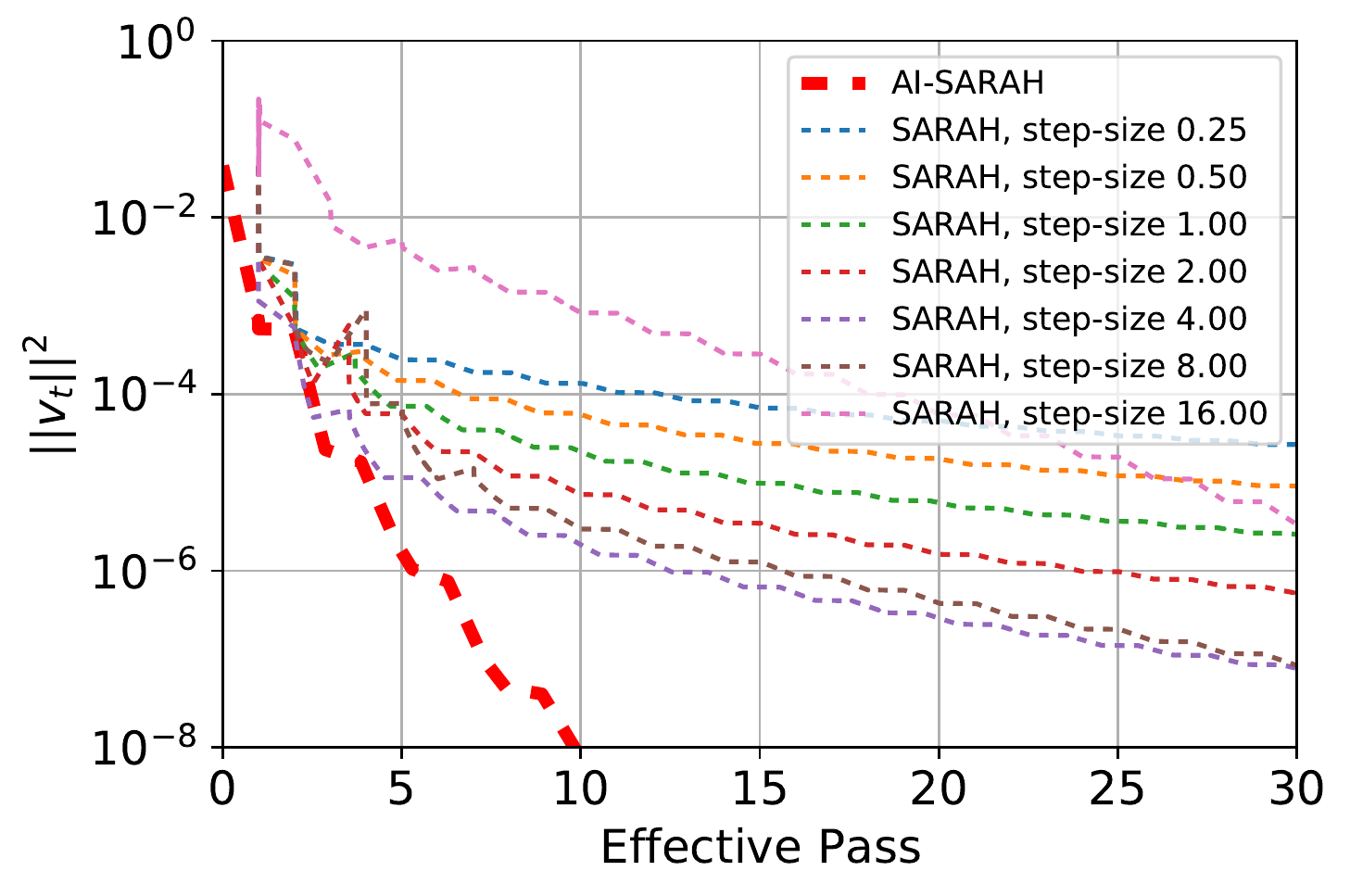}}
     \vskip-10pt
     \subfigure[]{\label{fig:c}\includegraphics[width=0.22\textwidth]{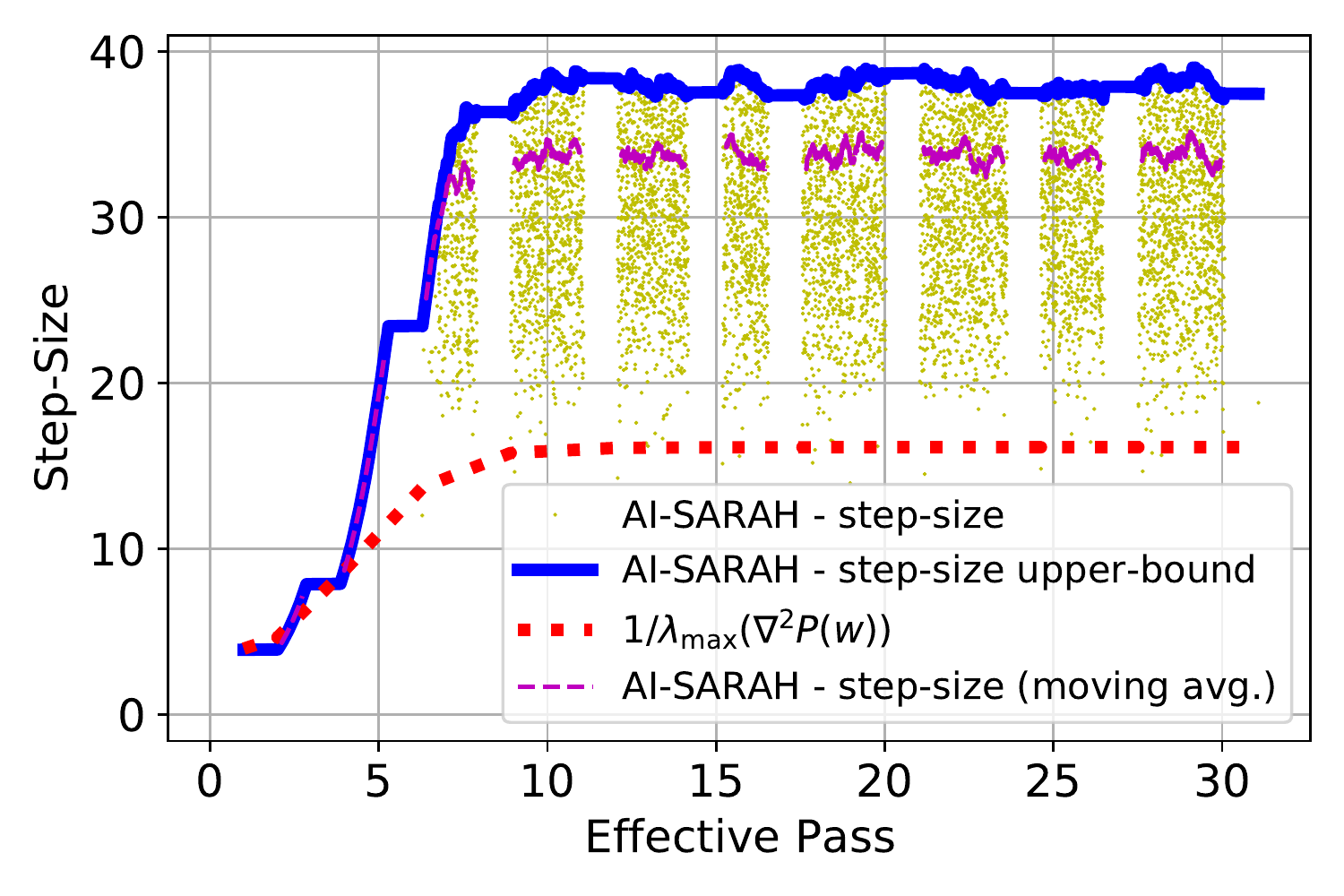}}
     \subfigure[]{\label{fig:d}\includegraphics[width=0.22\textwidth]{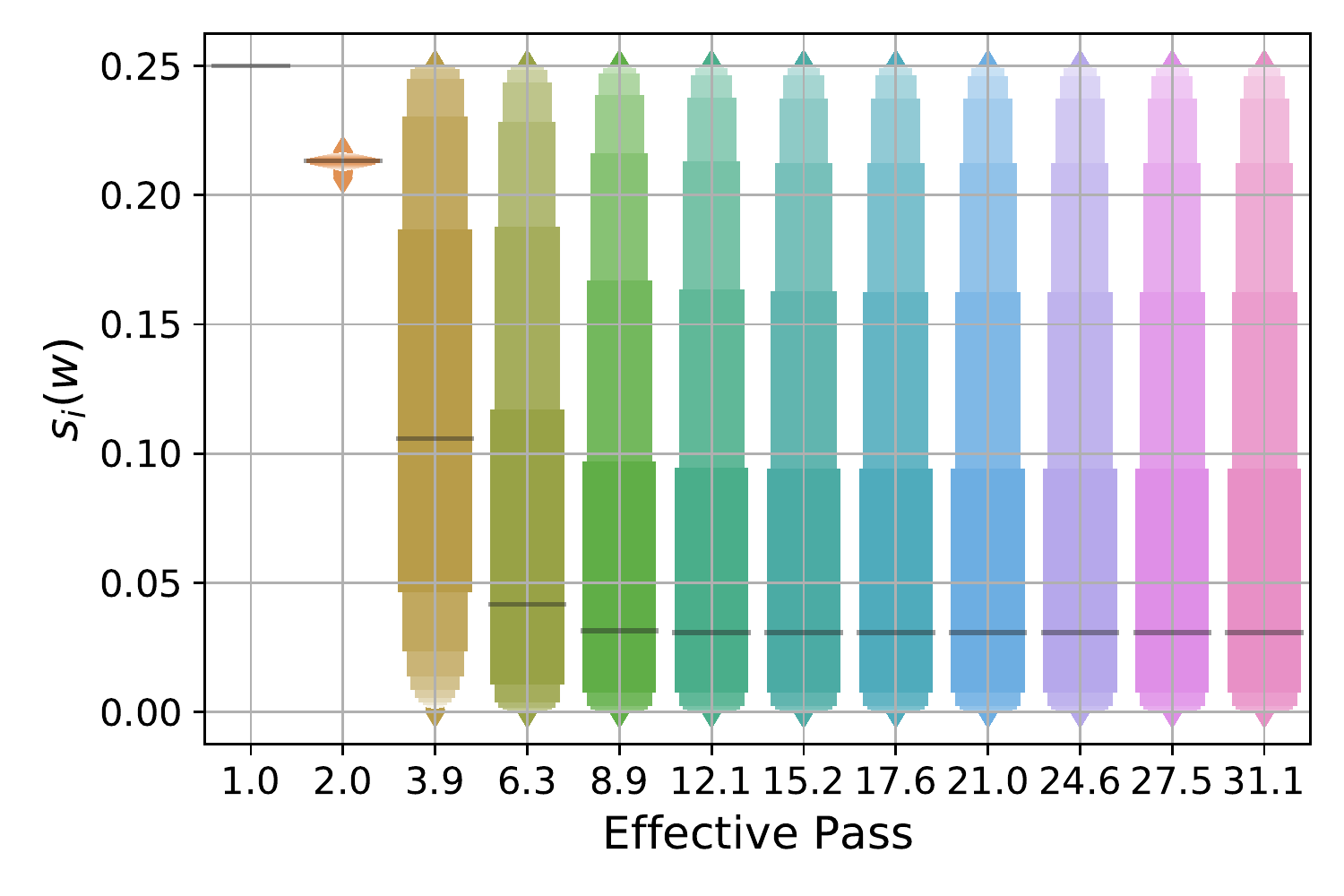}}
     \vskip-10pt
     \caption{\textit{\textit{AI-SARAH}} vs. \textit{\textit{SARAH}}: \textbf{(a)} evolution of the optimality gap $P(w) - \bar{P}$ and \textbf{(b)} the squared norm of stochastic recursive gradient $\|v_t\|^2$; \textit{\textit{AI-SARAH}}: \textbf{(c)} evolution of the step-size, upper-bound, local Lipschitz smoothness and \textbf{(d)} distribution of $s_i$ of stochastic functions.  \textit{Note: in (a), $\bar P$ is a lower bound of $P(w^*)$; in (c), the white spaces suggest full gradient computations at outer iterations; in (d), bars represent medians of $s_i$'s.}}
     \vskip-20pt
\end{figure}

\section{Theoretical Analysis}
\label{sec:theoretical-analysis}
In this section, we present the theoretical investigation on leveraging local Lipschitz smoothness to dynamically determine the step-size. We are trying to answer the main question: can we show convergence of using such an adaptive step-size and what are the benefits. 
\\
We present the theoretical framework in Algorithm \ref{algo:aisarah-theory} and refer to it as \textit{Theoretical-AI-SARAH}. For brevity, we show the main results in the section and defer the full technical details to Appendix \ref{app:line-segment}.
\begin{algorithm}[ht!]
\caption{\textit{Theoretical-AI-SARAH}}
\label{new_algorithm_1}
\label{alg_AI_Sarah_Theory}

	\begin{algorithmic}[1]
	\STATE \textbf{Parameter:} Inner loop size $m$
		\STATE \textbf{Initialize:} $\tilde{w}_0$
		\FOR{$k = 1, 2, ...$}
		\STATE $w_0 = \tilde{w}_{k-1}$
		\STATE $v_0 = \nabla P(w_0)$ 
		\FOR{$i \in \n$}
		\STATE Compute $L^0_i$ in the neighborhood of $w_0$
		\ENDFOR
        \STATE Compute $L^0$ with $\{L^0_i\}^n_{i=1}$ and set $\eta_0 = 1/L^0$
        \FOR{$t = 1, ... , m$}
		\STATE $w_{t}=w_{t-1} - \eta_{t-1} v_{t-1}$
		\STATE Sample $i_t$ randomly from $\n$
        \STATE  $v_t = v_{t-1} + \nabla f_{i_t}(w_t) - \nabla f_{i_t}(w_{t-1})$
        \FOR{$i \in \n$}
        \STATE Compute $L^t_i$ in the neighborhood of $w_t$
        \ENDFOR
        \STATE Compute $L^t$ with $\{L^t_i\}^n_{i=1}$
        \STATE $\eta_t = 
        \frac1{L^t}$
		\ENDFOR
        \STATE \textbf{Set} $\tilde w_k = w_t$ where $t$ is chosen with probability $q_t$ from $\{0,1,...,m\}$
		\ENDFOR
	\end{algorithmic}
	\label{algo:aisarah-theory}
\end{algorithm}
\\  
For $t \geq 0, i_t \in \n$, We assume $f_{i_t}$ is $L^t_i$-smooth on the line segment $\Delta^t_i = \{w \in \cR^d \;|\; w=w_t-\eta v_t,\;\eta \in [0,\frac{1}{L^t_i}]\}$. 
Then, for Lines $7$ and $15$, we have
\begin{align*}
    L^t_i = \max \|\nabla^2 f_i(w_t - \eta v_t)\|,\quad \text{where }\eta \in [0,\tfrac{1}{L^t_i}]. \tagthis \label{eq:line-segment}
\end{align*}
The problem \eqref{eq:line-segment} essentially computes the largest eigenvalue of Hessian matrices on the defined line segment. Note that, \eqref{eq:line-segment} computes $L^t_i$ implicitly as it appears on both sides of the equation. 
\\
For Line $12$, we propose to use either uniform or importance sampling strategy. That is, at $t\geq 1$, we sample $i_t$ randomly from $\n$ with probability $p^{t-1}_i$, where $p^t_i = 1/n$ or $p^t_i = L^t_i/\textstyle{\sum}_i^n L^t_i$. Then, for Line $13$ with specific sampling strategies, we have instead the form
\begin{align*}
    v_t = v_{t-1} + \tfrac{1}{np^{t-1}_i} \left(\nabla f_{i_t}(w_t) - \nabla f_{i_t}(w_{t-1})\right).
\end{align*}
\\
Then, for Lines $9$ and $17$, we either compute the maximum for the uniform sampling, i.e., $L^t = \max_{i\in \n} L^t_i$, or the average for the importance sampling, i.e., $L^t = \tfrac{1}{n}\sum^n_i L^t_i$, for $t \geq 0$. 
\\
Note that, for Line $20$, we define $q_t = \eta_t/\sum^m_{j=0} \eta_j$.
\\
Having presented Algorithm \ref{algo:aisarah-theory}, we can now show our main technical result in the following theorem.
\begin{theorem}\label{theorem:line-segment} Suppose $P$ is $\mu$-strongly convex, and each $f_i$ is convex and $L^t_i$-smooth on the line segment $\Delta^t_i = \{w \in \cR^d \;|\; w=w_t-\eta v_t,\;\eta \in [0,\frac{1}{L^t_i}]\}$.
For $k \geq 1$, let us define 
\begin{align*}
    \sigma^k_m=\left(\frac{1}{\mu \mathcal{H}} +\frac{\eta_{0} L^{0}}{2-\eta_{0} L^{0}}\right),
    \tagthis \label{eq:sigma}
\end{align*}
where $\mathcal{H} = \sum^m_{t=0} \eta_t$, and select $m$ and $\eta$ such that $\sigma_m<1$, Algorithm \ref{algo:aisarah-theory} converges as follows
\begin{equation*}
    \EE[\| \nabla P(\tilde{w}_{k})\|^2 \leq \left(\prod^{k}_{l=1}\sigma^l_m\right) \|\nabla P(\tilde{w}_{0})\|^2.
\end{equation*}
\end{theorem}
\begin{corollary}\label{cor:faster-than-sarah}
The convergence result of classical SARAH~\cite{sarah17} 
algorithm for strongly convex case has a similar form of $\sigma$ 
as \eqref{eq:sigma}
but
with $\mathcal H_{SARAH} = \alpha (m+1)$, where $\alpha \leq \frac1L$ is a constant step-size. In this case, $L$ is the global smoothness parameter of each $f_i, i \in [n]$. Now, as \eqref{eq:line-segment}
defines $L_i^t$
to be the parameter of smoothness only on a line segment, we trivially have
that
\begin{align*}
    L^t_i
    &\overset{\eqref{eq:line-segment}}{=}
      \max_{\eta \in [0,\tfrac{1}{L^t_i}]} \|\nabla^2 f_i(w_t - \eta v_t)\|
\leq 
\max_{w \in \cR^d} \|\nabla^2 f_i(w)\| \leq L.
\end{align*}
Thus,
$\mathcal{H} = \sum_{t=0}^m \eta_t
 = \sum_{t=0}^m \tfrac1{L^t}
 \geq \sum_{t=0}^m \tfrac1{L} = \frac{m+1}L \geq \alpha(m+1) = \mathcal H_{SARAH}$. Then, it is clear that, Algorithm~\ref{alg_AI_Sarah_Theory}
can achieve a faster convergence than classical \textit{SARAH}.
\end{corollary}
By Theorem \ref{theorem:line-segment} and Corollary \ref{cor:faster-than-sarah}, we show that, in theory, by leveraging local Lipschitz smoothness, Algorithm \ref{algo:aisarah-theory} is guaranteed to converge and can even achieve a faster convergence than classical \textit{SARAH} if local geometry permits.
\\[2pt]
With that being said, we note that Algorithm \ref{algo:aisarah-theory} requires the computations of the largest eigenvalues of Hessian matrices on the line segment for each $f_i$ at every outer and inner iterations. In general, such computations would be too expensive,  and thus would keep one from solving Problem \eqref{MainProb} efficiently in practice.
\\
In the next section, we will present our main contribution of the paper, the practical algorithm, \textit{AI-SARAH}. \textbf{It does not only eliminate the expensive computations in Algorithm \ref{algo:aisarah-theory}, but also eliminate efforts of tuning hyper-parameters.} 

\section{AI-SARAH}\label{sec:fully_adaptive}
We present the practical algorithm, \textit{\textit{AI-SARAH}}, in Algorithm~\ref{algo:aisarah}. 
\\
At every iteration, instead of incurring expensive costs on computing the parameters of local Lipshitz smoothness for all $f_i$ in Algorithm \ref{algo:aisarah-theory}, Algorithm \ref{algo:aisarah} estimates the local smoothness by approximately solving the sub-problem for only one $f_i$, i.e., $\min_{\alpha>0} \xi_t(\alpha)$, with a minimal extra cost in addition to computing stochastic gradient, i.e., $\nabla f_i$. Also, by approximately solving the sub-problem, Algorithm \ref{algo:aisarah} implicitly computes the step-size, i.e., $\alpha_{t-1}$ at $t\geq 1$. 
\\
In Algorithm \ref{algo:aisarah}, we adopts an adaptive upper-bound with exponential smoothing. To be specific, the upper-bound is updated with exponential smoothing on harmonic mean of the approximate solutions to the sub-problems, which also keeps track of the estimates of local Lipschitz smoothness. 
\\
In the following sections, we will present the details on the design of \textit{AI-SARAH}. 
\\
We note that \textbf{this algorithm is fully adaptive and requires no efforts of tuning, and can be implemented easily}. Notice that $\beta$ is treated as a smoothing factor in updating the upper-bound of the step-size, and the default setting is $\beta=0.999$. There exists one hyper-parameter in Algorithm \ref{algo:aisarah}, $\gamma$, which defines the early stopping criterion on Line 8, and the default setting is $\gamma = \frac{1}{32}$. We will show later in this chapter that, the performance of this algorithm is not sensitive to the choices of $\gamma$, and this is true regardless of the problems (i.e., regularized/non-regularized logistic regression and different datasets.)

\begin{algorithm}
\caption{\textit{AI-SARAH}}
\begin{algorithmic}[1]
\STATE {\bfseries Parameter:} $0 < \gamma < 1$ (default $\frac{1}{32}$), $\beta = 0.999$
\STATE {\bfseries Initialize:} $\tilde{w}_0$
\STATE {\bfseries Set:} $\alpha_{max} = \infty$
\FOR{k = 1, 2, ...}
\STATE $w_0 = \tilde{w}_{k-1}$
\STATE $v_0 = \nabla P(w_0)$
\STATE $t=1$
\WHILE{$\|v_t\|^2 \geq \gamma\|v_0\|^2$} 
\STATE Select random mini-batch $S_t$ from $\n$ uniformly with $|S_t| = b$
\STATE $\tilde \alpha_{t-1} \approx \arg\min_{\alpha>0} \xi_t(\alpha)$
\label{alg:step:l10}
\IF {$k=1$ and $t=1$}
\STATE $\delta^k_t = \frac{1}{\tilde \alpha_{t-1}}$
\ELSE
\STATE $\delta^k_t = \beta \delta^k_{t-1} + (1-\beta)\frac{1}{\tilde \alpha_{t-1}}$
\ENDIF
\STATE $\alpha_{max} = \frac{1}{\delta^k_t}$
\label{alg:step:l16}
\STATE $\alpha_{t-1} = \min \{\tilde \alpha_{t-1},\alpha_{max}\}$
\STATE $w_t = w_{t-1} - \alpha_{t-1} v_{t-1}$
\STATE $v_t = \nabla f_{S_t}(w_t) - \nabla f_{S_t}(w_{t-1}) + v_{t-1}$
\STATE $t=t+1$
\ENDWHILE
\STATE Set $\delta^{k+1}_0 = \delta^k_t$
\STATE Set $\tilde{w}_k=w_t$
\ENDFOR
\end{algorithmic}
\label{algo:aisarah}
\end{algorithm} 

\subsection{Estimate Local Lipschitz Smoothness}
\label{estimate_l}
In the previous chapter, we showed that Algorithm \ref{algo:aisarah-theory} computes the parameters of local Lipschitz smoothness, and it can be very expensive and thus prohibited in practice. To avoid the expensive cost, one can estimate the local Lipschitz smoothness instead of computing an exact parameter. Then, the question is how to estimate the parameter of local Lipschitz smoothness in practice. 
\\ 
{\bfseries Can we use line-search?}
The standard approach to estimate local Lipschitz smoothness is to use backtracking line-search. Recall \textit{\textit{SARAH}}'s update rule, i.e., $w_t = w_{t-1}-\alpha_{t-1} v_{t-1}$, where $v_{t-1}$ is a stochastic recursive gradient. The standard procedure is to apply line-search on function $f_{i_t}(w_{t-1}-\alpha v_{t-1})$. However, the main issue is that $-v_{t-1}$ is not necessarily a descent direction.
\\[2pt]
{\bfseries \textit{\textit{AI-SARAH}} sub-problem.}
Define the sub-problem\footnote{For the sake of simplicity, we use $f_{i_t}$ instead of $f_{S_t}$.} (as shown on line 10 of Algorithm \ref{algo:aisarah}) as
\begin{align*}
    \min_{\alpha > 0} \xi_t(\alpha)
    =& \min_{\alpha>0} \|\nabla f_{i_t} (w_{t-1} - \alpha v_{t-1}) - \nabla f_{i_t} (w_{t-1}) \\
    &\qquad\quad  + v_{t-1}\|^2, \tagthis{} \label{eq:newsubproblem}
\end{align*}
where $t \geq 1$ denotes an inner iteration and $i_t$ indexes a random sample selected at $t$.
We argue that, by (approximately) solving (\ref{eq:newsubproblem}), we can have a good estimate of the parameters of the local Lipschitz smoothness.
\\[2pt]
To illustrate this setting, we denote $L^t_i$ the parameter of local Lipschitz smoothness prescribed by $f_{i_t}$ at $w_{t-1}$. Let us focus on a simple quadratic function $f_{i_t}(w) = \tfrac12 (x_{i_t}^T w-y_{i_t})^2$.
Let $\tilde \alpha$ be the optimal step-size along direction  $-v_{t-1}$, i.e.   $\tilde \alpha
 = \arg\min_\alpha  f_{i_t}(w_{t-1} - \alpha v_{t-1})$. Then, the closed form solution of $\tilde \alpha$ can be easily derived as
$
\tilde \alpha
 = \frac{x_{i_t}^T w_{t-1}   -y_{i_t}}{x_{i_t}^T v_{t-1}}
$, whose value can be positive, negative, bounded or unbounded.

On the other hand, one can compute the step-size implicitly by solving (\ref{eq:newsubproblem}) and obtain $\alpha_{t-1}^i$, i.e., 
$\alpha_{t-1}^i = \arg\min_\alpha \xi_t(\alpha)$. Then, we have 
\begin{align*}
     \alpha_{t-1}^i = \frac{1}{x_{i_t}^Tx_{i_t}},
\end{align*}
which is exactly $\frac{1}{L^t_i}$.
\\[2pt]
\textbf{Simply put, as quadratic function has a constant Hessian, solving (\ref{eq:newsubproblem}) gives exactly $\frac{1}{L^t_i}$. For general (strongly) convex functions, if $\nabla^2 f_{i_t}(w_{t-1})$, does not change too much locally, we can still have a good estimate of $L^i_t$ by solving (\ref{eq:newsubproblem}) approximately}.
\\[2pt]
Based on a good estimate of $L^t_i$, we can then obtain the estimate of the local Lipschitz smoothness of $P(w_{t-1})$. And, that is
\begin{align*}
    L^t = \frac{1}{n} \sum_{i=1}^{n} L^t_i = \frac{1}{n} \sum _{i=1}^n \frac{1}{\alpha_{t-1}^i}.
\end{align*}
Clearly, if a step-size in the algorithm is selected as $1/L^t$, then a harmonic mean of the sequence of the step-size's, computed for various component functions could serve as a good adaptive upper-bound on the step-size computed in the algorithm. More details of intuition for the adaptive upper-bound can be found in Appendix \ref{sec:implementation_app}.

\subsection{Compute Step-size and Upper-bound}
On Line 10 of Algorithm~\ref{algo:aisarah}, the sub-problem is a one-dimensional minimization problem, which can be approximately solved by Newton method. Specifically in Algorithm~\ref{algo:aisarah}, we compute \textit{one-step Newton} at $\alpha=0$, and that is
\begin{align*}
    \tilde \alpha_{t-1} = -\tfrac{\xi'_t(0)}{|\xi''_t(0)|}. \tagthis{} \label{eq:one-step-newton} 
\end{align*}
Note that, for convex function in general, (\ref{eq:one-step-newton}) gives an approximate solution; for functions in particular forms such as quadratic ones, (\ref{eq:one-step-newton}) gives an exact solution.
\\[2pt]
The procedure prescribed in (\ref{eq:one-step-newton}) can be implemented very efficiently, and \textbf{it does not require any extra (stochastic) gradient computations if compared with classical \textit{\textit{SARAH}}.} The only extra cost per iteration is to perform two backward passes, i.e., one pass for $\xi_t'(0)$ and the other for $\xi_t''(0)$; see Appendix \ref{sec:implementation_app} for implementation details.   
\\[2pt]
As shown on Lines 11-16, 22 of Algorithm \ref{algo:aisarah}, $\alpha_{max}$ is updated at every inner iteration. Specifically, \textbf{the algorithm starts without an upper bound (i.e., $\alpha_{max} = \infty$} on Line 3); as $\tilde \alpha_{t-1}$ being computed at every $t\geq 1$, we employs the exponential smoothing on the harmonic mean of $\{\tilde \alpha_{t-1}\}$ to update the upper-bound. For $k \geq 1$ and $t \geq 1$, we define $\alpha_{max} = \frac{1}{\delta^k_t}$, where
\begin{align*}
    \delta^k_t = \begin{cases}
    \frac{1}{\tilde \alpha_{t-1}}, & k=1,t=1\\
    \beta \delta^k_{t-1} + (1-\beta) \frac{1}{\tilde \alpha_{t-1}}, & otherwise
    \end{cases}
\end{align*}
and $0 < \beta < 1$. We default $\beta = 0.999$ in Algorithm \ref{algo:aisarah}. At the end of the $k$th outer loop, denoted $t=T$, we let $\delta^{k+1}_0 = \delta^k_T$; see Appendix \ref{sec:implementation_app} for details on the design of the adaptive upper-bound.

\subsection{Choice of $\gamma$}
 We perform a sensitivity analysis on different choices of $\gamma$. Figures \ref{fig:gamma_sense} shows the evolution of the squared norm of full gradient, i.e., $\|\nabla P(w)\|^2$, for logistic regression on binary classification problems; see extended results in Appendix \ref{app:exp}. It is clear that the performance of $\gamma$'s, where, $\gamma \in \{1/8,1/16,1/32,1/64\}$, is consistent with only marginal improvement by using a smaller value. We default $\gamma = 1/32$ in Algorithm \ref{algo:aisarah}.

 \begin{figure*} 
   \centering
    \includegraphics[width=\textwidth]{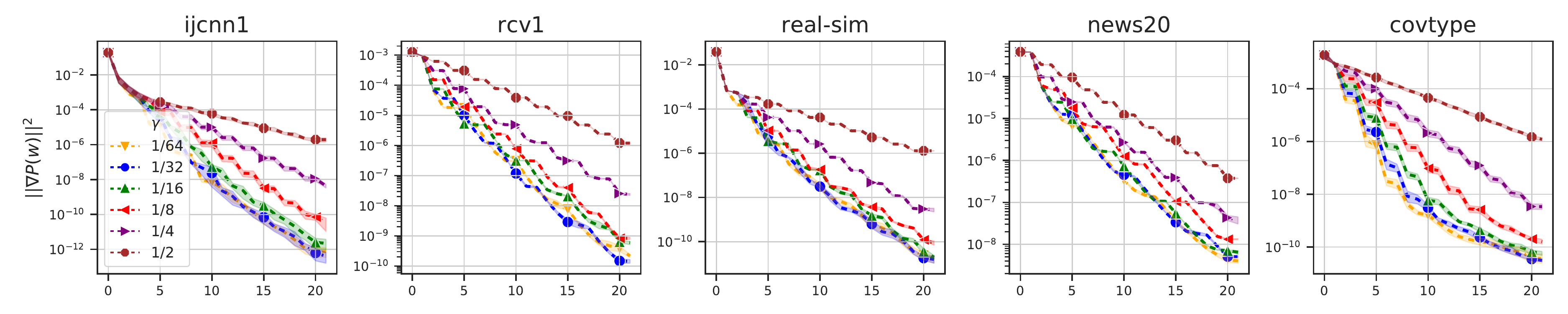}
    \includegraphics[width=\textwidth]{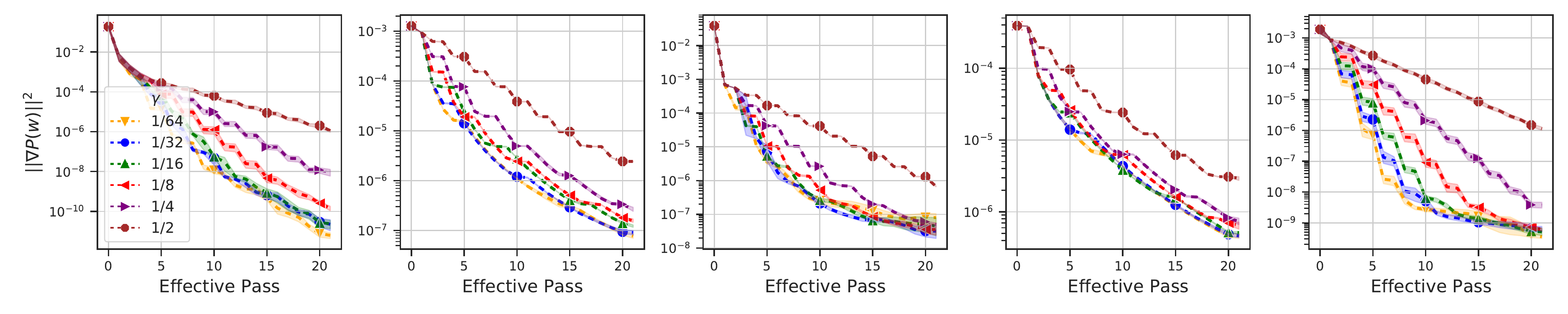}
    \vskip-10pt
    \caption{Evolution of $\|\nabla P(w)\|^2$ for $\gamma \in \{\frac{1}{64},\frac{1}{32},\frac{1}{16},\frac{1}{8},\frac{1}{4},\frac{1}{2}\}$: regularized (top row) and non-regularized (bottom row) logistic regression on \textit{ijcnn1, rcv1, real-sim, news20} and \textit{covtype}.}
    \label{fig:gamma_sense}
     \centering
     \includegraphics[width=\textwidth]{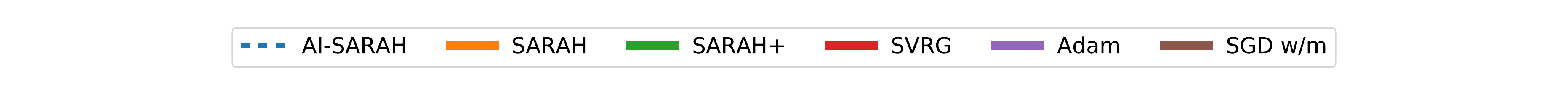}
     \vskip-8pt
     \includegraphics[width=\textwidth]{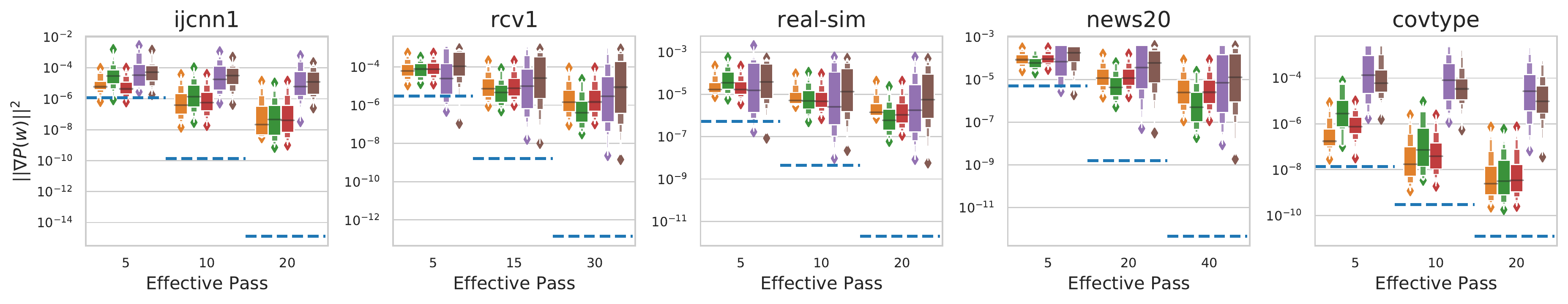}
    \includegraphics[width=\textwidth]{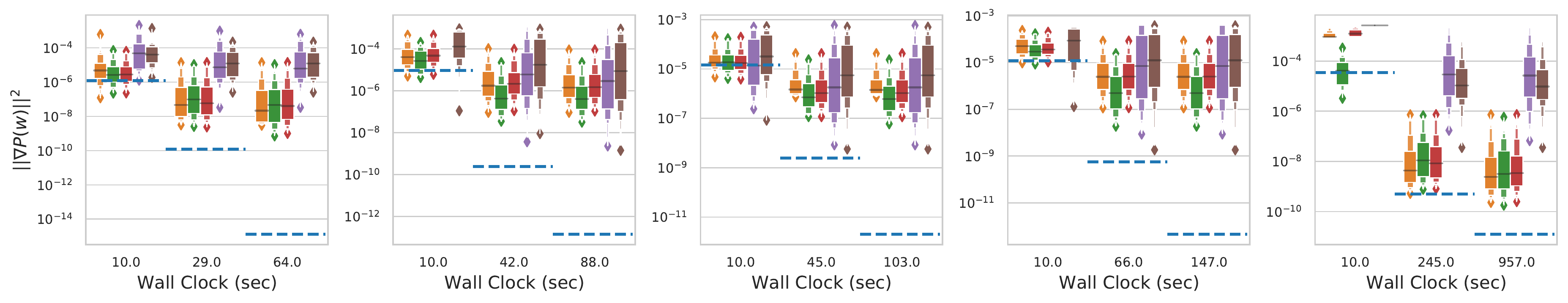}
    \vskip-10pt
    \caption{Running minimum per effective pass (top row) and wall clock time (bottom row) of $\|\nabla P(w)\|^2$ between other algorithms with all hyper-parameters configurations and \textit{\textit{AI-SARAH}} for the \textbf{regularized} case. }
    \label{fig:minimum_norm}
     \centering
     \includegraphics[width=\textwidth]{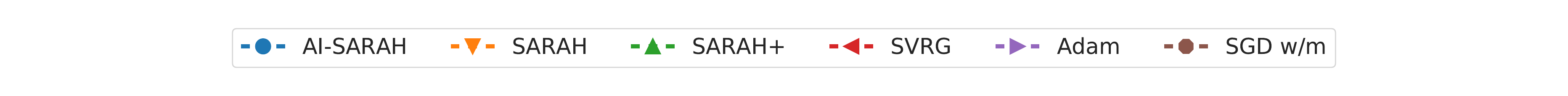}
     \vskip-8pt
     \includegraphics[width=\textwidth]{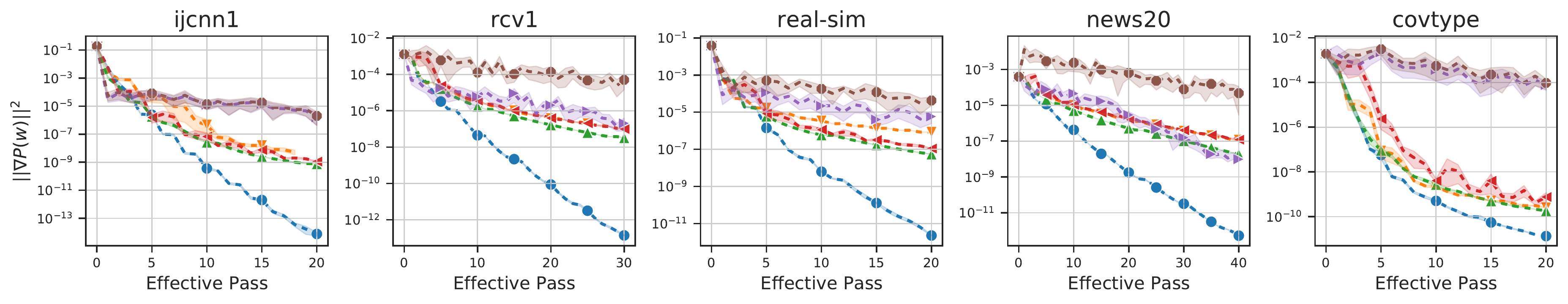}
    \includegraphics[width=\textwidth]{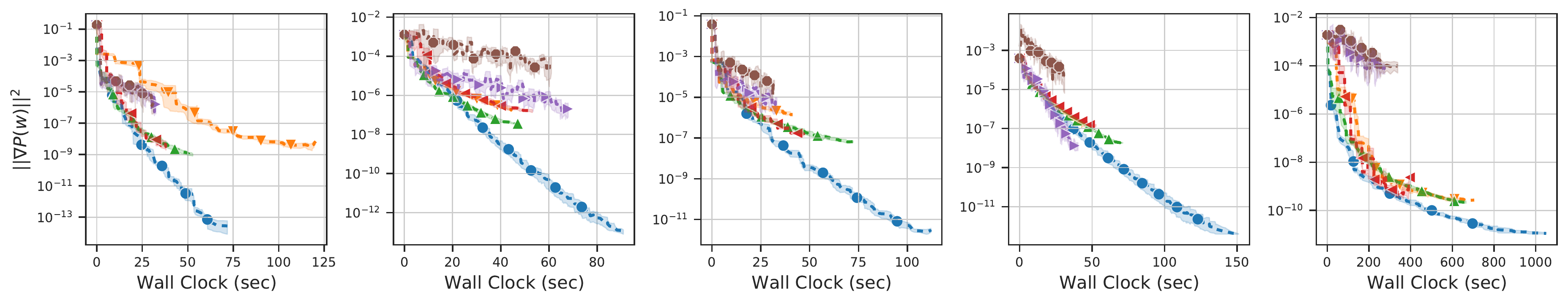}
    \vskip-10pt
    \caption{Evolution of $\|\nabla P(w)\|^2$ for the \textbf{regularized} case by effective pass (top row) and wall clock time (bottom row).}
    \label{fig:squared_norm}
\end{figure*}

\section{Numerical Experiment}
\label{experiments}
\begin{table}
\centering 
\caption{Summary of Datasets from \cite{chang2011libsvm}.}
 
\label{tab:data}
\begin{threeparttable}
\resizebox{\columnwidth}{!}
{%
\begin{tabular}{l r r r r}
\toprule
 Dataset    &  \# features & $n$ (\# Train) & \# Test & \% Sparsity \\
 \hline
 \textit{ijcnn1}\tnote{1}   & 22 & 49,990 & 91,701 & 40.91 \\
 \hline
 \textit{rcv1}\tnote{1}   & 47,236 & 20,242 & 677,399 & 99.85 \\ 
 \hline
  \textit{real-sim}\tnote{2}  & 20,958 & 54,231 & 18,078 & 99.76 \\ 
 \hline
 \textit{news20}\tnote{2}  & 1,355,191 & 14,997 & 4,999 & 99.97 \\ 
 \hline
 \textit{covtype}\tnote{2}  & 54 & 435,759 & 145,253 & 77.88 \\ 
 \hline
 \bottomrule
\end{tabular}
}
\begin{tablenotes}\footnotesize
\item[1] dataset has default training/testing sanples.
\item[2] dataset is randomly split by 75\%-training \& 25\%-testing.
\end{tablenotes}
\end{threeparttable}
\vskip-15pt
\end{table} 

In this chapter, we present the empirical study on the performance of \textit{\textit{AI-SARAH}} (see Algorithm \ref{algo:aisarah}). For brevity, we present a subset of experiments in the main paper, and defer the full experimental results and implementation details\footnote{Code will be made available upon publication.} in Appendix \ref{app:exp}. 

The problems we consider in the experiment are \reglize{} logistic regression for binary classification problems; see Appendix \ref{app:exp} for non-regularized case. 
Given a training sample $(x_i,y_i)$ indexed by $i \in \n$, the component function $f_i$ is in the form
$
    f_i(w) = \log(1+\exp(-y_ix_i^T w)) + \frac{\lambda}{2}\|w\|^2, 
$
where $\lambda = \frac{1}{n}$ for the \reglize{} case and $\lambda = 0$ for the non-regularized case.
 
\begin{figure*}
     \vskip-10pt
    \centering
    \includegraphics[width=\textwidth]{newplot/legend_3.pdf}
    \vskip-8pt
    \includegraphics[width=0.98\textwidth]{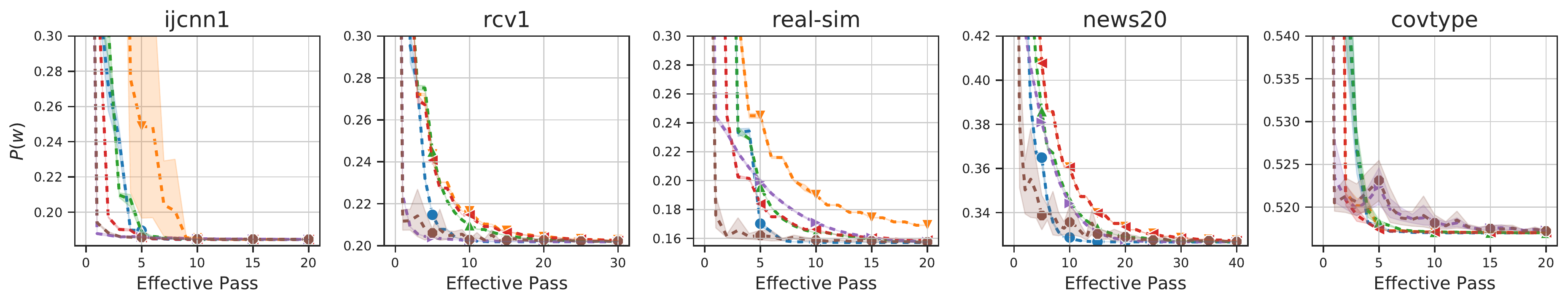}
    \includegraphics[width=0.98\textwidth]{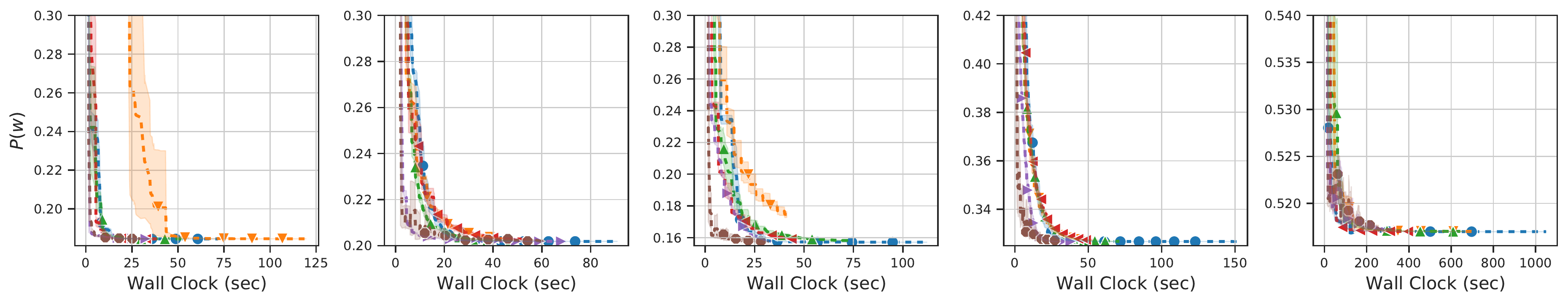}
    \vskip-10pt
    \caption{Evolution of $P(w)$ for the \textbf{regularized} case by effective pass (top row) and wall clock time (bottom row).}
    \label{fig:finite_sum}
\end{figure*}
The datasets chosen for the experiments are \textit{ijcnn1, rcv1, real-sim, news20} and \textit{covtype}. Table \ref{tab:data} shows the basic statistics of the datasets. More details and additional datasets can be found in Appendix \ref{app:exp}. 
 
We compare \textit{AI-SARAH} with \textit{SARAH}, \textit{SARAH}+, \textit{SVRG} \citep{svrg13}, \textit{ADAM} \citep{kingma2014adam} and \textit{SGD} with Momentum \citep{sutskever2013importance,loizou2020momentum,loizou2017linearly}. \textbf{While \textit{AI-SARAH} does not require hyper-parameter tuning, we fine-tune each of the other algorithms, which yields $\mathbf{\approx 5,000}$ runs in total for each dataset and case.} 
 
To be specific, we perform an extensive search on hyper-parameters: (1) \textit{ADAM} and \textit{SGD} with Momentum (\textit{SGD} w/m) are tuned with different values of the (initial) step-size and schedules to reduce the step-size; (2) \textit{SARAH} and \textit{SVRG} are tuned with different values of the (constant) step-size and inner loop size; (3) \textit{SARAH}+ is tuned with different values of the (constant) step-size and early stopping parameter. (See Appendix \ref{app:exp} for detailed tuning plan and the selected hyper-parameters.)
 
Figure \ref{fig:minimum_norm} shows the minimum $\|\nabla P(w)\|^2$ achieved at a few points of effective passes and wall clock time horizon. It is clear that, \textit{AI-SARAH}'s practical speed of convergence is faster than the other algorithms in most cases. Here, we argue that, if given an optimal implementation of \textit{AI-SARAH} (just as that of \textit{ADAM} and other built-in optimizer in Pytorch\footnote{Please see \url{https://pytorch.org/docs/stable/optim.html} for Pytorch built-in optimizers.} ), it is likely that our algorithm can be accelerated.
 
By selecting the fine-tuned hyper-parameters of all other algorithms, we compare them with \textit{AI-SARAH} and show the results in Figures \ref{fig:squared_norm}-\ref{fig:test_accuracy}. For these experiments, we use $10$ distinct random seeds to initialize $w$ and generate stochastic mini-batches. And, we use the marked dashes to represent the average and filled areas for $97\%$ confidence intervals.
 
Figure \ref{fig:squared_norm} presents the evolution of $\|\nabla P(w)\|^2$. Obviously from the figure, \textit{AI-SARAH} exhibits the strongest performance in terms of converging to a stationary point: by effective pass, the consistently large gaps are displayed between \textit{AI-SARAH} and the rest; by wall clock time, we notice that \textit{AI-SARAH} achieves the smallest $\|\nabla P(w)\|^2$ at the same time point. This validates our design, that is to leverage local Lipschitz smoothness and achieve a faster convergence than \textit{SARAH} and \textit{SARAH}+. 
 
In terms of minimizing the finite-sum functions, Figure \ref{fig:finite_sum} shows that, by effective pass, \textit{AI-SARAH} consistently outperforms \textit{SARAH} and \textit{SARAH}+ on all of the datasets with a possible exception on \textit{covtype} dataset. By wall clock time, \textit{AI-SARAH} yields a competitive performance on all of the datasets, and it delivers a stronger performance on \textit{ijcnn1} and \textit{real-sim} than \textit{SARAH}. 
\begin{figure*}
\vskip-10pt
    \centering
    \includegraphics[width=\textwidth]{newplot/legend_3.pdf}
    \vskip-8pt
    \includegraphics[width=0.95\textwidth]{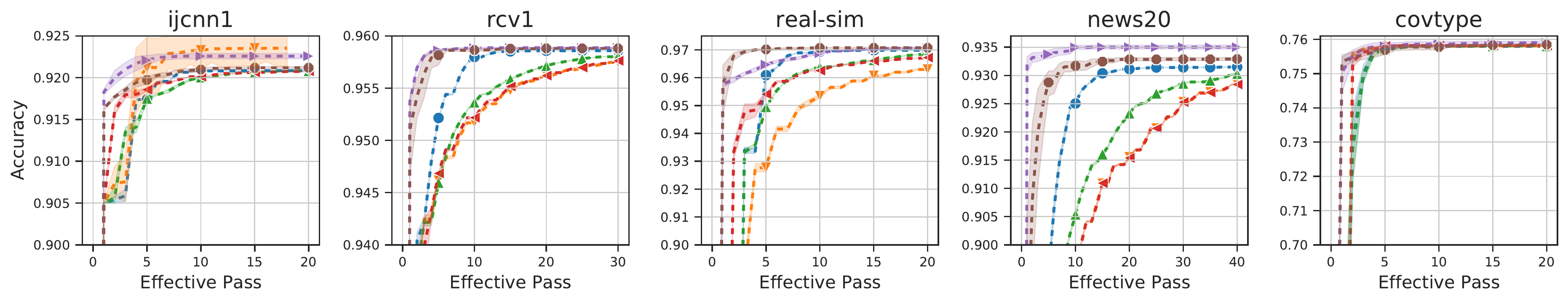}
    \includegraphics[width=0.95\textwidth]{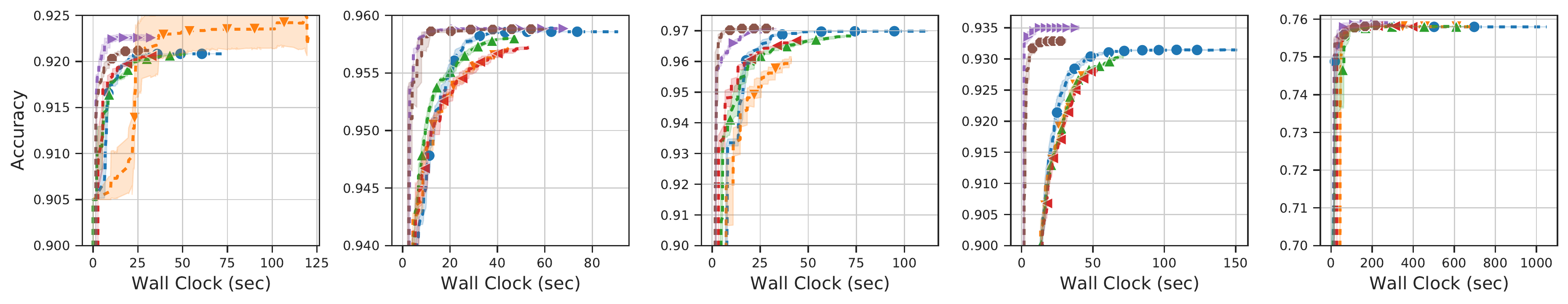}
    \vskip-5pt
    \caption{Running maximum of testing accuracy for the \textbf{regularized} case by effective pass (top row) and wall clock time (bottom row).}
    \label{fig:test_accuracy}
 \end{figure*}

 \begin{figure*}
    \centering
    \includegraphics[width=0.95\textwidth]{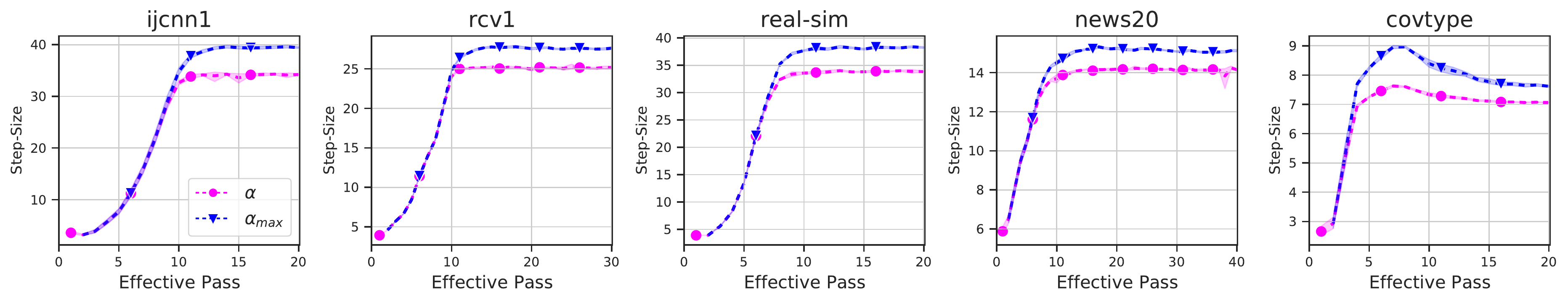}
    \vskip-10pt
    \caption{Evolution of \textit{AI-SARAH}'s step-size $\alpha$ and upper-bound $\alpha_{max}$ for the \textbf{regularized} case.}
    \label{fig:step}
    \vskip-10pt
\end{figure*} 
 
For completeness of illustration on the performance, we show the testing accuracy in Figure \ref{fig:test_accuracy}. Clearly, fine-tuned \textit{ADAM} dominates the competition. However, \textit{AI-SARAH} outperforms the other variance reduced methods on most of the datasets from both effective pass and wall clock time perspectives, and achieves the similar levels of accuracy as \textit{ADAM} does on \textit{rcv1}, \textit{real-sim} and \textit{covtype} datasets. 
\\[2pt]
Having illustrated the strong performance of \textit{AI-SARAH}, we continue the presentation by showing the trajectories of the adaptive step-size and upper-bound in Figure \ref{fig:step}.
\\[2pt]
This figure clearly shows that why \textit{AI-SARAH} can achieve such a strong performance, especially on the convergence to a stationary point. As mentioned in the previous chapters, the adaptivity is driven by the local Lipschitz smoothness. As shown in Figure \ref{fig:step}, \textit{AI-SARAH} starts with conservative step-size and upper-bound, both of which continue to increase while the algorithm progresses towards a stationary point. After a few effective passes, we observe: the step-size and upper-bound are stablized due to 
$\reg{}>0 $
(and hence strong convexity). In Appendix \ref{app:exp}, we can see that, as a result of the function being unregularized, the step-size and upper-bound could be continuously increasing due to the fact that the function is likely non-strongly convex.

\section{Conclusion}
In this paper, we propose \textit{AI-SARAH}, a practical variant of stochastic recursive gradient methods. The idea of design is simple yet powerful: by taking advantage of local Lipschitz smoothness, the step-size can be dynamically determined. With intuitive illustration and implementation details, we show how \textit{AI-SARAH} can efficiently estimate local Lipschitz smoothness and how it can be easily implemented in practice. Our algorithm is tune-free and adaptive at full scale. With extensive numerical experiment, we demonstrate that, without (tuning) any hyper-parameters, it delivers a competitive performance compared with \textit{SARAH(+)}, \textit{ADAM} and other first-order methods, all equipped with fine-tuned hyper-parameters.

\section*{Acknowledgements}
This work was partially conducted while A. Sadiev was visiting research assistant in Mohamed bin Zayed University of Artificial Intelligence (MBZUAI). 
The work of A. Sadiev was supported by a grant for research centers in the field of artificial intelligence, provided by the Analytical Center for the Government of the Russian Federation in accordance with the subsidy agreement (agreement identifier 000000D730321P5Q0002) and the agreement with the Moscow Institute of Physics and Technology dated November 1, 2021 No. 70-2021-00138.

\bibliographystyle{icml2022}

\bibliography{mybib}

\begin{thebibliography}{39}
\providecommand{\natexlab}[1]{#1}
\providecommand{\url}[1]{\texttt{#1}}
\expandafter\ifx\csname urlstyle\endcsname\relax
  \providecommand{\doi}[1]{doi: #1}\else
  \providecommand{\doi}{doi: \begingroup \urlstyle{rm}\Url}\fi

\bibitem[Bengio(2015)]{bengio2015rmsprop}
Bengio, Y.
\newblock Rmsprop and equilibrated adaptive learning rates for nonconvex
  optimization.
\newblock \emph{corr abs/1502.04390}, 2015.

\bibitem[Bottou et~al.(2018)Bottou, Curtis, and
  Nocedal]{bottou2018optimization}
Bottou, L., Curtis, F.~E., and Nocedal, J.
\newblock Optimization methods for large-scale machine learning.
\newblock \emph{SIAM Review}, 60\penalty0 (2):\penalty0 223--311, 2018.

\bibitem[Chang \& Lin(2011)Chang and Lin]{chang2011libsvm}
Chang, C.-C. and Lin, C.-J.
\newblock Libsvm: a library for support vector machines.
\newblock \emph{ACM transactions on intelligent systems and technology (TIST)},
  2\penalty0 (3):\penalty0 1--27, 2011.

\bibitem[Cutkosky \& Orabona(2020)Cutkosky and
  Orabona]{cutkosky2020momentumbased}
Cutkosky, A. and Orabona, F.
\newblock Momentum-based variance reduction in non-convex sgd.
\newblock \emph{arXiv preprint arXiv:1905.10018}, 2020.

\bibitem[Defazio(2016)]{defazio2016simple}
Defazio, A.
\newblock A simple practical accelerated method for finite sums.
\newblock In \emph{NeurIPS}, 2016.

\bibitem[Defazio et~al.(2014)Defazio, Bach, and Lacoste-Julien]{saga14}
Defazio, A., Bach, F., and Lacoste-Julien, S.
\newblock Saga: A fast incremental gradient method with support for
  non-strongly convex composite objectives.
\newblock In \emph{Advances in Neural Information Processing Systems},
  volume~27, pp.\  1646--1654. Curran Associates, Inc., 2014.

\bibitem[Duchi et~al.(2011)Duchi, Hazan, and Singer]{duchi2011adaptive}
Duchi, J., Hazan, E., and Singer, Y.
\newblock Adaptive subgradient methods for online learning and stochastic
  optimization.
\newblock \emph{Journal of machine learning research}, 12\penalty0
  (Jul):\penalty0 2121--2159, 2011.

\bibitem[Ghadimi \& Lan(2013)Ghadimi and Lan]{ghadimi2013stochastic}
Ghadimi, S. and Lan, G.
\newblock Stochastic first-and zeroth-order methods for nonconvex stochastic
  programming.
\newblock \emph{SIAM Journal on Optimization}, 23\penalty0 (4):\penalty0
  2341--2368, 2013.

\bibitem[Gower et~al.(2019)Gower, Loizou, Qian, Sailanbayev, Shulgin, and
  Richt{\'a}rik]{gower2019sgd}
Gower, R.~M., Loizou, N., Qian, X., Sailanbayev, A., Shulgin, E., and
  Richt{\'a}rik, P.
\newblock Sgd: General analysis and improved rates.
\newblock In \emph{International Conference on Machine Learning}, pp.\
  5200--5209, 2019.

\bibitem[Gower et~al.(2020{\natexlab{a}})Gower, Richt{\'a}rik, and
  Bach]{GowerRichBach2018}
Gower, R.~M., Richt{\'a}rik, P., and Bach, F.
\newblock Stochastic quasi-gradient methods: Variance reduction via jacobian
  sketching.
\newblock \emph{Mathematical Programming}, pp.\  1--58, 2020{\natexlab{a}}.

\bibitem[Gower et~al.(2020{\natexlab{b}})Gower, Sebbouh, and
  Loizou]{gower2020sgd}
Gower, R.~M., Sebbouh, O., and Loizou, N.
\newblock Sgd for structured nonconvex functions: Learning rates, minibatching
  and interpolation.
\newblock \emph{arXiv preprint arXiv:2006.10311}, 2020{\natexlab{b}}.

\bibitem[Horváth et~al.(2020)Horváth, Lei, Richtárik, and
  Jordan]{horvth2020adaptivity}
Horváth, S., Lei, L., Richtárik, P., and Jordan, M.~I.
\newblock Adaptivity of stochastic gradient methods for nonconvex optimization.
\newblock \emph{arXiv preprint arXiv:2002.05359}, 2020.

\bibitem[Johnson \& Zhang(2013)Johnson and Zhang]{svrg13}
Johnson, R. and Zhang, T.
\newblock Accelerating stochastic gradient descent using predictive variance
  reduction.
\newblock In \emph{Advances in Neural Information Processing Systems},
  volume~26, pp.\  315--323. Curran Associates, Inc., 2013.

\bibitem[Khaled et~al.(2020)Khaled, Sebbouh, Loizou, Gower, and
  Richtárik]{khaled2020unified}
Khaled, A., Sebbouh, O., Loizou, N., Gower, R.~M., and Richtárik, P.
\newblock Unified analysis of stochastic gradient methods for composite convex
  and smooth optimization.
\newblock \emph{arXiv preprint arXiv:2006.11573}, 2020.

\bibitem[Kingma \& Ba(2015)Kingma and Ba]{kingma2014adam}
Kingma, D. and Ba, J.
\newblock Adam: {A} method for stochastic optimization.
\newblock In \emph{{ICLR}}, 2015.

\bibitem[Kone\v{c}n\'{y} et~al.(2016)Kone\v{c}n\'{y}, Liu, Richt\'{a}rik, and
  Tak\'{a}\v{c}]{mS2GD}
Kone\v{c}n\'{y}, J., Liu, J., Richt\'{a}rik, P., and Tak\'{a}\v{c}, M.
\newblock Mini-batch semi-stochastic gradient descent in the proximal setting.
\newblock \emph{IEEE Journal of Selected Topics in Signal Processing},
  10\penalty0 (2):\penalty0 242--255, 2016.

\bibitem[Li \& Giannakis(2019)Li and Giannakis]{li2019adaptive}
Li, B. and Giannakis, G.~B.
\newblock Adaptive step sizes in variance reduction via regularization.
\newblock \emph{arXiv preprint arXiv:1910.06532}, 2019.

\bibitem[Li et~al.(2020)Li, Wang, and Giannakis]{pmlr-v119-li20n}
Li, B., Wang, L., and Giannakis, G.~B.
\newblock Almost tune-free variance reduction.
\newblock In \emph{Proceedings of the 37th International Conference on Machine
  Learning}, volume 119, pp.\  5969--5978. PMLR, 2020.

\bibitem[Li \& Orabona(2018)Li and Orabona]{li2018convergence}
Li, X. and Orabona, F.
\newblock On the convergence of stochastic gradient descent with adaptive
  stepsizes.
\newblock \emph{arXiv preprint arXiv:1805.08114}, 2018.

\bibitem[Liu et~al.(2019{\natexlab{a}})Liu, Jiang, He, Chen, Liu, Gao, and
  Han]{liu2019variance}
Liu, L., Jiang, H., He, P., Chen, W., Liu, X., Gao, J., and Han, J.
\newblock On the variance of the adaptive learning rate and beyond.
\newblock \emph{arXiv preprint arXiv:1908.03265}, 2019{\natexlab{a}}.

\bibitem[Liu et~al.(2019{\natexlab{b}})Liu, Han, and Huo]{LIU2019}
Liu, Y., Han, C., and Huo, T.
\newblock A class of stochastic variance reduced methods with an adaptive
  stepsize.
\newblock 2019{\natexlab{b}}.
\newblock URL
  \url{http://www.optimization-online.org/DB_FILE/2019/04/7170.pdf}.

\bibitem[Loizou \& Richt{\'a}rik(2017)Loizou and
  Richt{\'a}rik]{loizou2017linearly}
Loizou, N. and Richt{\'a}rik, P.
\newblock Linearly convergent stochastic heavy ball method for minimizing
  generalization error.
\newblock \emph{arXiv preprint arXiv:1710.10737}, 2017.

\bibitem[Loizou \& Richt{\'a}rik(2020)Loizou and
  Richt{\'a}rik]{loizou2020momentum}
Loizou, N. and Richt{\'a}rik, P.
\newblock Momentum and stochastic momentum for stochastic gradient, newton,
  proximal point and subspace descent methods.
\newblock \emph{Computational Optimization and Applications}, 77\penalty0
  (3):\penalty0 653--710, 2020.

\bibitem[Loizou et~al.(2020)Loizou, Vaswani, Laradji, and
  Lacoste-Julien]{loizou2020stochastic}
Loizou, N., Vaswani, S., Laradji, I., and Lacoste-Julien, S.
\newblock Stochastic polyak step-size for sgd: An adaptive learning rate for
  fast convergence.
\newblock \emph{arXiv preprint arXiv:2002.10542}, 2020.

\bibitem[Moulines \& Bach(2011)Moulines and Bach]{moulines2011non}
Moulines, E. and Bach, F.~R.
\newblock Non-asymptotic analysis of stochastic approximation algorithms for
  machine learning.
\newblock In \emph{Advances in Neural Information Processing Systems}, pp.\
  451--459, 2011.

\bibitem[Needell et~al.(2016)Needell, Srebro, and Ward]{needell2014stochastic}
Needell, D., Srebro, N., and Ward, R.
\newblock Stochastic gradient descent, weighted sampling, and the randomized
  kaczmarz algorithm.
\newblock \emph{Mathematical Programming, Series A}, 155\penalty0 (1):\penalty0
  549--573, 2016.

\bibitem[Nemirovski \& Yudin(1983)Nemirovski and Yudin]{NemYudin1983book}
Nemirovski, A. and Yudin, D.~B.
\newblock \emph{Problem complexity and method efficiency in optimization}.
\newblock Wiley Interscience, 1983.

\bibitem[Nemirovski et~al.(2009)Nemirovski, Juditsky, Lan, and
  Shapiro]{Nemirovski-Juditsky-Lan-Shapiro-2009}
Nemirovski, A., Juditsky, A., Lan, G., and Shapiro, A.
\newblock Robust stochastic approximation approach to stochastic programming.
\newblock \emph{SIAM Journal on Optimization}, 19\penalty0 (4):\penalty0
  1574--1609, 2009.

\bibitem[Nesterov(2003)]{nesterov2003introductory}
Nesterov, Y.
\newblock \emph{Introductory lectures on convex optimization: A basic course},
  volume~87.
\newblock Springer Science \& Business Media, 2003.

\bibitem[Nguyen et~al.(2018)Nguyen, Nguyen, van Dijk, Richt\'{a}rik,
  Scheinberg, and Tak\'{a}\v{c}]{pmlr-v80-nguyen18c}
Nguyen, L., Nguyen, P.~H., van Dijk, M., Richt\'{a}rik, P., Scheinberg, K., and
  Tak\'{a}\v{c}, M.
\newblock {SGD} and hogwild! {C}onvergence without the bounded gradients
  assumption.
\newblock In \emph{Proceedings of the 35th International Conference on Machine
  Learning}, volume~80 of \emph{Proceedings of Machine Learning Research}, pp.\
   3750--3758. PMLR, 2018.

\bibitem[Nguyen et~al.(2017)Nguyen, Liu, Scheinberg, and
  Tak{\'a}{\v{c}}]{sarah17}
Nguyen, L.~M., Liu, J., Scheinberg, K., and Tak{\'a}{\v{c}}, M.
\newblock Sarah: A novel method for machine learning problems using stochastic
  recursive gradient.
\newblock In \emph{Proceedings of the 34th International Conference on Machine
  Learning (ICML 2000)}, volume~70, pp.\  2613--2621, International Convention
  Centre, Sydney, Australia, 2017. PMLR.

\bibitem[Robbins \& Monro(1951)Robbins and Monro]{robbins1951stochastic}
Robbins, H. and Monro, S.
\newblock A stochastic approximation method.
\newblock \emph{The Annals of Mathematical Statistics}, pp.\  400--407, 1951.

\bibitem[Schmidt et~al.(2017)Schmidt, Le~Roux, and Bach]{schmidt2017minimizing}
Schmidt, M., Le~Roux, N., and Bach, F.
\newblock Minimizing finite sums with the stochastic average gradient.
\newblock \emph{Math. Program.}, 162\penalty0 (1-2):\penalty0 83--112, 2017.

\bibitem[Shalev-Shwartz et~al.(2007)Shalev-Shwartz, Singer, and
  Srebro]{Pegasos}
Shalev-Shwartz, S., Singer, Y., and Srebro, N.
\newblock Pegasos: primal estimated subgradient solver for {SVM}.
\newblock In \emph{24th International Conference on Machine Learning}, pp.\
  807--814, 2007.

\bibitem[Sutskever et~al.(2013)Sutskever, Martens, Dahl, and
  Hinton]{sutskever2013importance}
Sutskever, I., Martens, J., Dahl, G., and Hinton, G.
\newblock On the importance of initialization and momentum in deep learning.
\newblock In \emph{International conference on machine learning}, pp.\
  1139--1147. PMLR, 2013.

\bibitem[Tan et~al.(2016)Tan, Ma, Dai, and Qian]{tan2016barzilai}
Tan, C., Ma, S., Dai, Y.-H., and Qian, Y.
\newblock Barzilai-borwein step size for stochastic gradient descent.
\newblock In \emph{Proceedings of the 30th International Conference on Neural
  Information Processing Systems}, pp.\  685--693, 2016.

\bibitem[Vaswani et~al.(2019)Vaswani, Mishkin, Laradji, Schmidt, Gidel, and
  Lacoste-Julien]{Schmidt19}
Vaswani, S., Mishkin, A., Laradji, I., Schmidt, M., Gidel, G., and
  Lacoste-Julien, S.
\newblock Painless stochastic gradient: Interpolation, line-search, and
  convergence rates.
\newblock In Wallach, H., Larochelle, H., Beygelzimer, A., d\textquotesingle
  Alch\'{e}-Buc, F., Fox, E., and Garnett, R. (eds.), \emph{Advances in Neural
  Information Processing Systems}, volume~32, pp.\  3732--3745. Curran
  Associates, Inc., 2019.

\bibitem[Ward et~al.(2019)Ward, Wu, and Bottou]{ward2019adagrad}
Ward, R., Wu, X., and Bottou, L.
\newblock Adagrad stepsizes: Sharp convergence over nonconvex landscapes.
\newblock In \emph{International Conference on Machine Learning}, pp.\
  6677--6686, 2019.

\bibitem[Yang et~al.(2021)Yang, Chen, and Wang]{YANG2021}
Yang, Z., Chen, Z., and Wang, C.
\newblock Accelerating mini-batch sarah by step size rules.
\newblock \emph{Information Sciences}, 2021.
\newblock ISSN 0020-0255.
\newblock \doi{https://doi.org/10.1016/j.ins.2020.12.075}.

\end{thebibliography}
 
\clearpage

\clearpage
\onecolumn
\icmltitle{Supplementary Material}
\appendix

The Appendix is organized as follows. In Chapter \ref{app:line-segment}, we present the technical details of theoretical analysis in Chapter \ref{sec:theoretical-analysis} of the main paper. In Chapter~\ref{app:exp}, we present extended details on the design, implementation and results of our numerical experiments. In Chapters~\ref{app:alternative} and \ref{AppendixProofs}, we present an alternative theoretical analysis for investigating the benefit of using local Lipschitz smoothness to derive an adaptive step-size. 

\section{Technical Results and Proofs}\label{app:line-segment}
We consider finite-sum optimization problem
\begin{equation}
    \min_{w \in \cR^d} \left[P(w) \defeq \frac{1}{n}\sum^n_{i = 1} f_i(w)\right].
\end{equation}

\begin{assumption}
\label{smoothness_on_ls}
For $t\geq 0$, each $f_i$ is $L^{t}_i$-smooth on the line segment $\Delta = \left\{w\in \cR^d\;|\; w = w_t - \eta v_t, \forall \eta \in \left[0, \frac{1}{L^{t}_i}\right] \right\}$  and convex. For simplicity, we denote $$L^{t} = \max_{i \in [n]}L^{t}_i,~\Bar{L}^{t} = \frac{1}{n}\sum^n_i L^{ t}_i,~ \Bar{L}_M = \max_{t \in \{0,1,...,m\}} \Bar{L}^{t}.$$

Note that in Chapter \ref{sec:theoretical-analysis} of the main paper, we use $L^t$ universally for both maximum and average value of parameters of local Lipschitz smoothness. In this chapter, as we will present Algorithm \ref{algo:aisarah-theory} in two specific forms: importance sampling version (see Algorithm \ref{algo:aisarah-theory-importance}) and uniform sampling version (see Algorithm \ref{algo:aisarah-theory-uniform}), we use a different notation on the average, i.e., $\bar L^t = \frac{1}{n}\sum^n_i L^t_i$.

\end{assumption}
\begin{assumption}
\label{mu_convexity_ls}
Function $P$ is $\mu$-strongly convex.
\end{assumption}

\begin{definition}
Fix a outer loop $k\geq 1$ and consider Algorithms \ref{algo:aisarah-theory}, \ref{algo:aisarah-theory-importance} and \ref{algo:aisarah-theory-uniform} with an inner loop size $m$, we define a discrete probability distribution at $t \geq 1$ for all $i \in \n$, $p^{t}_i = \frac{L^{t}_i}{\sum^n_{i = 1}L^{t}_i},$
and probabilities $q_t$ for all $t \geq 0$, $q_t = \frac{\eta_t}{\mathcal{H}},~\text{where}~\mathcal{H} = \sum^m_{j = 0}\eta_j.$
\end{definition}

\subsection{Theoretical-AI-SARAH with Importance Sampling}
We present the importance sampling algorithm in Algorithm \ref{algo:aisarah-theory-importance}. Now, let us start by presenting the following lemmas. 
\begin{algorithm}
\caption{\textit{Theoretical-AI-SARAH with Importance Sampling}}
	\begin{algorithmic}[1]
	\STATE \textbf{Parameter:} Inner loop size $m$
		\STATE \textbf{Initialize:} $\tilde{w}_0$
		\FOR{$k = 1, 2, ...$}
		\STATE $w_0 = \tilde{w}_{k-1}$
		\STATE $v_0 = \nabla P(w_0)$ 
		\FOR{$i \in \n$}
		\STATE $L^0_i = \max_{\eta \in [0, \frac{1}{L^0_i}]}\|\nabla^2 f_i(w_0 - \eta v_0)\|$
		\ENDFOR
        \STATE $\bar L^0 = \frac{1}{n}\sum^n_{i = 1} L^0_i$ and $\eta_0 = \frac{1}{\bar L^0}$
        \FOR{$t = 1, ... , m$}
        \STATE $w_t = w_{t-1} - \eta_{t-1} v_{t-1}$
		\STATE Sample $i_t$ from $\n$ with probability $p^{t-1}_i$
        \STATE  $v_t = v_{t-1} + \tfrac{1}{np^{t-1}_i} \left(\nabla f_{i_t}(w_t) - \nabla f_{i_t}(w_{t-1})\right)$
        \FOR{$i \in \n$}
        \STATE $L^t_i = \max_{\eta \in [0, \frac{1}{L^t_i}]}\|\nabla^2 f_i(w_t - \eta v_t)\|$
        \ENDFOR
         \STATE $\bar L^t = \frac{1}{n}\sum^n_{i = 1} L^t_i$
        \STATE $\eta_t = \min\left\{\frac{1}{\bar L^{t}}, \frac{\bar L^{t-1}}{\bar L^{t}}\eta_{t-1}\right\}$
		\ENDFOR
        \STATE \textbf{Set} $\tilde{w}_k = w_t$ where $t$ is chosen with probability $q_t$ from $\{0,1,...,m\}$
		\ENDFOR
	\end{algorithmic}
	\label{algo:aisarah-theory-importance}
\end{algorithm}

\begin{lemma}
\label{lemmaTechnicalLam_1_ls}
Consider $v_t$ defined in Algorithm~\ref{algo:aisarah-theory-importance}. Then for any $t\geq1$ in Algorithm~\ref{algo:aisarah-theory-importance}, it holds that
\begin{align*}
\EE[\|\nabla P(w_t) - v_t\|^2] = \sum_{j=1}^t \EE[\|v_j - v_{j-1}\|^2] - \sum_{j=1}^t\EE[\|\nabla P(w_j) - \nabla P(w_{j-1})\|^2].
\end{align*}
\end{lemma}
\begin{proof}
Let $\Exp_j$ denote the expectation by conditioning on the information $w_0,w_1, \dots ,w_j $ as well as $v_0,v_1,\dots,v_{j-1}$. Then,
\begin{align*}
    \Exp_j[\|\nabla P(w_j) - v_j\|^2] &= 
    \Exp_j\left[\|\left(\nabla P(w_{j-1}) - v_{j-1}\right) 
    + \left(\nabla P(w_j) - \nabla P(w_{j-1})\right) 
    - (v_j - v_{j-1})\|^2\right]\\
    &=\Exp_j[\|\nabla P(w_{j-1}) - v_{j-1}\|^2]
    +\|\nabla P(w_j) - \nabla P(w_{j-1})\|^2
    + \Exp_j[\|v_j - v_{j-1}\|^2]\\
    &\quad + 2 \left\langle\nabla P(w_{j-1}) - v_{j-1}, \nabla P(w_j)-\nabla P(w_{j-1})\right\rangle\\
    &\quad - 2\left\langle\nabla P(w_{j-1}) - v_{j-1}, \Exp_j[v_j - v_{j-1}]\right\rangle\\
    &\quad - 2\left\langle\nabla P(w_j) - \nabla P(w_{j-1}), \Exp_j[v_j - v_{j-1}]\right\rangle\\
    &= \Exp_j[\|\nabla P(w_{j-1}) - v_{j-1}\|^2]
    - \|\nabla P(w_j) - \nabla P(w_{j-1})\|^2
    + \Exp_j[\|v_j - v_{j-1}\|^2], \tagthis \label{lemma:A4-proof}
\end{align*}
where the last equality follows from
\begin{align*}
    \Exp_j[v_j - v_{j-1}]&=\Exp_j[\frac{1}{np^{j-1}_{i_j}}\left(\nabla f_{i_j}(w_j) - \nabla  f_{i_j}(w_{j-1})\right)]\\
    &=\sum^n_{i_j}\frac{p^{j-1}_{i_j}}{np^{j-1}_{i_j}}\left(\nabla f_{i_j}(w_j) - \nabla  f_{i_j}(w_{j-1})\right)\\
    &=\nabla P(w_j) - \nabla P(w_{j-1}). \tagthis \label{lemma:trivial}
\end{align*}
By taking expectation of \eqref{lemma:A4-proof}, we have
\begin{align*}
    \Exp[\|\nabla P(w_j) - v_j\|^2] = \Exp[\|\nabla P(w_{j-1}) - v_{j-1}\|^2]
    - \Exp[\|\nabla P(w_j) - \nabla P(w_{j-1})\|^2]
    + \Exp[\|v_j - v_{j-1}\|^2].
\end{align*}
By summing it over $j=1,...,t$ and note that $\|\nabla P(v_0) - v_0\|^2=0$, we have 
\begin{align*}
\Exp[\|\nabla P(w_t) - v_t\|^2] = \sum_{j=1}^t \Exp[\|v_j - v_{j-1}\|^2] - \sum_{j=1}^t\Exp[\|\nabla P(w_j) - \nabla P(w_{j-1})\|^2].
\end{align*}
\end{proof}

\begin{lemma}
\label{Lemma_BoundGradient_ls}
Fix a outer loop $k \geq 1$ and consider Algorithm~\ref{algo:aisarah-theory-importance} with $\eta_t \leq 1/\bar{L}^{t}$ for any $t \in [m]$. Under Assumption \ref{smoothness_on_ls},
\begin{equation*}
    \sum_{t=0}^{m}\frac{\eta_t}{2}\EE[\| \nabla P(w_{t})\|^2]
    \leq  \EE[P(w_0)-P(w^*)]  +  \sum_{t=0}^{m}\frac{\eta_t}{2}\EE[\| \nabla P(w_{t})-v_{t}\|^2].
\end{equation*}
\end{lemma}

\begin{proof}
By Assumption \ref{smoothness_on_ls} and the update rule $w_t = w_{t-1} - \eta_{t-1} v_{t-1}$ of Algorithm~\ref{algo:aisarah-theory-importance}, we obtain
\begin{eqnarray*}
   P(w_{t}) &\leq & P(w_{t-1})- \eta_{t-1} \langle \nabla P(w_{t-1}), v_{t-1} \rangle+ \frac{\bar{L}^{t-1}}{2} \eta_{t-1}^2 \|v_{t-1}\|^2 \notag\\
   &= & P(w_{t-1})- \frac{\eta_{t-1}}{2}\| \nabla P(w_{t-1})\|^2 + \frac{\eta_{t-1}}{2}\| \nabla P(w_{t-1})-v_{t-1}\|^2 - \left(\frac{\eta_{t-1}}{2} - \frac{\bar{L}^{t-1}}{2} \eta_{t-1}^2 \right) \|v_{t-1}\|^2,
\end{eqnarray*}
where, in the equality above, we use the fact that $\langle a,b\rangle = \frac{1}{2}(\|a\|^2 + \|b\|^2 - \|a-b\|^2)$.

By assuming that $\eta_{t-1} \leq \frac{1}{\bar{L}^{t-1}}$, it holds that $ \left(1 - \bar{L}^{t-1} \eta_{t-1} \right)\geq0$, $\forall t \in [m]$. Thus,
\begin{eqnarray*}
 \frac{\eta_{t-1}}{2} \| \nabla P(w_{t-1})\|^2 
   &\leq&  [P(w_{t-1})-P(w_{t}) ] + \frac{\eta_{t-1}}{2}\| \nabla P(w_{t-1})-v_{t-1}\|^2 - \frac{\eta_{t-1}}{2} \left(1 - \bar{L}^{t-1} \eta_{t-1}\right) \|v_{t-1}\|^2.
\end{eqnarray*}
By taking expectations
\begin{eqnarray*}
 \EE[\frac{\eta_{t-1}}{2}\| \nabla P(w_{t-1})\|^2]
   &\leq&   \EE[P(w_{t-1})]-\EE[P(w_{t})]  + \frac{\eta_{t-1}}{2}\EE[\| \nabla P(w_{t-1})-v_{t-1}\|^2]\\ &&-  \frac{\eta_{t-1}}{2}\left(1 - \bar{L}^{t-1} \eta_t \right) \EE[\|v_{t-1}\|^2]\notag\\
   &\overset{\eta_{t-1} \leq \frac{1}{\bar{L}^{t-1}}}{\leq}&  \EE[P(w_{t-1})]-\EE[P(w_{t})]  + \frac{\eta_{t-1}}{2}\EE[\| \nabla P(w_{t-1})-v_{t-1}\|^2].
\end{eqnarray*}
Summing over $t=1,2,\dots, m+1$, we have
\begin{eqnarray*}
 \sum_{t=1}^{m+1}\frac{\eta_{t-1}}{2}\EE[\| \nabla P(w_{t-1})\|^2]
   &\leq&   \sum_{t=1}^{m+1} \EE[P(w_{t-1})-P(w_{t})] +  \sum_{t=1}^{m+1}\frac{\eta_{t-1}}{2}\EE[\| \nabla P(w_{t-1})-v_{t-1}\|^2]\notag\\
   &=&  \EE[P(w_{0})-P(w_{m+1})] +  \sum_{t=1}^{m+1}\frac{\eta_{t-1}}{2} \EE[\| \nabla P(w_{t-1})-v_{t-1}\|^2\notag\\
    &\leq &   \EE[P(w_{0})-P(w_{*})] +  \sum_{t=1}^{m+1}\frac{\eta_{t-1}}{2}\EE[\| \nabla P(w_{t-1})-v_{t-1}\|^2],
\end{eqnarray*}
where the last inequality holds since $w^*$ is the global minimizer of $P.$

The last expression can be equivalently written as
\begin{eqnarray*}
 \sum_{t=0}^{m}\frac{\eta_{t}}{2}\EE[\| \nabla P(w_{t})\|^2]
    &\leq &  \EE[P(w_{0})-P(w_{*})] + \sum_{t=0}^{m}\frac{\eta_{t}}{2}\EE[\| \nabla P(w_{t})-v_{t}\|^2],
 \end{eqnarray*}
 which completes the proof.
 \end{proof}

\begin{lemma}
\label{Lemma_boundonDifference2_ls}
Consider Algorithm~\ref{algo:aisarah-theory-importance} with $\eta_t =  \min\left\{\frac{1}{\bar{L}^{t}}, \frac{\bar{L}^{t-1}}{\bar{L}^{t}}\eta_{t-1}\right\}$.  Suppose $f_i$ is convex for all $i \in [n]$. Then, under Assumption~\ref{smoothness_on_ls}, for any $t\geq 1$,
\begin{align*}
    \EE[\|\nabla P(w_t) - v_t\|^2] \leq \left(\frac{\eta_{0} \Bar{L}^{0}}{2-\eta_{0} \Bar{L}^{0}}\right) \EE[\|v_{0}\|^2].
\end{align*}
\end{lemma}

\begin{proof}
\begin{eqnarray*}
 \EE_j\left[\|v_j\|^2\right] &\leq&   \EE_j\left[\|v_{j-1} - \frac{1}{np^{j-1}_i}\left(\nabla f_{i_j}(w_{j-1}) - \nabla f_{i_j}(w_j) \right)\|^2\right]\notag\\
   &=& \|v_{j-1}\|^2 +  \EE_j\left[\frac{1}{\left(np^{j-1}_i\right)^2}\|\nabla f_{i_j}(w_{j-1}) - \nabla f_{i_j}(w_j)\|^2\right] 
   \\&&\quad-  \EE_j\left[\frac{2}{\eta_{j-1}np^{j-1}_i} \left\langle\nabla f_{i_j}(w_{j-1}) - \nabla f_{i_j}(w_j), w_{j-1} - w_j\right\rangle\right]\notag\\
   &\leq& \|v_{j-1}\|^2 +  \EE_j\left[\frac{1}{(np^{j-1}_i)^2}\|\nabla f_{i_j}(w_{j-1}) - \nabla f_{i_j}(w_j)\|^2\right] 
   \\&&-  \EE_j\left[\frac{2}{\eta_{j-1}np^{j-1}_i L^{j-1}_{i_j}} \|\nabla f_{i_j}(w_{j-1}) - \nabla f_{i_j}(w_j)\|^2\right].\notag\\
\end{eqnarray*}
For each outer loop $k\geq 1$, it holds that $L^{j-1}_{i_j}  = np^{j-1}_i \Bar{L}^{j-1} $. Thus,
\begin{eqnarray*}
    \EE_j[\|v_j\|^2] &\leq&  \|v_{j-1}\|^2 +  \EE_j\left[\left\|\frac{1}{np^{j-1}_i}\left(\nabla f_{i_j}(w_{j-1}) - \nabla f_{i_j}(w_j)\right)\right\|^2\right] 
    \\&&-  \frac{2}{\eta_{j-1} \Bar{L}^{j-1}}\EE_j\left[ \left\|\frac{1}{np^{j-1}_i}\left(\nabla f_{i_j}(w_{j-1}) - \nabla f_{i_j}(w_j)\right)\right\|^2\right]\notag\\
    &=&  \|v_{j-1}\|^2 + \left(1- \frac{2}{\eta_{j-1} \Bar{L}^{j-1}}\right)  \EE_j\left[\left\|\frac{1}{np^{j-1}_i}\left(\nabla f_{i_j}(w_{j-1}) - \nabla f_{i_j}(w_j)\right)\right\|^2\right]\notag\\
    &=&  \|v_{j-1}\|^2 + \left(1- \frac{2}{\eta_{j-1} \Bar{L}^{j-1}}\right)  \EE_j\left[\left\|v_j - v_{j-1}\right\|^2\right].\notag\\
    &\leq&  \|v_{j-1}\|^2 + \left(1- \frac{2}{\eta_{j-1} \Bar{L}^{j-1}}\right)  \EE_j\left[\left\|v_j - v_{j-1}\right\|^2\right].\notag\\
\end{eqnarray*}

By rearranging, taking expectations again, and assuming that $\eta_{j-1} < 2/ \Bar{L}^{j-1}$ for any $j$ from $1$ to $t+1$
\begin{eqnarray*}
 \EE[\|v_j - v_{j-1}\|^2] &\leq& \EE\left[\left(\frac{\eta_{j-1} \Bar{L}^{j-1}}{2-\eta_{j-1} \Bar{L}^{j-1}}\right) \left[\|v_{j-1}\|^2 -  \|v_j\|^2\right]\right].  \notag\\
\end{eqnarray*}

By summing the above inequality over $j=1,\dots, t$ $(t\geq1)$, we have
\begin{eqnarray}
\label{eq:lemma-6}
 \sum_{j=1}^t  \EE[\|v_j - v_{j-1}\|^2] &\leq&   \sum_{j=1}^t\EE\left[\left(\frac{\eta_{j-1} \Bar{L}^{j-1}}{2-\eta_{j-1} \Bar{L}^{j-1}}\right)\left[\|v_{j-1}\|^2 - \|v_j\|^2 \right]\right]  \notag\\
 &=&  \left(\frac{\eta_{0} \Bar{L}^{0}}{2-\eta_{0} \Bar{L}^{0}}\right) \EE\left[\|v_{0}\|^2\right] - \left(\frac{\eta_{t-1} \Bar{L}^{t-1}}{2-\eta_{t-1} \Bar{L}^{t-1}}\right) \EE\left[\|v_{t}\|^2\right] \notag\\&&- \sum_{j=1}^{t-1}\EE\left[\left(\frac{\eta_{j-1} \Bar{L}^{j-1}}{2-\eta_{j-1} \Bar{L}^{j-1}} - \frac{\eta_{j} \Bar{L}^{j}}{2-\eta_{j} \Bar{L}^{j}}\right)\|v_{j}\|^2  \right]\notag\\
 &\leq& \left(\frac{\eta_{0} \Bar{L}^{0}}{2-\eta_{0} \Bar{L}^{0}}\right) \EE\left[\|v_{0}\|^2\right].
\end{eqnarray}

Now, by using Lemma~\ref{lemmaTechnicalLam_1_ls} and \eqref{eq:lemma-6}, we obtain
\begin{eqnarray*}
\EE[\|\nabla P(w_t) - v_t\|^2] &{\leq}& \sum_{j=1}^t \EE\left[\|v_j - v_{j-1}\|^2\right]\notag\\& {\leq} &\left(\frac{\eta_{0} \Bar{L}^{0}}{2-\eta_{0} \Bar{L}^{0}}\right) \EE\left[\|v_{0}\|^2\right].
\end{eqnarray*}
\end{proof}

Using the above lemmas, we can present one of our main results in the following theorem.
\begin{theorem}\label{the:converge_ls} Suppose that Assumptions \ref{smoothness_on_ls}, {\ref{mu_convexity_ls}}, holds. 
Let us define 
\begin{equation*}
    \bar \sigma^k_m=\left(\frac{1}{\mu \mathcal{H}} +\frac{\eta_{0} \Bar{L}^{0}}{2-\eta_{0} \Bar{L}^{0}}\right),
\end{equation*}
and select $m$ and $\eta$ such that $\bar \sigma^k_m<1$. Then, Algorithm \ref{algo:aisarah-theory-importance} converges as follows
\begin{equation*}
    \EE[\| \nabla P(\tilde{w}_{k})\|^2 \leq \left(\prod^{k}_{l=1}\bar \sigma^l_m\right) \|\nabla P(\tilde{w}_{0})\|^2.
\end{equation*}
\end{theorem}
\begin{proof}
Since $v_0=\nabla P(w_0)$ implies $\|\nabla P(w_0)-v_0 \|^2=0$, then by Lemma~\ref{Lemma_boundonDifference2_ls}, we obtain
\begin{equation*}
\sum_{t=0}^m \frac{\eta_t}{\mathcal{H}}\EE[\|\nabla P(w_t) - v_t\|^2]  \leq \left(\frac{\eta_{0} \Bar{L}^{0}}{2-\eta_{0} \Bar{L}^{0}}\right) \EE[\|v_{0}\|^2].
\end{equation*}
Combine this with Lemma~\ref{Lemma_BoundGradient_ls}, we have that
\begin{eqnarray*}
 \sum_{t=0}^{m}\frac{\eta_t}{\mathcal{H}}\EE[\| \nabla P(w_{t})\|^2]
    &\leq &  \frac{2}{\mathcal{H}} \EE[P(w_{0})-P(w_{*})] +  \sum_{t=0}^{m}\frac{\eta_t}{\mathcal{H}}\EE[\| \nabla P(w_{t})-v_{t}\|^2]\notag\\
    &{\leq} &  \frac{2}{\mathcal{H}} \EE[P(w_{0})-P(w_{*})] + \frac{\eta_{0} \Bar{L}^{0}}{2-\eta_{0} \Bar{L}^{0}} \EE[\|v_{0}\|^2].
\end{eqnarray*}
Since we consider one outer loop, with $k\geq1$, we have $v_0=\nabla P(w_0)=\nabla P(\tilde{w}_{k-1})$ and $\tilde{w}_{k}=w_t$, where $t$ is drawn at random from $\{0,1,\dots, m\}$ with probabilities $q_t$. Therefore, the following holds,
\begin{eqnarray*}
 \EE[\| \nabla P(\tilde{w}_{k})\|^2]&=&\sum_{t=0}^{m}\sum_{t=0}^{m}\frac{\eta_t}{\mathcal{H}}\EE[\| \nabla P(w_{t})\|^2]\notag\\
    &{\leq} &  \frac{2}{\mathcal{H}} \EE[P(\tilde{w}_{k-1})-P(w_{*})] + \frac{\eta_{0} \Bar{L}^{0}}{2-\eta_{0} \Bar{L}^{0}} \EE[\|\nabla P(\tilde{w}_{k-1})\|^2]\notag\\
    &\leq & \left(\frac{1}{\mu \mathcal{H}} +\frac{\eta_{0} \Bar{L}^{0}}{2-\eta_{0} \Bar{L}^{0}}\right) \EE[\|\nabla P(\tilde{w}_{k-1})\|^2].\notag\\
\end{eqnarray*}
Let us define $\bar{\sigma}^k_m=\left(\frac{1}{\mu \mathcal{H}} +\frac{\eta_{0} \Bar{L}^{0}}{2-\eta_{0} \Bar{L}^{0}}\right)$,
then the above expression can be written as
\begin{eqnarray*}
 \EE[\| \nabla P(\tilde{w}_{k})\|^2]
    &\leq & \bar{\sigma}^k_m \EE[\|\nabla P(\tilde{w}_{k-1})\|^2].\notag\\
\end{eqnarray*}
By expanding the recurrence, we obtain
\begin{eqnarray*}
 \EE[\| \nabla P(\tilde{w}_{k})\|^2]
    &\leq &  \left(\prod^{k}_{l=1}\bar{\sigma}^l_m\right)\|\nabla P(\tilde{w}_{0})\|^2.\notag\\
\end{eqnarray*}
This completes the proof.
\end{proof}

\subsection{Theoretical-AI-SARAH with Uniform Sampling}
We present the uniform sampling algorithm in Algorithm \ref{algo:aisarah-theory-uniform}. Now, let us start by presenting the following lemmas.
\begin{algorithm}[ht!]
\caption{\textit{Theoretical-AI-SARAH with Uniform Sampling}}
	\begin{algorithmic}[1]
	\STATE \textbf{Parameter:} Inner loop size $m$
		\STATE \textbf{Initialize:} $\tilde{w}_0$
		\FOR{k = 1, 2, ...}
		\STATE $w_0 = \tilde{w}_{k-1}$
		\STATE $v_0 = \nabla P(w_0)$ 
		\FOR{$i \in \n$}
		\STATE $L^0_i = \max_{\eta \in [0, \frac{1}{L^0_i}]}\|\nabla^2 f_i(w_0 - \eta v_0)\|$
		\ENDFOR
        \STATE $L^0 = \max_{i \in \n} L^0_i$ and $\eta_0 = \frac{1}{L^0}$
        \FOR{t = 1, ... , m}
        \STATE $w_t = w_{t-1} - \eta_{t-1} v_{t-1}$
		\STATE Sample $i_t$ uniformly at random from $\n$
        \STATE  $v_t = v_{t-1} + \nabla f_{i_t}(w_t) - \nabla f_{i_t}(w_{t-1})$
        \FOR{$i \in \n$}
        \STATE $L^t_i = \max_{\eta \in [0, \frac{1}{L^t_i}]}\|\nabla^2 f_i(w_t - \eta v_t)\|$
        \ENDFOR
         \STATE $L^t = \max_{i \in \n} L^t_i$
        \STATE $\eta_t = \min\left\{\frac{1}{L^{t}}, \frac{L^{t-1}}{ L^{t}}\eta_{t-1}\right\}$
		\ENDFOR
        \STATE \textbf{Set} $\tilde{w}_k = w_t$ where $t$ is chosen with probability $q_t$ from $\{0,1,...,m\}$
		\ENDFOR
	\end{algorithmic}
	\label{algo:aisarah-theory-uniform}
\end{algorithm}

		
		

\begin{lemma}
\label{lemmaTechnical_1_uniform}
Consider $v_t$ defined in Algorithm~\ref{algo:aisarah-theory-uniform}. Then for any $t\geq1$ in Algorithm~\ref{algo:aisarah-theory-uniform}, it holds that
\begin{align*}
\EE[\|\nabla P(w_t) - v_t\|^2] = \sum_{j=1}^t \EE[\|v_j - v_{j-1}\|^2] - \sum_{j=1}^t\EE[\|\nabla P(w_j) - \nabla P(w_{j-1})\|^2].
\end{align*}
\end{lemma}
\begin{proof}
The proof is the same as that of Lemma \ref{lemmaTechnicalLam_1_ls} except that we have $p^{j-1}_{i_j} = \frac{1}{n}$ in \eqref{lemma:trivial} for uniform sampling.
\end{proof}

\begin{lemma}
\label{Lemma_BoundGradient_ls_uniform}
Fix a outer loop $k \geq 1$ and consider Algorithm~\ref{algo:aisarah-theory-uniform} with $\eta_t \leq 1/L^{t}$ for any $t \in [m]$. Under Assumption \ref{smoothness_on_ls},
\begin{equation*}
    \sum_{t=0}^{m}\frac{\eta_t}{2}\EE[\| \nabla P(w_{t})\|^2]
    \leq  \EE[P(w_0)-P(w^*)]  +  \sum_{t=0}^{m}\frac{\eta_t}{2}\EE[\| \nabla P(w_{t})-v_{t}\|^2].
\end{equation*}
\end{lemma}
\begin{proof}
By Assumption \ref{smoothness_on_ls} and the update rule $w_t = w_{t-1} - \eta_{t-1} v_{t-1}$ of Algorithm~\ref{algo:aisarah-theory-uniform}, we obtain
\begin{eqnarray*}
   P(w_{t}) &\leq & P(w_{t-1})- \eta_{t-1} \langle \nabla P(w_{t-1}), v_{t-1} \rangle+ \frac{L^{t-1}}{2} \eta_{t-1}^2 \|v_{t-1}\|^2 \notag\\
   &= & P(w_{t-1})- \frac{\eta_{t-1}}{2}\| \nabla P(w_{t-1})\|^2 + \frac{\eta_{t-1}}{2}\| \nabla P(w_{t-1})-v_{t-1}\|^2 - \left(\frac{\eta_{t-1}}{2} - \frac{L^{t-1}}{2} \eta_{t-1}^2 \right) \|v_{t-1}\|^2,
\end{eqnarray*}
where, in the equality above, we use the fact that $\langle a,b\rangle = \frac{1}{2}(\|a\|^2 + \|b\|^2 - \|a-b\|^2)$.

By assuming that $\eta_{t-1} \leq \frac{1}{L^{t-1}}$, it holds  $ \left(1 - L^{t-1} \eta_{t-1} \right)\geq0$, $\forall t \in [m]$. Thus,
\begin{eqnarray*}
 \frac{\eta_{t-1}}{2} \| \nabla P(w_{t-1})\|^2 
   &\leq&  [P(w_{t-1})-P(w_{t}) ] + \frac{\eta_{t-1}}{2}\| \nabla P(w_{t-1})-v_{t-1}\|^2 - \frac{\eta_{t-1}}{2} \left(1 - L^{t-1} \eta_{t-1}\right) \|v_{t-1}\|^2.
\end{eqnarray*}
By taking expectations
\begin{eqnarray*}
 \EE[\frac{\eta_{t-1}}{2}\| \nabla P(w_{t-1})\|^2]
   &\leq&   \EE[P(w_{t-1})]-\EE[P(w_{t})]  + \frac{\eta_{t-1}}{2}\EE[\| \nabla P(w_{t-1})-v_{t-1}\|^2]\\ &&-  \frac{\eta_{t-1}}{2}\left(1 - L^{t-1} \eta_{t-1} \right) \EE[\|v_{t-1}\|^2]\notag\\
   &\overset{\eta_{t-1} \leq \frac{1}{L^{t-1}}}{\leq}&  \EE[P(w_{t-1})]-\EE[P(w_{t})]  + \frac{\eta_{t-1}}{2}\EE[\| \nabla P(w_{t-1})-v_{t-1}\|^2].
\end{eqnarray*}
Summing over $t=1,2,\dots, m+1$, we have
\begin{eqnarray*}
 \sum_{t=1}^{m+1}\frac{\eta_{t-1}}{2}\EE[\| \nabla P(w_{t-1})\|^2]
   &\leq&   \sum_{t=1}^{m+1} \EE[P(w_{t-1})-P(w_{t})] +  \sum_{t=1}^{m+1}\frac{\eta_{t-1}}{2}\EE[\| \nabla P(w_{t-1})-v_{t-1}\|^2]\notag\\
   &=&  \EE[P(w_{0})-P(w_{m+1})] +  \sum_{t=1}^{m+1}\frac{\eta_{t-1}}{2} \EE[\| \nabla P(w_{t-1})-v_{t-1}\|^2\notag\\
    &\leq &   \EE[P(w_{0})-P(w_{*})] +  \sum_{t=1}^{m+1}\frac{\eta_{t-1}}{2}\EE[\| \nabla P(w_{t-1})-v_{t-1}\|^2],
\end{eqnarray*}
where the last inequality holds since $w^*$ is the global minimum of $P.$

The last expression can be equivalently written as
\begin{eqnarray*}
 \sum_{t=0}^{m}\frac{\eta_{t}}{2}\EE[\| \nabla P(w_{t})\|^2]
    &\leq &  \EE[P(w_{0})-P(w_{*})] + \sum_{t=0}^{m}\frac{\eta_{t}}{2}\EE[\| \nabla P(w_{t})-v_{t}\|^2],
 \end{eqnarray*}
 which completes the proof.
\end{proof}

\begin{lemma}
\label{Lemma_boundonDifference2_ls_uniform}
Consider Algorithm~\ref{algo:aisarah-theory-uniform} with $\eta_t =  \min\left\{\frac{1}{L^{t}}, \frac{L^{t-1}}{L^{t}}\eta_{t-1}\right\}$.  Suppose $f_i$ is convex for all $i \in [n]$. Then, under Assumption~\ref{smoothness_on_ls}, for any $t\geq1$,
\begin{align*}
    \EE[\|\nabla P(w_t) - v_t\|^2] \leq \left(\frac{\eta_{0} L^{0}}{2-\eta_{0} L^{0}}\right) \EE[\|v_{0}\|^2].
\end{align*}
\end{lemma}

\begin{proof}
\begin{eqnarray*}
 \EE_{i_j}\left[\|v_j\|^2\right] &\leq&   \EE_{i_j}\left[\|v_{j-1} - \left(\nabla f_{i_j}(w_{j-1}) - \nabla f_{i_j}(w_j) \right)\|^2\right]\notag\\
   &=& \|v_{j-1}\|^2 +  \EE_{i_j}\left[\|\nabla f_{i_j}(w_{j-1}) - \nabla f_{i_j}(w_j)\|^2\right] 
   \\&&\quad-  \EE_{i_j}\left[\frac{2}{\eta_{j-1}} \left\langle\nabla f_{i_j}(w_{j-1}) - \nabla f_{i_j}(w_j), w_{j-1} - w_j\right\rangle\right]\notag\\
   &\leq& \|v_{j-1}\|^2 +  \EE_{i_j}\left[\|\nabla f_{i_j}(w_{j-1}) - \nabla f_{i_j}(w_j)\|^2\right] 
   \\&&-  \EE_{i_j}\left[\frac{2}{\eta_{j-1} L^{j-1}_{i_j}} \|\nabla f_{i_j}(w_{j-1}) - \nabla f_{i_j}(w_j)\|^2\right].\notag\\
\end{eqnarray*}
For each outer loop $k\geq 1$, it holds that $L^{j-1}_{i_j} \leq L^{j-1} $. Thus,
\begin{eqnarray*}
    \EE_{i_j}[\|v_j\|^2] &\leq&  \|v_{j-1}\|^2 +  \EE_{i_j}\left[\left\|\nabla f_{i_j}(w_{j-1}) - \nabla f_{i_j}(w_j)\right\|^2\right] 
    \\&&-  \frac{2}{\eta_{j-1} L^{j-1}}\EE_{i_j}\left[ \left\|\nabla f_{i_j}(w_{j-1}) - \nabla f_{i_j}(w_j)\right\|^2\right]\notag\\
    &=&  \|v_{j-1}\|^2 + \left(1- \frac{2}{\eta_{j-1} L^{j-1}}\right)  \EE_{i_j}\left[\left\|\nabla f_{i_j}(w_{j-1}) - \nabla f_{i_j}(w_j)\right\|^2\right]\notag\\
    &=&  \|v_{j-1}\|^2 + \left(1- \frac{2}{\eta_{j-1} L^{j-1}}\right)  \EE_{i_j}\left[\left\|v_j - v_{j-1}\right\|^2\right].\notag\\
    &\leq&  \|v_{j-1}\|^2 + \left(1- \frac{2}{\eta_{j-1} L^{j-1}}\right)  \EE_{i_j}\left[\left\|v_j - v_{j-1}\right\|^2\right].\notag\\
\end{eqnarray*}

By rearranging, taking expectations again, and assuming that $\eta_{j-1} < 2/ L^{j-1}$ for any $j$ from $1$ to $t+1$,
\begin{eqnarray*}
 \EE[\|v_j - v_{j-1}\|^2] &\leq& \EE\left[\left(\frac{\eta_{j-1} L^{j-1}}{2-\eta_{j-1} L^{j-1}}\right) \left[\|v_{j-1}\|^2 -  \|v_j\|^2\right]\right].  \notag\\
\end{eqnarray*}

By summing the above inequality over $j=1,\dots, t$ $(t\geq1)$, we have
\begin{eqnarray}
\label{eq:lemma-10}
 \sum_{j=1}^t  \EE[\|v_j - v_{j-1}\|^2] &\leq&   \sum_{j=1}^t\EE\left[\left(\frac{\eta_{j-1} L^{j-1}}{2-\eta_{j-1} L^{j-1}}\right)\left[\|v_{j-1}\|^2 - \|v_j\|^2 \right]\right]  \notag\\
 &=&  \left(\frac{\eta_{0} L^{0}}{2-\eta_{0} L^{0}}\right) \EE\left[\|v_{0}\|^2\right] - \left(\frac{\eta_{t-1} L^{t-1}}{2-\eta_{t-1} L^{t-1}}\right) \EE\left[\|v_{t}\|^2\right] \notag\\&&- \sum_{j=1}^{t-1}\EE\left[\left(\frac{\eta_{j-1} L^{j-1}}{2-\eta_{j-1} L^{j-1}} - \frac{\eta_{j} L^{j}}{2-\eta_{j} L^{j}}\right)\|v_{j}\|^2  \right]\notag\\
 &\leq& \left(\frac{\eta_{0} L^{0}}{2-\eta_{0}L^{0}}\right) \EE\left[\|v_{0}\|^2\right].
\end{eqnarray}

Now, by using Lemma~\ref{lemmaTechnical_1_uniform} and \eqref{eq:lemma-10}, we obtain
\begin{eqnarray*}
\EE[\|\nabla P(w_t) - v_t\|^2] &{\leq}& \sum_{j=1}^t \EE\left[\|v_j - v_{j-1}\|^2\right]\notag\\& {\leq} &\left(\frac{\eta_{0} L^{0}}{2-\eta_{0} L^{0}}\right) \EE\left[\|v_{0}\|^2\right].
\end{eqnarray*}
\end{proof}
Using the above lemmas, we can present one of our main results in the following theorem.
\begin{theorem}\label{the:converge_ls_uniform} Suppose that Assumption~\ref{smoothness_on_ls}, \ref{mu_convexity_ls}, holds. 
Let us define 
\begin{equation*}
    \sigma^k_m=\left(\frac{1}{\mu \mathcal{H}} +\frac{\eta_{0} L^{0}}{2-\eta_{0} L^{0}}\right),
\end{equation*}
and select $m$ and $\eta$ such that $\sigma^k_m<1$. Then, Algorithm \ref{algo:aisarah-theory-uniform} converges as follows
\begin{equation*}
    \EE[\| \nabla P(\tilde{w}_{k})\|^2 \leq \left(\prod^{k}_{l=1}\sigma^l_m\right) \|\nabla P(\tilde{w}_{0})\|^2.
\end{equation*}
\end{theorem}
\begin{proof}
Since $v_0=\nabla P(w_0)$ implies $\|\nabla P(w_0)-v_0 \|^2=0$, then by Lemma~\ref{Lemma_boundonDifference2_ls_uniform}, we obtain:
\begin{equation*}
\sum_{t=0}^m \frac{\eta_t}{\mathcal{H}}\EE[\|\nabla P(w_t) - v_t\|^2]  \leq \left(\frac{\eta_{0} L^{0}}{2-\eta_{0} L^{0}}\right) \EE[\|v_{0}\|^2].
\end{equation*}
Combine this with Lemma~\ref{Lemma_BoundGradient_ls_uniform}, we have
\begin{eqnarray*}
 \sum_{t=0}^{m}\frac{\eta_t}{\mathcal{H}}\EE[\| \nabla P(w_{t})\|^2]
    &\leq &  \frac{2}{\mathcal{H}} \EE[P(w_{0})-P(w_{*})] +  \sum_{t=0}^{m}\frac{\eta_t}{\mathcal{H}}\EE[\| \nabla P(w_{t})-v_{t}\|^2]\notag\\
    &{\leq} &  \frac{2}{\mathcal{H}} \EE[P(w_{0})-P(w_{*})] + \frac{\eta_{0} L^{0}}{2-\eta_{0} L^{0}} \EE[\|v_{0}\|^2].
\end{eqnarray*}
Since we consider one outer loop, with $k\geq1$, we have $v_0=\nabla P(w_0)=\nabla P(\tilde{w}_{k-1})$ and $\tilde{w}_{k}=w_t$, where $t$ is drawn at random from $\{0,1,\dots, m\}$ with probabilities $q_t$. Therefore, the following holds,
\begin{eqnarray*}
 \EE[\| \nabla P(\tilde{w}_{k})\|^2]&=&\sum_{t=0}^{m}\sum_{t=0}^{m}\frac{\eta_t}{\mathcal{H}}\EE[\| \nabla P(w_{t})\|^2]\notag\\
    &{\leq} &  \frac{2}{\mathcal{H}} \EE[P(\tilde{w}_{k-1})-P(w_{*})] + \frac{\eta_{0} L^{0}}{2-\eta_{0} L^{0}} \EE[\|\nabla P(\tilde{w}_{k-1})\|^2]\notag\\
    &\leq & \left(\frac{1}{\mu \mathcal{H}} +\frac{\eta_{0} L^{0}}{2-\eta_{0} L^{0}}\right) \EE[\|\nabla P(\tilde{w}_{k-1})\|^2].\notag\\
\end{eqnarray*}
Let us use $\sigma^k_m=\left(\frac{1}{\mu \mathcal{H}} +\frac{\eta_{0} L^{0}}{2-\eta_{0} L^{0}}\right)$,
then the above expression can be written as
\begin{eqnarray*}
 \EE[\| \nabla P(\tilde{w}_{k})\|^2]
    &\leq & \sigma^k_m \EE[\|\nabla P(\tilde{w}_{k-1})\|^2].\notag\\
\end{eqnarray*}
By expanding the recurrence, we obtain
\begin{eqnarray*}
 \EE[\| \nabla P(\tilde{w}_{k})\|^2]
    &\leq &  \left(\prod^{k}_{l=1}\sigma^l_m\right)\|\nabla P(\tilde{w}_{0})\|^2.\notag\\
\end{eqnarray*}
This completes the proof.
\end{proof}

\newpage

\newpage

\section{Extended details on Numerical Experiment}\label{app:exp}
In this chapter, we present the extended details of the design, implementation and results of the numerical experiments. 
\subsection{Problem and Data}\label{sec:problem_data_app}
The machine learning tasks studied in the experiment are binary classification problems. As a common practice in the empirical research of optimization algorithms, the \textit{LIBSVM} datasets\footnote{\textit{LIBSVM} datasets are available at \url{https://www.csie.ntu.edu.tw/~cjlin/libsvmtools/datasets/}.} are chosen to define the tasks. Specifically, \textbf{we selected $10$ popular binary class datasets: \textit{ijcnn1, rcv1, news20, covtype, real-sim, a1a, gisette, w1a, w8a} and \textit{mushrooms}} (see Table \ref{tab:app_data} for basic statistics of the datasets).
\begin{table}[!h]
\vskip-10pt
\centering 
\caption{Summary of Datasets from \cite{chang2011libsvm}.}
\label{tab:app_data}
\begin{threeparttable}
\begin{tabular}{c c c c c}
\toprule
 Dataset    &  $d-1$ (\# feature) & $n$ (\# Train) & $n_{test}$ (\# Test) & \% Sparsity \\
 \hline
 \textit{ijcnn1}\tnote{1}   & 22 & 49,990 & 91,701 & 40.91 \\
 \hline
 \textit{rcv1}\tnote{1}   & 47,236 & 20,242 & 677,399 & 99.85 \\ 
 \hline
 \textit{news20}\tnote{2}  & 1,355,191 & 14,997 & 4,999 & 99.97 \\ 
 \hline
 \textit{covtype}\tnote{2}  & 54 & 435,759 & 145,253 & 77.88 \\
 \hline
  \textit{real-sim}\tnote{2}  & 20,958 & 54,231 & 18,078 & 99.76 \\ 
  \hline
 \hline
 \textit{a1a}\tnote{1}   & 123 & 1,605 & 30,956 & 88.73 \\ 
 \hline
 \textit{gisette}\tnote{1}   & 5,000 & 6,000 & 1,000 & 0.85 \\ 
 \hline
 \textit{w1a}\tnote{1}   & 300 & 2,477 & 47,272 & 96.11 \\ 
 \hline
 \textit{w8a}\tnote{1}   & 300 & 49,749 & 14,951 & 96.12 \\ 
 \hline
 \textit{mushrooms}\tnote{2}  & 112 & 6,093 & 2,031 & 81.25 \\
 \hline
 \bottomrule
\end{tabular}
\begin{tablenotes}\footnotesize
\item[1] dataset has default training/testing samples.
\item[2] dataset is randomly split by 75\%-training \& 25\%-testing.
\end{tablenotes}
\end{threeparttable}
\vskip-10pt
\end{table} 

\subsubsection{Data Pre-Processing}
Let $(\rchi_i,y_i)$ be a training (or testing) sample indexed by $i \in \n$ (or $i \in [n_{test}]$), where $\rchi_i \in \cR^{d-1}$ is a feature vector and $y_i$ is a label. We pre-processed the data such that $\rchi_i$ is of a unit length in Euclidean norm and $y_i \in \{-1, +1\}$.

\subsubsection{Model and Loss Function}
The selected model, $h_i: \cR^d \mapsto \cR$, is in the linear form
\begin{align*}
    h_i(\omega,\varepsilon) = \rchi_i^T\omega + \varepsilon, \quad \forall i \in \n, \tagthis \label{eq:model_app}
\end{align*}
where $\omega \in \cR^{d-1}$ is a weight vector and $\varepsilon \in \cR$ is a bias term. 

For simplicity of notation, from now on, we let $x_i \defeq [\rchi_i^T\; 1]^T \in \cR^d$ be an augmented feature vector, $w \defeq [\omega^T \; \varepsilon]^T \in \cR^d$ be a parameter vector, and $h_i(w) = x_i^Tw$ for $i \in \n$.

Given a training sample indexed by $i \in \n$, the loss function is defined as a logistic regression 
\begin{align*}
    f_i(w) = \log(1+\exp(-y_ih_i(w)) + \frac{\lambda}{2}\|w\|^2. \tagthis \label{eq:loss_app}
\end{align*}
In (\ref{eq:loss_app}), $\frac{\lambda}{2}\|w\|^2$ is the \mbox{$\ell^2$-regularization} of a particular choice of $\lambda > 0$, where we used $\lambda = \frac{1}{n}$ in the experiment; for the non-regularized case, $\lambda$ was set to $0$. Accordingly, the finite-sum minimization problem we aimed to solve is defined as 
\begin{align*}
\min_{w \in \cR^d}
    \bigg\{P(w) \defeq \frac{1}{n}\sum_{i=1}^n f_i(w)\bigg\}. \tagthis \label{eq:problem_app}
\end{align*}
Note that (\ref{eq:problem_app}) is a convex function. For the $\reglize{}$ case, i.e., $\lambda = 1/n$ in (\ref{eq:loss_app}), (\ref{eq:problem_app}) is \scon{$\mu$} and $\mu = \frac{1}{n}$. However, without the $\reg{}$, i.e., $\lambda = 0$ in (\ref{eq:loss_app}), (\ref{eq:problem_app}) is \scon{$\mu$} if and only if there there exists $\mu > 0$ such that $\nabla^2 P(w) \succeq \mu I$ for $w \in \cR^d$ (provided $\nabla P(w) \in \mathcal{C}$). 

\subsection{Algorithms}\label{sec:implementation_app}
This section provides the implementation details\footnote{Code will be made available upon publication.} of the algorithms, practical consideration, and discussions.
\subsubsection{Tune-free \textit{AI-SARAH}}
In Chapter \ref{sec:fully_adaptive} of the main paper, we introduced \textit{AI-SARAH} (see Algorithm \ref{algo:aisarah}), a tune-free and fully adaptive algorithm. \textbf{The implementation of Algorithm \ref{algo:aisarah} was quite straightforward, and we highlight the implementation of Line $10$ with details}:  for logistic regression, the one-dimensional (constrained optimization) sub-problem $\min_{\alpha>0} \xi_t(\alpha)$ can be approximately solved by computing the Newton step at $\alpha=0$, i.e., $\tilde \alpha_{t-1} =- \frac{\xi'_t(0)}{|\xi''_t(0)|}$. This can be easily implemented with automatic differentiation in Pytorch\footnote{For detailed description of the automatic differentiation engine in Pytorch, please see \url{https://pytorch.org/tutorials/beginner/blitz/autograd_tutorial.html}.}, and only two additional backward passes w.r.t $\alpha$ is needed. For function in some particular form, such as a linear least square loss function, an exact solution in closed form can be easily derived. 

As mentioned in Chapter \ref{sec:fully_adaptive}, we have an adaptive upper-bound, i.e., $\alpha_{max}$, in the algorithm. To be specific, the algorithm starts without an upper-bound, i.e., $\alpha_{max}=\infty$ on Line 3 of Algorithm \ref{algo:aisarah}. Then, $\alpha_{max}$ is updated per (inner) iteration. Recall in Chapter \ref{sec:fully_adaptive}, $\alpha_{max}$ is computed as a harmonic mean of the sequence, i.e., $\{\tilde \alpha_{t-1}\}$, and an exponential smoothing is applied on top of the simple harmonic mean. 

\textbf{Having an upper-bound stabilizes the algorithm from stochastic optimization perspective}. For example, when the training error of the randomly selected mini-batch at $w_t$ is drastically reduced or approaching zero, the one-step Newton solution in (\ref{eq:one-step-newton}) could be very large, i.e. $\tilde \alpha_{t-1} \gg 0$, which could be too aggressive to other mini-batch and hence Problem (\ref{MainProb}) prescribed by the batch. \textbf{On the other hand, making the upper-bound adaptive allows the algorithm to adapt to the local geometry and avoid restrictions} on using a large step-size when the algorithm tries to make aggressive progress with respect to Problem (\ref{MainProb}). With the adaptive upper-bound being derived by an \textbf{exponential smoothing of the harmonic mean, the step-size is determined by emphasizing the current estimate of local geometry while taking into account the history of the estimates}. The exponential smoothing further stabilizes the algorithm by balancing the trade-off of being locally focused (with respect to $f_{S_t}$) and globally focused (with respect to $P$).   

It is worthwhile to mention that \textbf{Algorithm \ref{algo:aisarah} does not require computing extra gradient of $f_{S_t}$ with respect to $w$ if compared with \textit{SARAH} and \textit{SARAH+}}. At each inner iteration, $t \geq 1$, Algorithm \ref{algo:aisarah} computes $\nabla f_{S_t}(w_{t-1} - \alpha v_{t-1})$ with $\alpha=0$ just as \textit{SARAH} and \textit{SARAH+} would compute $\nabla f_{S_t}(w_{t-1})$, and the only difference is that $\alpha$ is specified as a variable in Pytorch. After the adaptive step-size $\alpha_{t-1}$ is determined (Line 17), Algorithm \ref{algo:aisarah} computes $\nabla f_{S_t}(w_{t-1} - \alpha_{t-1} v_{t-1})$ just as \textit{SARAH} and \textit{SARAH+} would compute $\nabla f_{S_t}(w_t)$.

In Chapter \ref{sec:fully_adaptive} of the main paper, we discussed the sensitivity of Algorithm \ref{algo:aisarah} on the choice of $\gamma$. Here, we present the full results (on $10$ chosen datasets for both \reglize{} and non-regularized cases) in Figures \ref{fig:gamma_reg_app}, \ref{fig:gamma_reg_app_2}, \ref{fig:gamma_no_reg_app}, and \ref{fig:gamma_no_reg_app_2}. Note that, in this experiment, we chose $\gamma \in \{\frac{1}{64},\frac{1}{32},\frac{1}{16},\frac{1}{8},\frac{1}{4},\frac{1}{2}\}$, and for each $\gamma$, dataset and case, we used $10$ distinct random seeds and ran each experiment for $20$ effective passes. 
\begin{figure*}
    \centering
    \includegraphics[width=0.9\textwidth]{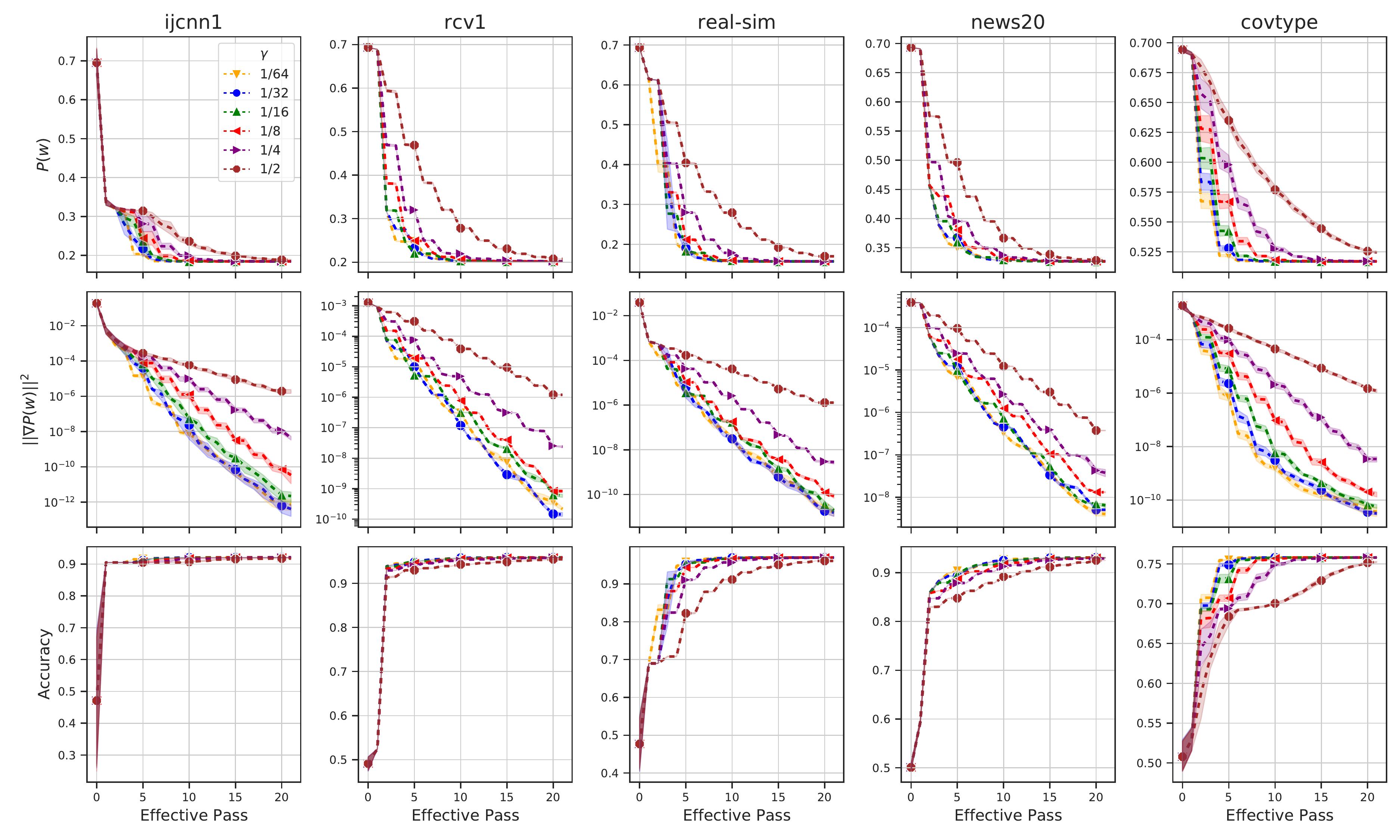}
    \caption{\reglize{} case \textit{ijcnn1, rcv1, real-sim, news20} and \textit{covtype} with $\gamma \in \{\frac{1}{64},\frac{1}{32},\frac{1}{16},\frac{1}{8},\frac{1}{4},\frac{1}{2}\}$: evolution of $P(w)$ (top row) and $\|\nabla P(w)\|^2$ (middle row) and running maximum of testing accuracy (bottom row).}
    \label{fig:gamma_reg_app}
\end{figure*}
\begin{figure*}
    \centering
    \includegraphics[width=0.9\textwidth]{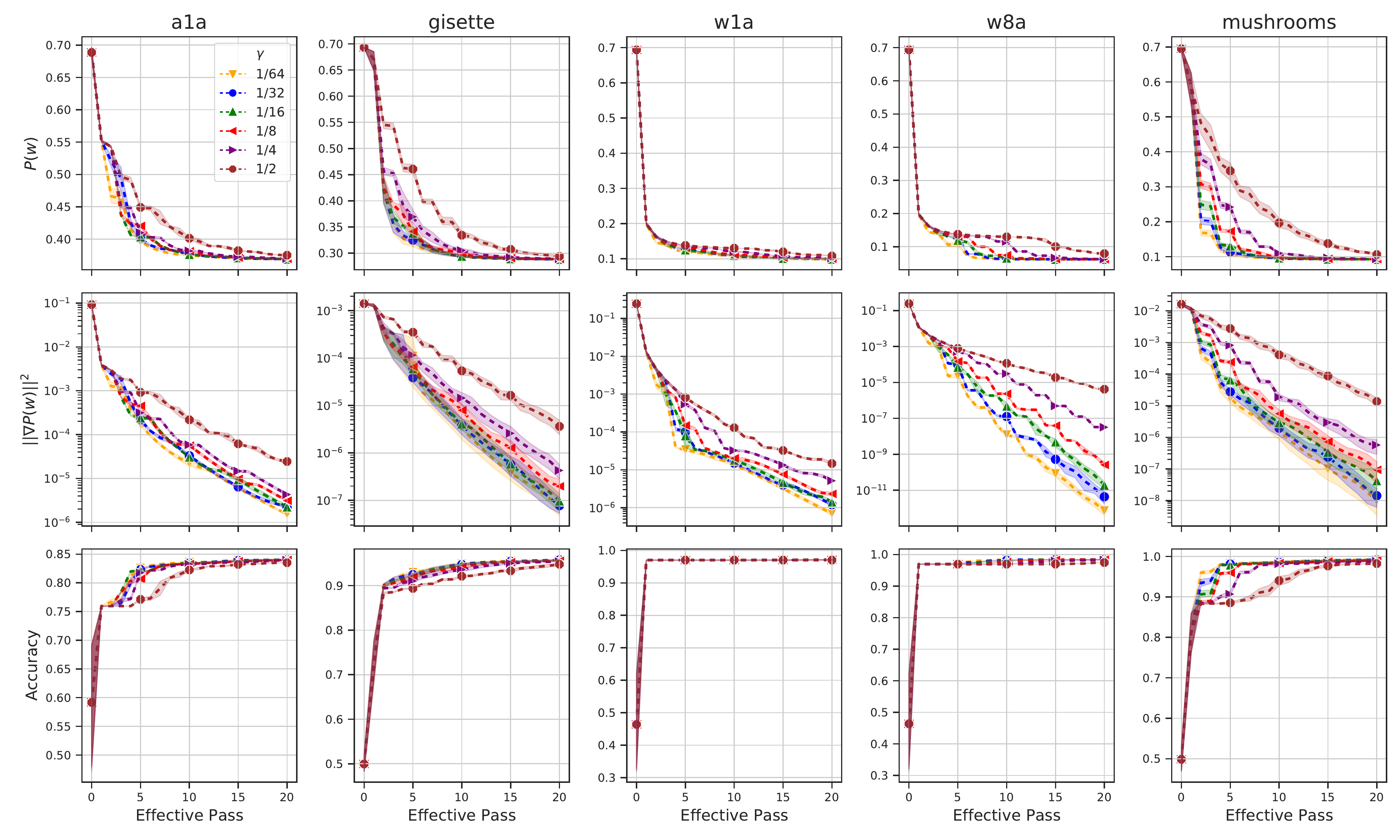}
    \caption{\reglize{} case of \textit{a1a, gisette, w1a, w8a} and \textit{mushrooms} with $\gamma \in \{\frac{1}{64},\frac{1}{32},\frac{1}{16},\frac{1}{8},\frac{1}{4},\frac{1}{2}\}$: evolution of $P(w)$ (top row) and $\|\nabla P(w)\|^2$ (middle row) and running maximum of testing accuracy (bottom row).}
    \label{fig:gamma_reg_app_2}
\end{figure*}

\begin{figure*}
    \centering
    \includegraphics[width=0.9\textwidth]{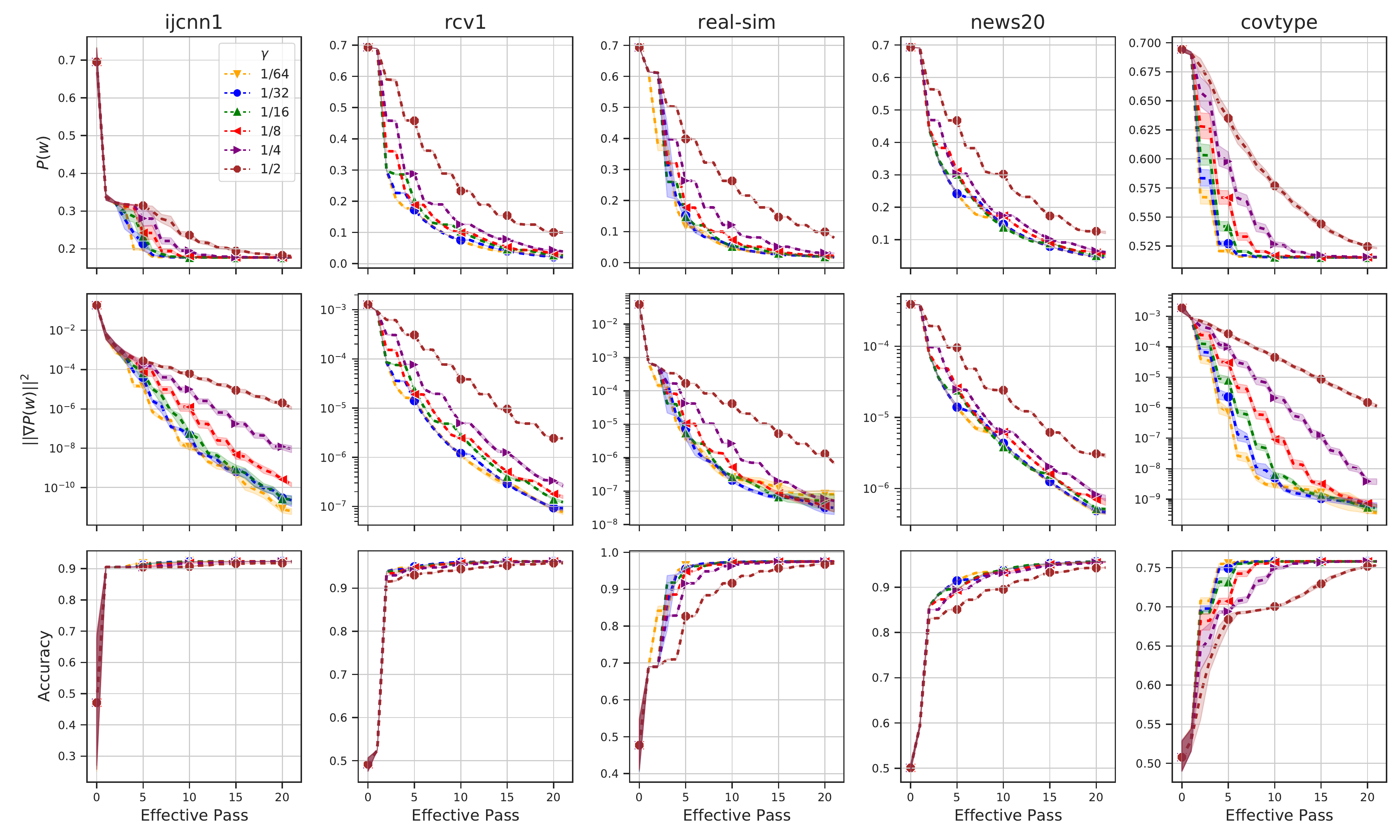}
    \caption{Non-regularized case \textit{ijcnn1, rcv1, real-sim, news20} and \textit{covtype} with $\gamma \in \{\frac{1}{64},\frac{1}{32},\frac{1}{16},\frac{1}{8},\frac{1}{4},\frac{1}{2}\}$: evolution of $P(w)$ (top row) and $\|\nabla P(w)\|^2$ (middle row) and running maximum of testing accuracy (bottom row).}
    \label{fig:gamma_no_reg_app}
\end{figure*}
\begin{figure*}
    \centering
    \includegraphics[width=0.9\textwidth]{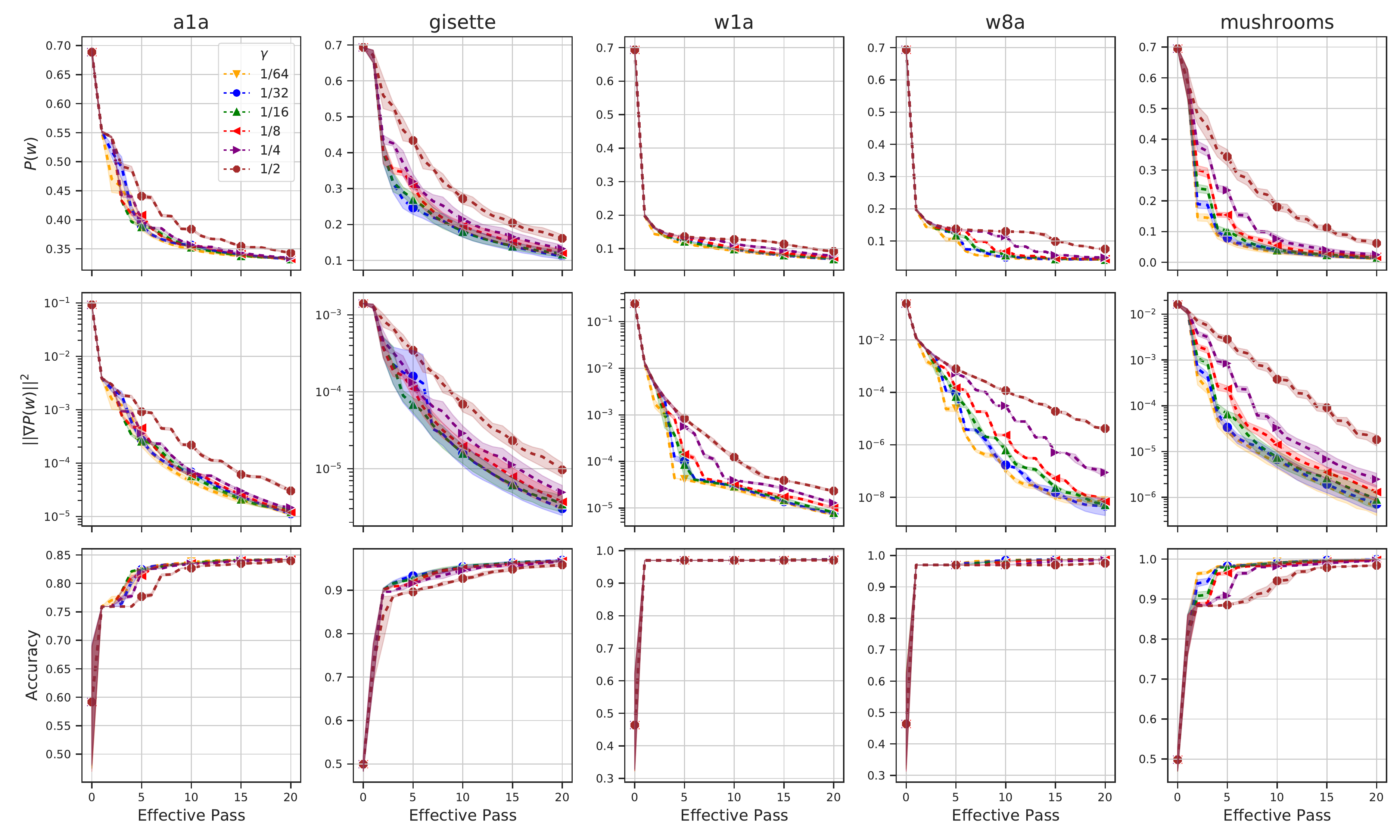}
    \caption{Non-regularized case \textit{a1a, gisette, w1a, w8a} and \textit{mushrooms} with $\gamma \in \{\frac{1}{64},\frac{1}{32},\frac{1}{16},\frac{1}{8},\frac{1}{4},\frac{1}{2}\}$: evolution of $P(w)$ (top row) and $\|\nabla P(w)\|^2$ (middle row) and running maximum of testing accuracy (bottom row).}
    \label{fig:gamma_no_reg_app_2}
\end{figure*}

\subsubsection{Other Algorithms}\label{sec:hyper-parameter-tuning}
In our numerical experiment, we compared the performance of \textbf{TUNE-FREE} \textit{AI-SARAH} (Algorithm \ref{algo:aisarah}) with that of 5 \textbf{FINE-TUNED} state-of-the-art (stochastic variance reduced or adaptive) first-order methods: \textit{SARAH}, \textit{SARAH}+, \textit{SVRG}, \textit{ADAM} and \textit{SGD} with Momentum (\textit{SGD} w/m). These algorithms were implemented in Pytorch, where \textit{ADAM} and \textit{SGD} w/m are built-in optimizers of Pytorch.
\paragraph{Hyper-parameter tuning.} For \textit{ADAM} and \textit{SGD} w/m, we selected $60$ different values of the (initial) step-size on the interval $[10^{-3},10]$ and $5$ different schedules to decrease the step-size after every effective pass on the training samples; for \textit{SARAH} and \textit{SVRG}, we selected $10$ different values of the (constant) step-size and $16$ different values of the inner loop size; for \textit{SARAH}+, the values of step-size were selected in the same way as that of \textit{SARAH} and \textit{SVRG}. In addition, we chose $5$ different values of the inner loop early stopping parameter. Table \ref{tab:fine-tune_app} presents the detailed tuning plan for these algorithms. 
\begin{table}
\centering 
\caption{Tuning Plan - Choice of Hyper-parameters.}
\vfill
\label{tab:fine-tune_app}
\begin{threeparttable}
\resizebox{\columnwidth}{!}{%
\begin{tabular}{c c c c c c}
\toprule
 Method    &  \# Configuration &  Step-Size & Schedule ($\%$)\tnote{1} & Inner Loop Size (\# Effective Pass) & Early Stopping ($\gamma$) \\
 \hline
  \textit{\textit{SARAH}} & $160$ & $\{0.1,0.2,...,1\}/L$ & n/a & $\{0.5,0.6,...,2\}$ & n/a \\ 
 \hline
 \textit{\textit{SARAH}+} & $50$  & $\{0.1,0.2,...,1\}/L$ & n/a & n/a & $1/\{2,4,8,16,32\}$ \\
 \hline
 \textit{\textit{SVRG}} & $160$ & $\{0.1,0.2,...,1\}/L$ & n/a & $\{0.5,0.6,...,2\}$ & n/a\\
 \hline
 \textit{\textit{ADAM}}\tnote{2} & $300$ & $[10^{-3}, 10]$ & $\{0,1,5,10,15\}$ & n/a & n/a \\
 \hline
 \textit{\textit{SGD} w/m}\tnote{3} & $300$ & $[10^{-3}, 10]$ & $\{0,1,5,10,15\}$ & n/a & n/a \\
 \hline
 \bottomrule
\end{tabular}
}
\begin{tablenotes}\footnotesize
\item[1] Step-size is scheduled to decrease by $X\%$ every effective pass over the training samples.
\item[2] $\beta_1 = 0.9, \beta_2 = 0.999$.
\item[3] $\beta = 0.9$.
\end{tablenotes}
\end{threeparttable}
\vskip-10pt
\end{table} 

\textbf{\textit{Selection criteria:}}

We defined the best hyper-parameters as the ones yielding the minimum ending value of the loss function, where the running budget is presented in Table \ref{tab:budget_app}. Specifically, the criteria are: (1) filtering out the ones exhibited a "spike" of the loss function, i.e., the initial value of the loss function is surpassed at any point within the budget; (2) selecting the ones achieved the minimum ending value of the loss function.
\begin{table}
\centering 
\caption{Running Budget (\# Effective Pass).}
\vfill
\label{tab:budget_app}
\begin{threeparttable}
\begin{tabular}{c c c}
\toprule
 Dataset    &  Regularized &  Non-regularized \\
 \hline
 \textit{ijcnn1} & 20  & 20 \\
 \hline
  \textit{rcv1} &  30 & 40 \\
 \hline
 \textit{news20} & 40 & 50 \\
 \hline
 \textit{covtype} & 20 & 20 \\
 \hline
 \textit{real-sim} & 20 & 30 \\
 \hline
 \hline
  \textit{a1a} &  30 & 40 \\
 \hline
 \textit{gisette} & 30  & 40 \\
 \hline
 \textit{w1a} & 40 & 50 \\
 \hline
 \textit{w8a} & 30 & 40 \\
 \hline
 \textit{mushrooms} & 30 & 40 \\
 \bottomrule
\end{tabular}
\end{threeparttable}
\vskip-10pt
\end{table} 

\textbf{\textit{Hightlights of the hyper-parameter search:}}
\begin{itemize}[noitemsep,nolistsep,topsep=0pt,leftmargin=10pt]
    \item To take into account the randomness in the performance of these algorithms provided different hyper-parameters, we ran each configuration with $5$ distinct random seeds. \textbf{The total number of runs for each dataset and case is $4,850$.}
    \item Tables \ref{tab:fine_tune_reg_app} and \ref{tab:fine_tune_noreg_app} present the best hyper-parameters selected from the candidates for the regularized and non-regularized cases.
    \item Figures \ref{fig:hyper_reg_app}, \ref{fig:hyper_reg_app_2}, \ref{fig:hyper_no_reg_app} and \ref{fig:hyper_no_reg_app_2} show the performance of different hyper-parameters for all tuned algorithms; it is clearly that, \textbf{the performance is highly dependent on the choices of hyper-parameter for \textit{SARAH}, \textit{SARAH}+, and \textit{SVRG}}. And, \textbf{the performance of \textit{ADAM} and \textit{SGD} w/m are very SENSITIVE to the choices of hyper-parameter}. 
\end{itemize}
\begin{table*}
\centering 
\caption{Fine-tuned Hyper-parameters - \reglize{} Case.}
\vfill
\label{tab:fine_tune_reg_app}
\begin{threeparttable}
\begin{tabular}{c c c c c c c}
\toprule
 Dataset    &  \textit{ADAM} & \textit{SGD} w/m & \textit{SARAH} & \textit{SARAH}+ & \textit{SVRG} \\
 & ($\alpha_0,x\%$) & ($\alpha_0,x\%$) & ($\alpha,m$) & ($\alpha,\gamma$) & ($\alpha,m$)\\
 \hline
 \textit{ijcnn1}   & (0.07,\;15\%) & (0.4,\;15\%) & (3.153,\;1015) & (3.503,\;1/32) & (3.503,\;1562) \\
 \hline
  \textit{rcv1}   & (0.016,\;10\%) & (4.857,\;10\%) & (3.924,\;600) & (3.924,\;1/32) & (3.924,\;632) \\ 
 \hline
 \textit{news20}  & (0.028,\;15\%) & (6.142,\;10\%) & (3.786,\;468) & (3.786,\;1/32) & (3.786,\;468) \\ 
 \hline
 \textit{covtype}  & (0.07,\;15\%) & (0.4,\;15\%) & (2.447,\;13616) & (2.447,\;1/32) & (2.447,\;13616) \\ 
 \hline
 \textit{real-sim}  & (0.16,\;15\%) & (7.428,\;15\%) & (3.165,\;762) & (3.957,\;1/32) & (3.957,\;1694) \\ 
 \hline
 \hline
 \textit{a1a}   & (0.7,\;15\%) & (4.214,\;15\%) & (2.758,\;50) & (2.758,\;1/32) & (2.758,\;50) \\
 \hline
  \textit{gisette}   & (0.028,\;15\%) & (8.714,\;10\%) & (2.320,\;186) & (2.320,\;1/16) & (2.320,\;186)\\ 
 \hline
 \textit{w1a}  & (0.1,\;10\%) & (3.571,\;10\%) & (3.646,\;60) & (3.646,\;1/32) & (3.646,\;76) \\ 
 \hline
 \textit{w8a}  & (0.034,\;15\%) & (2.285,\;15\%) & (2.187,\;543) & (3.645,\;1/32) & (3.645,\;1554) \\ 
 \hline
 \textit{mushrooms}  & (0.220,\;15\%) & (3.571,\;0\%) & (2.682,\;190) & (2.682,\;1/32) & (2.682,\;190) \\ 
 \hline
 \bottomrule
\end{tabular}
\end{threeparttable}
\end{table*} 
\begin{table*}
\centering 
\caption{Fine-tuned Hyper-parameters - Non-regularized Case.}
\vfill
\label{tab:fine_tune_noreg_app}
\begin{threeparttable}
\begin{tabular}{c c c c c c c}
\toprule
 Dataset    &  \textit{ADAM} & \textit{SGD} w/m & \textit{SARAH} & \textit{SARAH}+ & \textit{SVRG} \\
 & ($\alpha_0,x\%$) & ($\alpha_0,x\%$) & ($\alpha,m$) & ($\alpha,\gamma$) & ($\alpha,m$)\\
 \hline
 \textit{ijcnn1}   & (0.1,\;15\%) & (0.58,\;15\%) & (3.153,\;1015) & (3.503,\;1/32) & (3.503,\;1562) \\
 \hline
  \textit{rcv1}   & (5.5,\;10\%) & (10.0,\;0\%) & (3.925,\;632) & (3.925,\;1/32) & (3.925,\;632)\\ 
 \hline
 \textit{news20}  & (1.642,\;10\%) & (10.0,\;0\%) & (3.787,\;468) & (3.787,\;1/32) & (3.787,\;468) \\ 
 \hline
 \textit{covtype}  & (0.16,\;15\%) & (2.2857,\;15\%) & (2.447,\;13616) & (2.447,\;1/32) & (2.447,\;13616) \\ 
 \hline
 \textit{real-sim}  & (2.928,\;15\%) & (10.0,\;0\%) & (3.957,\;1609) & (3.957,\;1/16) & (3.957,\;1694) \\ 
 \hline
 \hline
 \textit{a1a}   & (1.642,\;15\%) & (6.785,\;1\%) & (2.763,\;50) & (2.763,\;1/32) & (2.763,\;50) \\
 \hline
  \textit{gisette}   & (2.285,\;1\%) & (10.0,\;0\%) & (2.321,\;186) & (2.321,\;1/32) & (2.321,\;186)\\ 
 \hline
 \textit{w1a}  & (8.714,\;10\%) & (10.0,\;0\%) & (3.652,\;76) & (3.652,\;1/32) & (3.652,\;76) \\ 
 \hline
 \textit{w8a}  & (0.16,\;10\%) & (10.0,\;5\%) & (2.552,\;543) & (3.645,\;1/32) & (3.645,\;1554) \\ 
 \hline
 \textit{mushrooms}  & (10.0,\;0\%) & (10.0,\;0\%) & (2.683,\;190) & (2.683,\;1/32) & (2.683,\;190) \\ 
 \hline
 \bottomrule
\end{tabular}
\end{threeparttable}
\end{table*} 

\begin{figure*}
    \centering
    \includegraphics[width=0.9\textwidth]{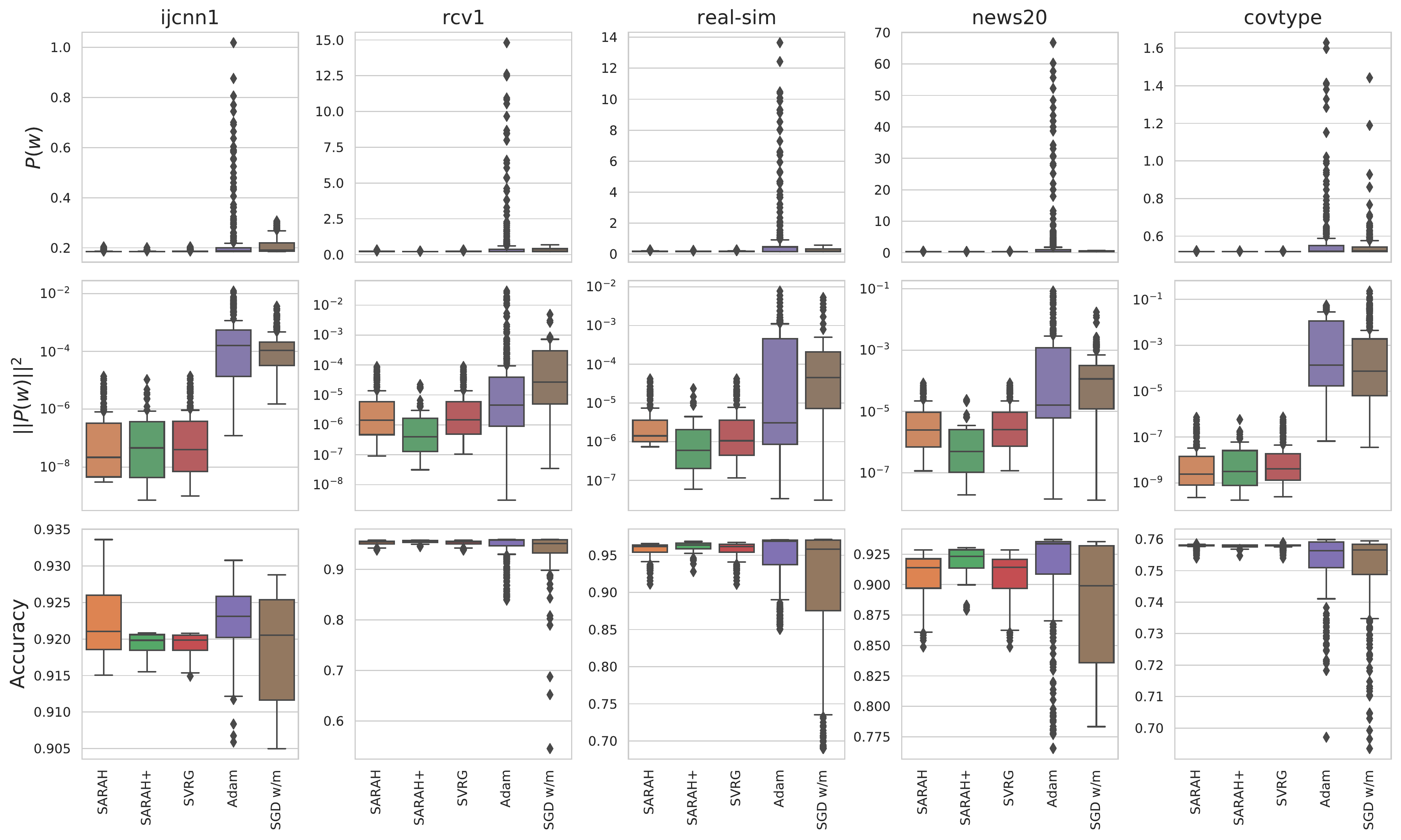}
    \caption{Ending loss (top row), ending squared norm of full gradient (middle row), maximum testing accuracy (bottom row) of different hyper-paramters and algorithms for the \textbf{\reglize{} case} on \textit{ijcnn1, rcv1, real-sim, news20} and \textit{covtype} datasets.}
    \label{fig:hyper_reg_app}
\end{figure*}
\begin{figure*}
    \centering
    \includegraphics[width=0.9\textwidth]{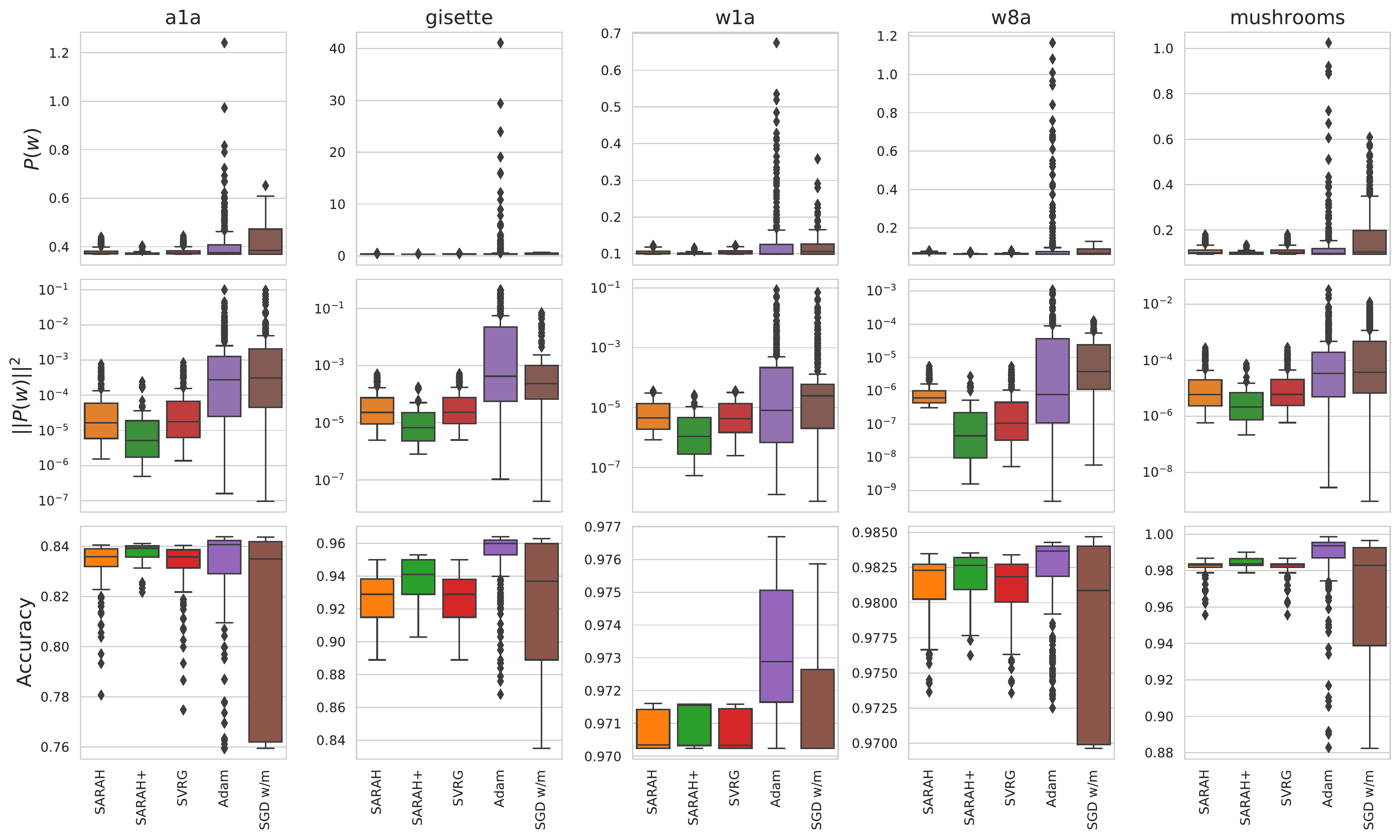}
    \caption{Ending loss (top row), ending squared norm of full gradient (middle row), maximum testing accuracy (bottom row) of different hyper-paramters and algorithms for the \textbf{\reglize{} case} on \textit{a1a, gisette, w1a, w8a} and \textit{mushrooms} datasets.}
    \label{fig:hyper_reg_app_2}
\end{figure*}

\begin{figure*}
    \centering
    \includegraphics[width=0.9\textwidth]{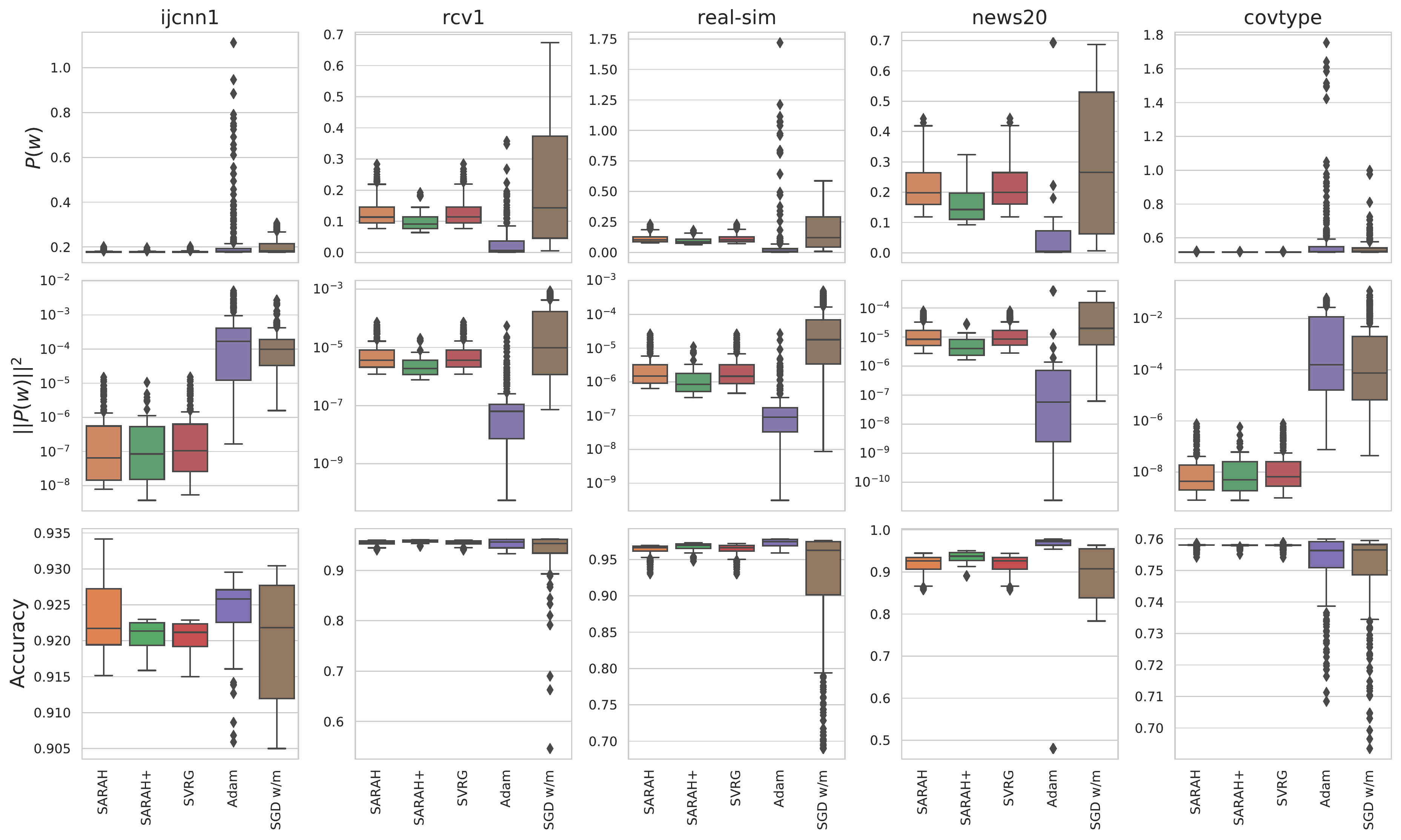}
    \caption{Ending loss (top row), ending squared norm of full gradient (middle row), maximum testing accuracy (bottom row) of different hyper-paramters and algorithms for the \textbf{non-regularized case} on \textit{ijcnn1, rcv1, real-sim, news20} and \textit{covtype} datasets.}
    \label{fig:hyper_no_reg_app}
\end{figure*}
\begin{figure*}
    \centering
    \includegraphics[width=0.9\textwidth]{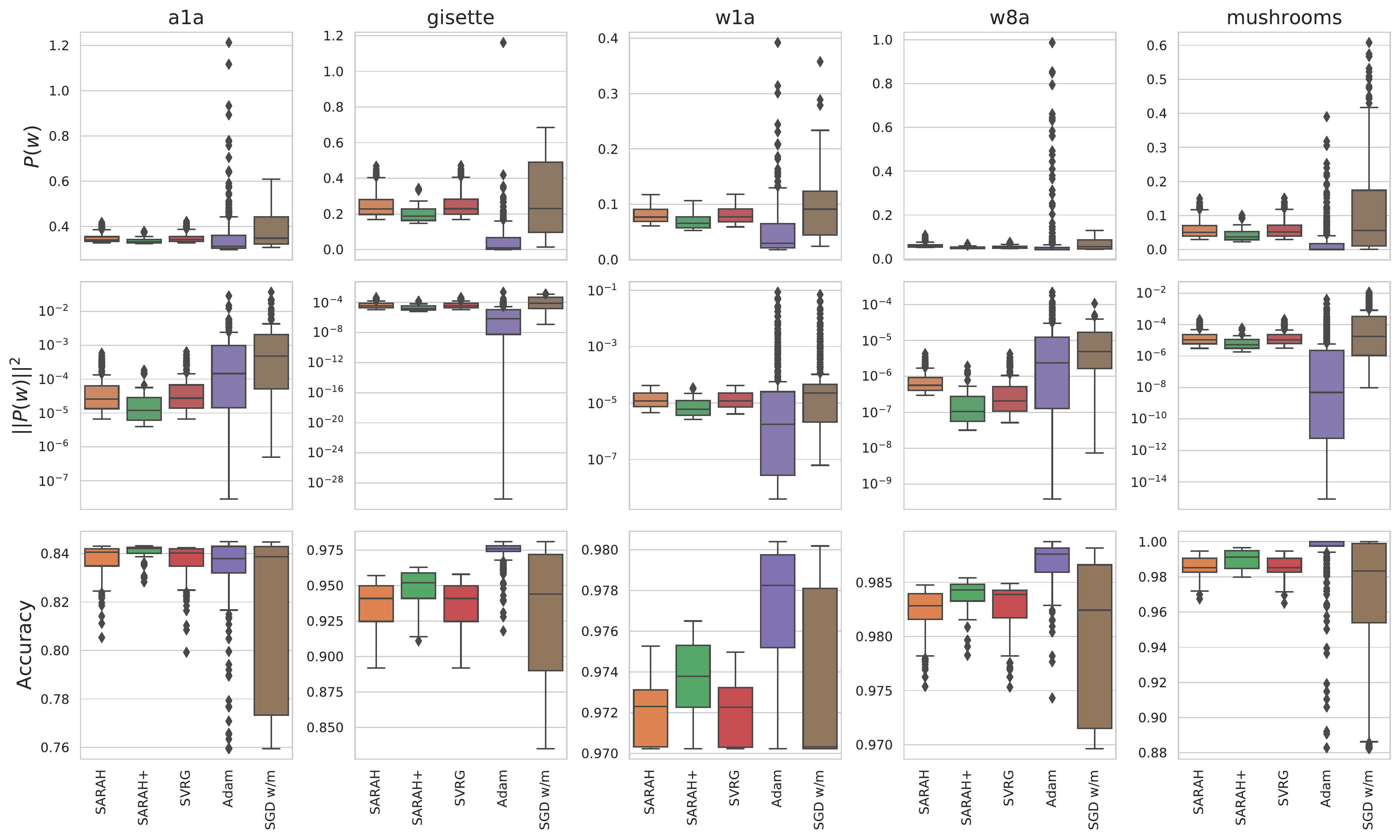}
    \caption{Ending loss (top row), ending squared norm of full gradient (middle row), maximum testing accuracy (bottom row) of different hyper-paramters and algorithms for the \textbf{non-regularized case} on \textit{a1a, gisette, w1a, w8a} and \textit{mushrooms} datasets.}
    \label{fig:hyper_no_reg_app_2}
\end{figure*}

\paragraph{Global Lipschitz smoothness of $P(w)$.} Tuning the (constant) step-size of \textit{SARAH}, \textit{SARAH}+ and \textit{SVRG} requires the parameter of (global) Lipschitz smoothness of $P(w)$, denoted the (global) Lipschitz constant $L$, and it can be computed as, given (\ref{eq:loss_app}) and (\ref{eq:problem_app}), 
\begin{align*}
    L = \frac{1}{4} \lambda_{max} (\frac{1}{n}\sum_{i=1}^n x_ix_i^T) + \lambda,
\end{align*}
where $\lambda_{max}(A)$ denotes the largest eigenvalue of $A$ and $\lambda$ is the penalty term of the \mbox{$\ell^2$-regularization} in (\ref{eq:loss_app}). Table \ref{tab:lip_app} shows the values of $L$ for the regularized and non-regularized cases on the chosen datasets. 
\begin{table*}
\centering 
\caption{Global Lipschitz Constant $L$}
\vfill
\label{tab:lip_app}
\begin{threeparttable}
\begin{tabular}{c c c}
\toprule
 Dataset    &  Regularized &  Non-regularized \\
 \hline
 \textit{ijcnn1} & 0.285408  & 0.285388 \\
 \hline
  \textit{rcv1} &  0.254812 & 0.254763 \\
 \hline
 \textit{news20} & 0.264119 & 0.264052 \\
 \hline
 \textit{covtype} & 0.408527 & 0.408525 \\
 \hline
 \textit{real-sim} & 0.252693 & 0.252675 \\
 \hline
 \hline
 \textit{a1a} & 0.362456  & 0.361833 \\
 \hline
  \textit{gisette} & 0.430994  & 0.430827 \\
 \hline
 \textit{w1a} & 0.274215 & 0.273811 \\
 \hline
 \textit{w8a} & 0.274301 & 0.274281 \\
 \hline
 \textit{mushrooms} & 0.372816 & 0.372652 \\
 \bottomrule
\end{tabular}
\end{threeparttable}
\end{table*} 

\clearpage

\subsection{Extended Results of Experiment}
In Chapter \ref{experiments}, we compared tune-free \& fully adaptive \textit{AI-SARAH} (Algorithm \ref{algo:aisarah}) with fine-tuned \textit{SARAH}, \textit{SARAH}+, \textit{SVRG}, \textit{ADAM} and \textit{SGD} w/m. In this section, we present the extended results of our empirical study on the performance of \textit{AI-SARAH}.

Figures \ref{fig:best_vs_ai_sarah_reg} and \ref{fig:best_vs_ai_sarah_no_reg} compare the average ending $\|\nabla P(w)\|^2$ achieved by \textit{AI-SARAH} with the other algorithms, configured with all candidate hyper-parameters. 

It is clear that,
\begin{itemize}[noitemsep,nolistsep,topsep=0pt,leftmargin=10pt]
    \item without tuning, \textit{AI-SARAH} achieves the best convergence (to a stationary point) in practice on most of the datasets for both cases;
    \item while fine-tuned \textit{ADAM} achieves a better result for the non-regularized case on \textit{a1a, gisette, w1a} and \textit{mushrooms}, \textit{AI-SARAH} outperforms \textit{ADAM} for at least $80\%$ (\textit{a1a}), $55\%$ (\textit{gisette}), $50\%$ (\textit{w1a}), and $50\%$ (\textit{mushrooms}) of all candidate hyper-parameters.
\end{itemize}

Figure \ref{fig:miss_data_reg_app} shows the results of the non-regularized case for \textit{ijcnn1, rcv1, real-sim, news20} and \textit{covtype} datasets. Figures \ref{fig:other_data_reg_app} and \ref{fig:other_data_no_reg_app} present the results of the \reglize{} case and non-regularized case respectively on \textit{a1a, gisette, w1a, w8a} and \textit{mushrooms} datasets. For completeness of presentation, we present the evolution of \textit{AI-SARAH}'s step-size and upper-bound on \textit{a1a, gisette, w1a, w8a} and \textit{mushrooms} datasets in Figures \ref{fig:step_reg_app} and \ref{fig:step_no_reg_app}. Consistent with the results shown in Chapter \ref{experiments} of the main paper, \textit{AI-SARAH} delivers a competitive performance in practice.

\begin{figure*}
\centering
    \includegraphics[width=0.8\textwidth]{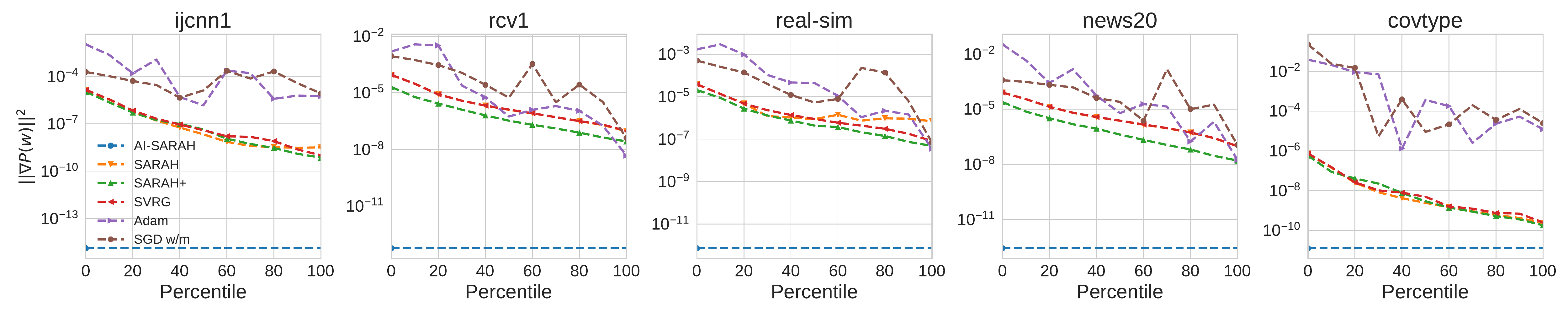}
    \includegraphics[width=0.8\textwidth]{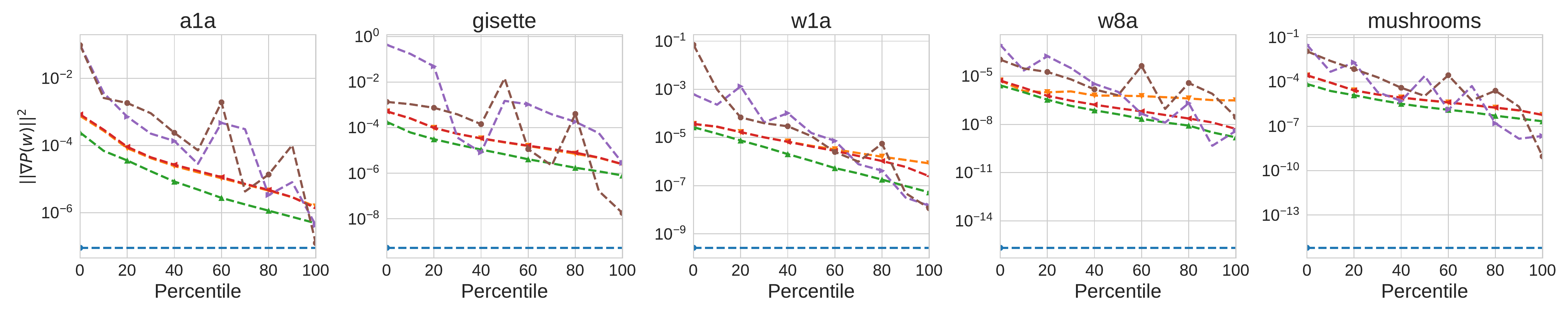}
    \caption{Average ending $\|\nabla P(w)\|^2$ for \reglize{} case - \textit{AI-SARAH} vs. Other Algorithms: \textit{\textit{AI-SARAH} is shown as the horizontal lines; for each of the other algorithms, the average ending $\|\nabla P(w)\|^2$ from different configurations of hyper-parameters are indexed from $0$ percentile (the worst choice) to $100$ percentile (the best choice); see Section \ref{sec:hyper-parameter-tuning} for details of the selection criteria}. }
    \label{fig:best_vs_ai_sarah_reg}
\end{figure*}
\begin{figure*}[!ht]
\vskip+20pt
 \centering
    \includegraphics[width=0.8\textwidth]{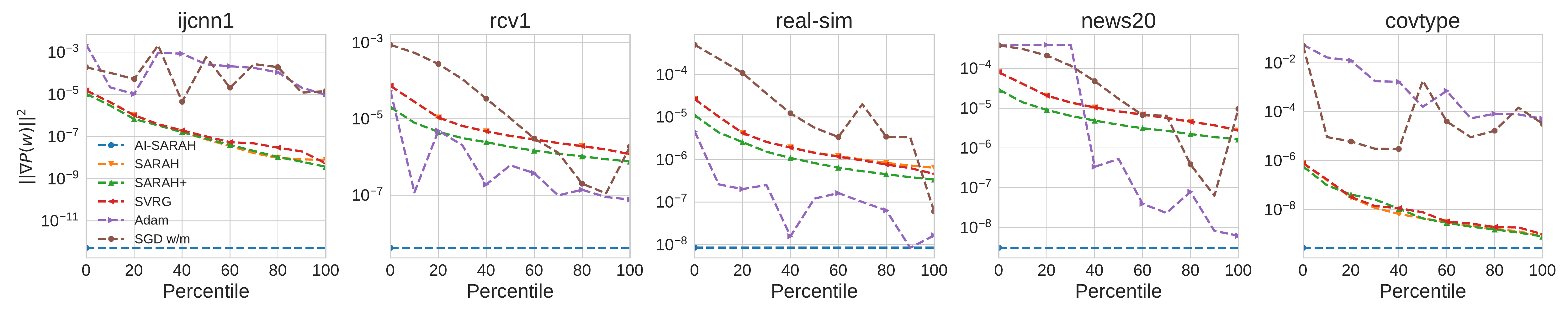}
    \includegraphics[width=0.8\textwidth]{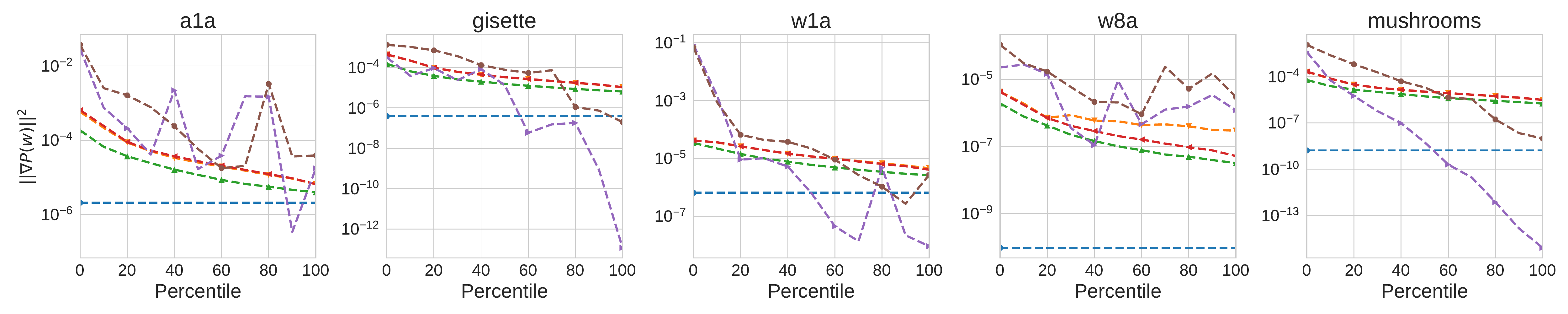}
    \caption{Average ending $\|\nabla P(w)\|^2$ for non-regularized case - \textit{AI-SARAH} vs. Other Algorithms.}
    \label{fig:best_vs_ai_sarah_no_reg}
\end{figure*}
\begin{figure*}
    \centering
    \includegraphics[width=0.9\textwidth]{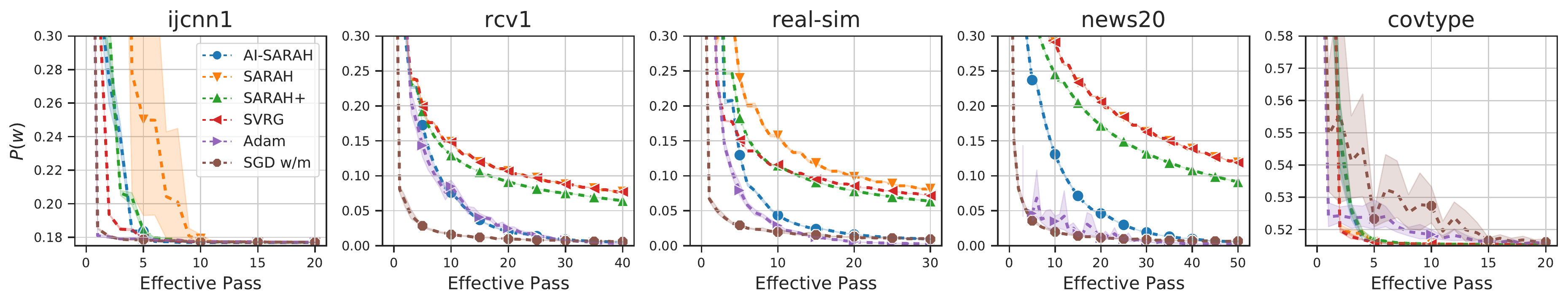}
    \includegraphics[width=0.9\textwidth]{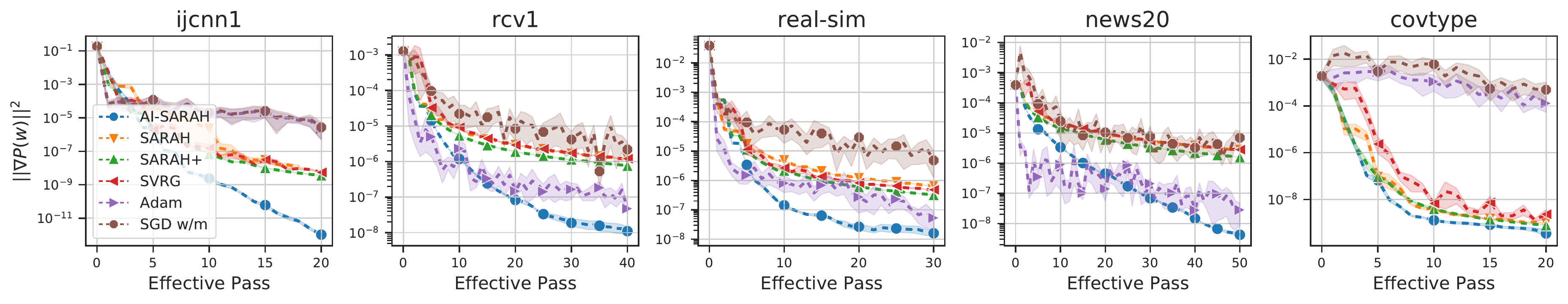}
    \includegraphics[width=0.9\textwidth]{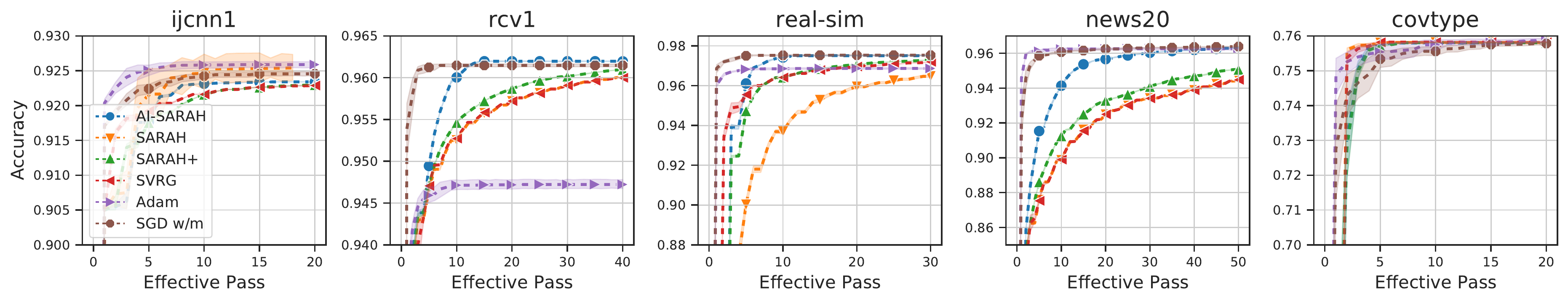}
    \caption{Non-regularized case: evolution of $P(w)$ (top row), $\|\nabla P(w)\|^2$ (middle row), and running maximum of testing accuracy (bottom row).}
    \label{fig:miss_data_reg_app}
\end{figure*} 

\begin{figure*}
    \centering
    \includegraphics[width=0.9\textwidth]{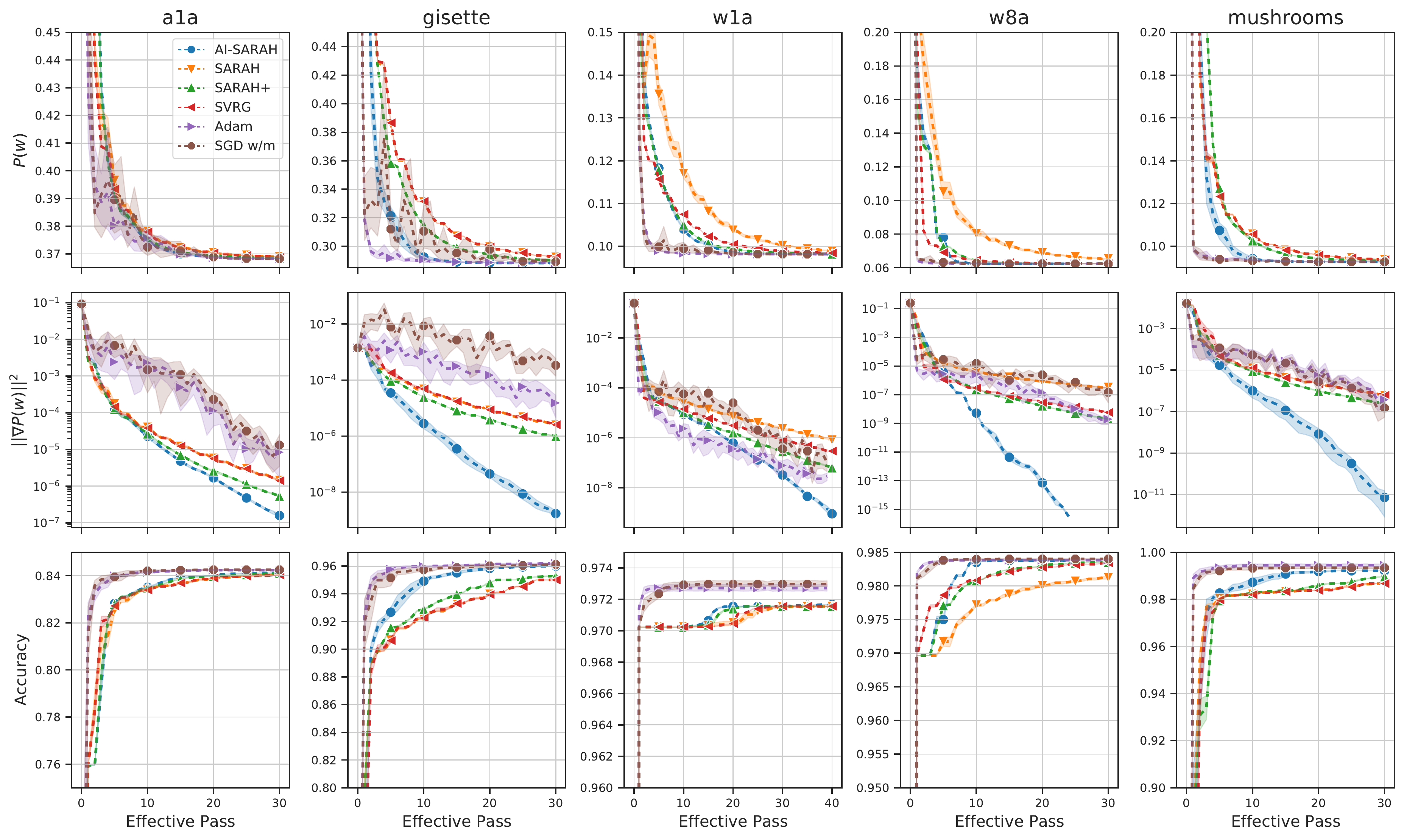}
    \caption{\reglize{} case: evolution of $P(w)$ (top row), $\|\nabla P(w)\|^2$ (middle row), and running maximum of testing accuracy (bottom row).}
    \label{fig:other_data_reg_app}
\end{figure*} 
\begin{figure*}
    \centering
    \includegraphics[width=0.9\textwidth]{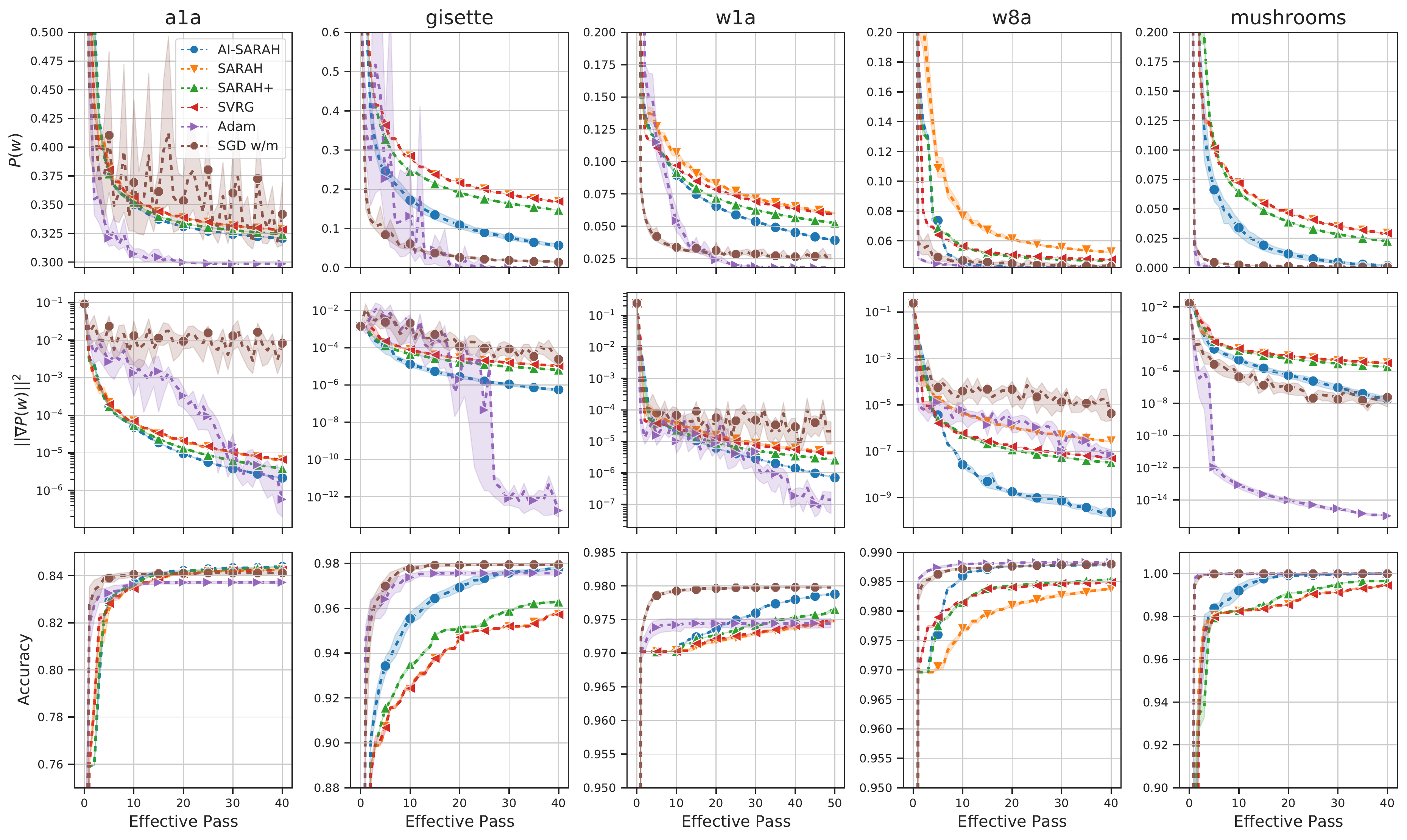}
    \caption{Non-regularized case: evolution of $P(w)$ (top row), $\|\nabla P(w)\|^2$ (middle row), and running maximum of testing accuracy (bottom row).}
    \label{fig:other_data_no_reg_app}
\end{figure*} 
\clearpage
\begin{figure*}[ht!]
    \centering
    \includegraphics[width=\textwidth]{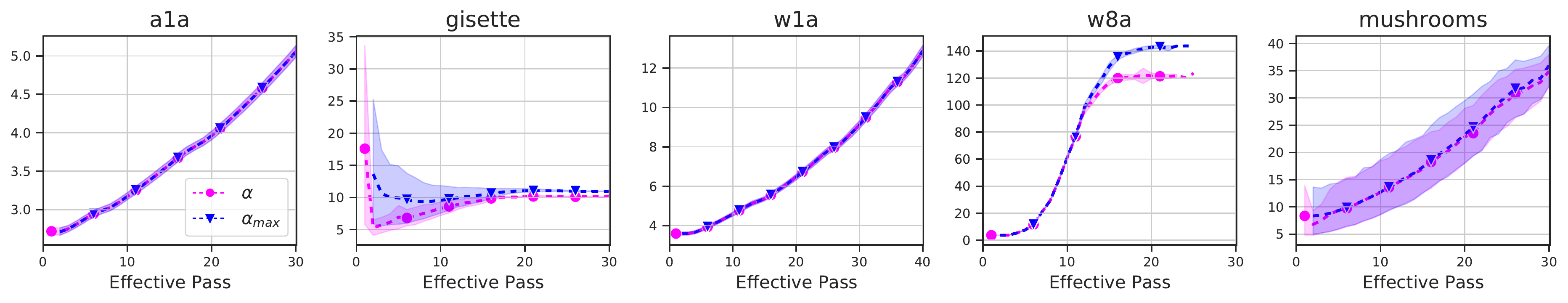}
    \caption{\reglize{} case: evolution of \textit{AI-SARAH}'s step-size $\alpha$ and upper-bound $\alpha_{max}$.}
    \label{fig:step_reg_app}
\vskip+20pt
    \centering
    \includegraphics[width=\textwidth]{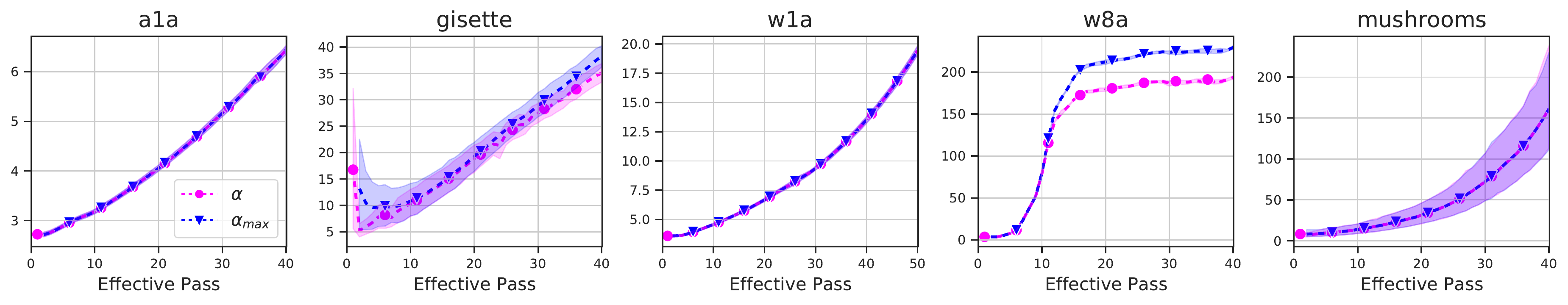}
    \caption{Non-regularized case: evolution of \textit{AI-SARAH}'s step-size $\alpha$ and upper-bound $\alpha_{max}$.}
    \label{fig:step_no_reg_app}
\end{figure*}

\clearpage
\section{Alternative Theoretical Analysis}
\label{app:alternative}
In this chapter, we provide an alternative theoretical framework, as mentioned in Chapter \ref{sec:theoretical-analysis} of the main paper to investigate how to leverage local Lipschitz smoothness. We present the alternative algorithm in Algorithm \ref{algo:aisarah-alternative}.

\begin{algorithm} 
\caption{Alternative Algorithm}
\begin{algorithmic}[1]
\STATE {\bfseries Parameters:} Inner loop size $m$.
\STATE {\bfseries Initialize:} $\tilde{w}_0$
\FOR{k = 1, 2, ...}
\STATE $w_0 = \tilde{w}_{k-1}$
\STATE $v_0 = \nabla P(w_0)$
\STATE Choose $\alpha_{\min}^k$, $\alpha_{\max}^k$ such that $0<\alpha_{\min}^k\leq\alpha_{\max}^k$.
\FOR{$t = 1,...,m$}
\STATE Select random mini-batch $S_t$ from $\n$ uniformly with $|S_t| = b$
\STATE $\alpha_{t-1} \approx \arg\min_{\alpha \in [\alpha_{\min}^k, \alpha_{\max}^k]} \xi_t(\alpha)$
\STATE $w_t = w_{t-1} - \alpha_{t-1} v_{t-1}$
\STATE $v_t = \nabla f_{S_t}(w_t) - \nabla f_{S_t}(w_{t-1}) + v_{t-1}$
\ENDFOR
\STATE Set $\tilde{w}_k = w_t$ with $t$ chosen uniformly at random from $\{0,1,...,m\}$\\
\ENDFOR
\end{algorithmic}
\label{algo:aisarah-alternative}
\end{algorithm} 



Like \textit{SVRG} and \textit{SARAH}, Algorithm \ref{algo:aisarah-alternative} adopts a loop structure, which is divided into the outer loop, where a full gradient is computed, and the inner loop, where only stochastic gradient is computed. However, unlike \textit{SVRG} and \textit{SARAH}, the step-size is computed implicitly. In particular, at each iteration $t \in [m]$ of the inner loop, the step-size is chosen by approximately solving a simple one-dimensional constrained optimization problem. 
Define the sub-problem (optimization problem) at $t \geq 1$ as
\begin{align}
\label{min_xi}
    \min_{\alpha \in [ \alpha^k_{\min},  \alpha^k_{\max}]} \xi_t(\alpha), 
\end{align}
where
$
 \xi_t(\alpha) :=  \|v_t\|^2 
    = \|\nabla f_{S_t}(w_{t-1} - \alpha v_{t-1}) - \nabla f_{S_t}(w_{t-1}) + v_{t-1}\|^2
$, $\alpha^k_{\min}$ and  $\alpha^k_{\max}$ are lower-bound and upper-bound of the step-size respectively. 
These bounds do not allow large fluctuations of the (adaptive) step-size.  We denote $\alpha_{t-1}$ the approximate solution of (\ref{min_xi}).
Now, let us present some remarks regarding Algorithm \ref{algo:aisarah-alternative}.
\begin{remark}
As we will explain with more details in the following sections, the values of $\alpha_{\min}^k$ and $\alpha_{\max}^k$ cannot be arbitrarily large. To guarantee convergence, we will need to assume that $\alpha_{\max}^k  \leq \frac{2}{L_k^{\max}}$, where  $L_k^{\max}= \max_{i \in [n]} L^i_{k}$. Here, $L^i_{k}$ is the local smoothness parameter of $f_i$ defined on a working-set for each outer loop (see Definition~\ref{def:WandL1}).
\end{remark}
\begin{remark}
\label{remarkSARAH}
\textit{SARAH} \cite{sarah17} can be seen as a special case of Algorithm \ref{algo:aisarah-alternative}, where $\alpha_{\min}^k=\alpha_{\max}^k=\alpha$ for all outer loops ($k \geq 1$). In other words, a constant step-size is chosen for the algorithm. However, if $\alpha_{\min}^k<\alpha_{\max}^k$, then the selection of the step-size in Algorithm \ref{algo:aisarah-alternative} allows a faster convergence of  $\|v_t\|^2$ than \textit{SARAH} in each inner loop.
\end{remark}

\begin{remark}
At $t \geq 1$, let us select a mini-batch of size $n$, i.e., $|S_t| = n$. In this case, Algorithm \ref{algo:aisarah-alternative} is equivalent to deterministic gradient descent with a very particular way of selecting the step-size, i.e. by solving the following problem $$ \min_{\alpha \in [\alpha_{\min}^k, \alpha_{\max}^k]}
 \xi_t(\alpha),$$
 where $\xi_t(\alpha)= \|\nabla P \left(w_{t-1} - \alpha \nabla P(w_{t-1})\right)\|^2$.
 In other words, the step-size is selected to  minimize the squared norm of the full gradient with respect to $w_t$.
\end{remark}


\subsection{Definitions / Assumptions}
First, we present the main definitions and assumptions that are used in our convergence analysis.
\begin{definition}
\label{smoothf}
Function $f: \cR^d \rightarrow \cR$ is $L$-smooth if:
$f(x)\leq f(y)+ \langle\nabla f(y), x-y\rangle +\frac{L}{2}\|x-y\|^2 , \forall x,y \in \cR^d$,\\
and it is $L_C$-smooth if:
$$f(x)\leq f(y)+ \langle\nabla f(y), x-y\rangle +\frac{L_C}{2}\|x-y\|^2 , \forall x,y \in \cC.$$
\end{definition}
\begin{definition}
Function $f: \cR^d \rightarrow \cR$ is $\mu$-strongly convex if:
$f(x)\geq f(y)+ \langle\nabla f(y), x-y\rangle +\frac{\mu}{2}\|x-y\|^2 , \forall x,y \in \cR^d.$ If $\mu=0$ then function $f$ is a (non-strongly) convex function.
\end{definition}
Having presented the two main definitions for the class of problems that we are interested in, let us now present the working-set  $\Lew_k$ which contains all iterates produced in the $k$-th outer loop of Algorithm~\ref{algo:aisarah-alternative}.
\begin{definition}[Working-Set $\Lew_k$]\label{def:WandL1}
For any outer loop $k \geq 1$ in Algorithm~\ref{algo:aisarah-alternative}, starting at $\tilde w_{k-1}$ we define
\begin{equation}
\Lew_k:=\{ w \in \cR^d \ | \ \|\tilde w_{k-1}  - w\| 
 \leq m \cdot \alpha_{\max}^k \|v_0\|\}.
\end{equation}
\end{definition}
Note that the working-set $\Lew_k$ can be seen as a ball of all vectors $w$'s, which are not further away from $\tilde w_{k-1}$ than $m \cdot \alpha_{\max}^k \|v_0\|$. Here, recall that $m$ is the total number of iterations of an inner loop, $\alpha_{\max}^k$ is an upper bound of the step-size $\alpha_{t-1}$, $\forall t \in[m]$, and $\|v_0\|$ is simply the norm of the full gradient evaluated at the starting point $\tilde w_{k-1}$ in the outer loop.
\\
By combining Definition~\ref{smoothf} with the working-set $\Lew_k$, we are now ready to provide the main assumption used in our analysis.
\begin{assumption}
\label{assumptionLik}
Functions $f_i$, $i \in \n$, of problem~(\ref{MainProb}) are $L^i_{\Lew_k}$-smooth. Since we only focus on the working-set $\Lew_k$, we simply write $L^i_{k}$-smooth.
\end{assumption}
Let us denote $L_i$ the smoothness parameter of function $f_i$, $i \in \n$, in the domain $\cR^d$. Then, it is easy to see that $L^i_{k}\leq L_i,\;\forall i \in \n$. In addition, under Assumption~\ref{assumptionLik}, it holds that function $P$ is $\bar{L}_k$-smooth in the working-set $\Lew_k$, where  $\bar{L}_k=\frac{1}{n}\sum_{i=1}^n L^i_{k}.$
\\
As we will explain with more details in the next section for our theoretical results, we will assume that $\alpha_{\max}^k  \leq \frac{2}{L_k^{\max}}$, where  $L_k^{\max}= \max_{i \in [n]} L^i_{k}$. 

\subsection{Convergence Guarantees}
Now, we can derive the convergence rate of Algorithm \ref{algo:aisarah-alternative}. Here, we highlight that, all of our theoretical results can be applied to \textit{SARAH}. We also note that, some quantities involved in our results, such as $L_k$ and $L_k^{\max}$, are dependent upon the working set $\Lew_k$ (defined for each outer loop $k \geq 1$). Similar to~\cite{sarah17}, we start by presenting two important lemmas, serving as the foundation of our theory. 

The first lemma provides an upper bound on the quantity $\sum_{t=0}^{m}\Exp[\| \nabla P(w_{t})\|^2]$. Note that it does not require any convexity assumption.
\begin{lemma}
\label{Lemma_BoundGradient}
Fix a outer loop $k \geq 1$ and consider Algorithm~\ref{algo:aisarah} with $\alpha_{\max}^k \leq 1/\bar{L}_k$. Under Assumption~\ref{assumptionLik},
$
 \sum_{t=0}^{m}\Exp[\| \nabla P(w_{t})\|^2]
    \leq    \frac{2}{\alpha_{\min}^k} \Exp[P(w_0)-P(w^*)]  + \frac{\alpha_{\max}^k}{\alpha_{\min}^k}  \sum_{t=0}^{m}\Exp[\| \nabla P(w_{t})-v_{t}\|^2]$.
\end{lemma}
The second lemma provides an informative bound on the quantity $\Exp[\|\nabla P(w_t) - v_t\|^2]$. Note that it requires convexity of component functions $f_i$, $i \in [n]$.
\begin{lemma}
\label{Lemma_boundonDifference2}
Fix a outer loop $k \geq 1$ and consider Algorithm~\ref{algo:aisarah} with $\alpha_{\max}^k < 2/ L_k^{\max}$.  Suppose $f_i$ is convex for all $i \in \n$. Then, under Assumption~\ref{assumptionLik}, for any $t\geq1$ :
\begin{align*}
    \Exp[\|\nabla P(w_t) - v_t\|^2] \leq \left(\tfrac{\alpha_{\max}^k L_k^{\max}}{2-\alpha_{\max}^k L_k^{\max}}\right) \Exp[\|v_{0}\|^2].
\end{align*}
\end{lemma}
Equipped with the above lemmas, we can then present our main theorem and show the linear convergence of Algorithm~\ref{algo:aisarah-alternative} for solving strongly convex smooth problems.

\begin{theorem}\label{the:converge_1} Suppose that Assumption~\ref{assumptionLik} holds and $P$ is strongly convex with convex component functions $f_i$, $i \in \n$. 
Let us define $$\sigma_m^k=\tfrac{1}{\mu \alpha_{\min}^k (m+1)} + \tfrac{\alpha_{\max}^k}{\alpha_{\min}^k} \cdot \tfrac{\alpha_{\max}^k L_k^{\max}}{2-\alpha_{\max}^k L_k^{\max}},$$
and select $m$ and $\alpha_{\max}^k$ such that $\sigma_m^k<1$, $\forall k \geq1$. Then, Algorithm~\ref{algo:aisarah-alternative} converges as follows:
$$
 \Exp[\| \nabla P(\tilde{w}_{k})\|^2]
     \leq    \left(\textstyle{\prod}_{\ell=1}^k \sigma_m^\ell\right) \|\nabla P(\tilde{w}_{0})\|^2.
$$ 
\end{theorem}
As a corollary of our main theorem, it is easy to see that we can also obtain the convergence of \textit{SARAH}~\cite{sarah17}. Recall, from Remark~\ref{remarkSARAH}, that \textit{SARAH} can be seen as a special case of Algorithm \ref{algo:aisarah-alternative} if, for all outer loops, $\alpha_{\min}^k=\alpha_{\max}^k=\alpha$.
In this case, we can have
$$\sigma_m^k=\tfrac{1}{\mu \alpha (m+1)} + \tfrac{\alpha L_k^{\max}}{2-\alpha L_k^{\max}}.$$
If we further assume that all functions $f_i$, $i \in \n$, are $L$-smooth and do not take advantage of the local smoothness (in other words, do not use the working-set $\Lew_k$), then $L_k^{\max}=L$ for all $k \geq 1$.
Then, with these restrictions, we have
$$\sigma_m=\sigma_m^k=\tfrac{1}{\mu \alpha (m+1)} + \tfrac{\alpha L}{2-\alpha L}<1.$$
As a result, Theorem~\ref{the:converge_1} guarantees the following linear convergence:
$
 \Exp[\| \nabla P(\tilde{w}_{k})\|^2]
     \leq    (\sigma_m )^k \|\nabla P(\tilde{w}_{0})\|^2,$
which is exactly the convergence of classical \textit{SARAH} provided in~\cite{sarah17}.

\clearpage
\section{Technical Preliminaries \& Proofs of Main Results}
\label{AppendixProofs}
In this Chapter, we present technical details for the results in Chapter \ref{app:alternative}. Let us start by presenting some important technical lemmas that will be later used for our main proofs.
\subsection{Technical Preliminaries}

\begin{lemma}{\cite{nesterov2003introductory}} 
\label{Standard}
Suppose that function $f$ is convex and $L$-Smooth in $C \subseteq \R^n$. Then for any $w$, $w'$ $\in C$: 
\begin{equation}
\label{najkna}
     \left\langle\nabla f(w) - \nabla f(w'), (w - w') \right\rangle  \geq \frac{1}{L}  \|\nabla f(w) - \nabla f(w')\|^2.
\end{equation}
\end{lemma}

\begin{lemma}
\label{TechLemma}
Let Assumption~\ref{assumptionLik} hold for all functions $f_i$ of problem~(\ref{MainProb}). That is, let us assume that function $f_i$ is $L^i_{k}$-smooth $\forall i \in [n]$. Then, function  $P(w) \defeq \frac{1}{n}\sum_{i=1}^n f_i(w)$ is $\bar{L}_k$-smooth, where  $\bar{L}_k=\frac{1}{n}\sum_{i=1}^n L^i_{k}.$
\end{lemma}
\begin{proof}
For each function $f_i$, we have by definition of $L^i_{k}$-local smoothness, 
$$f_i(x)\leq f_i(y)+ \langle\nabla f_i(y), x-y\rangle +\frac{L^i_{k}}{2}\|x-y\|^2 , \forall x,y \in \Lew_k.$$
Summing through all $i's$ and dividing by $n$, we get
$$P(x)\leq P(y)+ \langle\nabla P(y), x-y\rangle +\frac{\bar{L}_k}{2}\|x-y\|^2 , \forall x,y \in \Lew_k.$$
\end{proof}

The next Lemma was first proposed in~\cite{sarah17}. We add it here with its proof for completeness and will use it later for our main theoretical result.
\begin{lemma} \cite{sarah17}
\label{lemmaTechnicalLam}
Consider $v_t$ defined in (\ref{TheVt}). Then for any $t\geq1$ in Algorithm~\ref{algo:aisarah}, it holds that:
\begin{align}
\label{lakjsndal}
\Exp[\|\nabla P(w_t) - v_t\|^2] = \sum_{j=1}^t \Exp[\|v_j - v_{j-1}\|^2] - \sum_{j=1}^t\Exp[\|\nabla P(w_j) - \nabla P(w_{j-1})\|^2].
\end{align}
\end{lemma}
\begin{proof}
Let $\Exp_j$ denote the expectation by conditioning on the information $w_0,w_1, \dots ,w_j $ as well as $v_0,v_1,\dots,v_{j-1}$. Then,
\begin{align*}
    \Exp_j[\|\nabla P(w_j) - v_j\|^2] &= 
    \Exp_j\left[\|\left(\nabla P(w_{j-1}) - v_{j-1}\right) 
    + \left(\nabla P(w_j) - \nabla P(w_{j-1})\right) 
    - (v_j - v_{j-1})\|^2\right]\\
    &=\Exp_j[\|\nabla P(w_{j-1}) - v_{j-1}\|^2]
    +\Exp_j[\|\nabla P(w_j) - \nabla P(w_{j-1})\|^2]
    + \Exp_j[\|v_j - v_{j-1}\|^2]\\
    &\quad + 2 \left(\nabla P(w_{j-1}) - v_{j-1}\right)^T\left(\nabla P(w_j)-\nabla P(w_{j-1})\right)\\
    &\quad - 2\left(\nabla P(w_{j-1}) - v_{j-1}\right)^T\Exp_j[v_j - v_{j-1}]\\
    &\quad - 2\left(\nabla P(w_j) - \nabla P(w_{j-1})\right)^T\Exp_j[v_j - v_{j-1}]\\
    &= \Exp_j[\|\nabla P(w_{j-1}) - v_{j-1}\|^2]
    - \Exp_j[\|\nabla P(w_j) - \nabla P(w_{j-1})\|^2]
    + \Exp_j[\|v_j - v_{j-1}\|^2],
\end{align*}
where the last equality follows from
$$\Exp_j[v_j - v_{j-1}]=\Exp_j[\nabla f_{i_j}(w_j) - \nabla  f_{i_j}(w_{j-1})]=\nabla P(w_j) - \nabla P(w_{j-1}).$$
By taking expectation in the above expression, using the tower property, and summing over $j=1,...,t$, we obtain
\begin{align*}
\Exp[\|\nabla P(w_t) - v_t\|^2] = \sum_{j=1}^t \Exp[\|v_j - v_{j-1}\|^2] - \sum_{j=1}^t\Exp[\|\nabla P(w_j) - \nabla P(w_{j-1})\|^2].
\end{align*}
\end{proof}

\subsection{Proofs of Lemmas and Theorems}
For simplicity of notation, we use $|S|=1$ in the following proofs, and a generalization to $|S|>1$ is straightforward. 
\subsubsection{Proof of Lemma~\ref{Lemma_BoundGradient}}
By Assumption~\ref{assumptionLik}, Lemma~\ref{TechLemma} and the update rule $w_t = w_{t-1} - \alpha_{t-1} v_{t-1}$ of Algorithm~\ref{algo:aisarah-alternative}, we obtain:
\begin{eqnarray*}
   P(w_{t}) &\leq & P(w_{t-1})- \alpha_{t-1} \langle \nabla P(w_{t-1}), v_{t-1} \rangle+ \frac{\bar{L}_k}{2} \alpha_{t-1}^2 \|v_{t-1}\|^2 \notag\\
   &= & P(w_{t-1})- \frac{\alpha_{t-1}}{2}\| \nabla P(w_{t-1})\|^2 + \frac{\alpha_{t-1}}{2}\| \nabla P(w_{t-1})-v_{t-1}\|^2 - \left(\frac{\alpha_{t-1}}{2} - \frac{\bar{L}_k}{2} \alpha_{t-1}^2 \right) \|v_{t-1}\|^2,
\end{eqnarray*}
where, in the equality above, we use the fact that $\langle a,b\rangle = \frac{1}{2}(\|a\|^2 + \|b\|^2 - \|a-b\|^2)$.

By rearranging and using the lower and upper bounds of the step-size $\alpha_{t-1}$ in the outer loop $k$ ($\alpha_{\min}^k\leq\alpha_{t-1}\leq\alpha_{\max}^k$), we get:
\begin{eqnarray*}
 \frac{\alpha_{\min}^k}{2} \| \nabla P(w_{t-1})\|^2 
   &\leq&  [P(w_{t-1})-P(w_{t}) ] + \frac{\alpha_{\max}^k}{2}\| \nabla P(w_{t-1})-v_{t-1}\|^2 - \frac{\alpha_{t-1} }{2} \left(1 - \bar{L}_k \alpha_{t-1} \right) \|v_{t-1}\|^2.
\end{eqnarray*}
By assuming that $\alpha_{\max}^k \leq \frac{1}{\bar{L}_k}$, it holds that $\alpha_{t-1} \leq \frac{1}{\bar{L}_k}$ and $ \left(1 - \bar{L}_k \alpha_{t-1} \right)\geq0$, $\forall t \in [m]$. Thus,
\begin{eqnarray*}
 \frac{\alpha_{\min}^k}{2} \| \nabla P(w_{t-1})\|^2 
   &\leq&  [P(w_{t-1})-P(w_{t}) ] + \frac{\alpha_{\max}^k}{2}\| \nabla P(w_{t-1})-v_{t-1}\|^2 - \frac{\alpha_{\min}^k}{2} \left(1 - \bar{L}_k \alpha_{\max}^k \right) \|v_{t-1}\|^2.
\end{eqnarray*}
By taking expectations and multiplying both sides with  $\frac{2}{\alpha_{\min}^k}:$
\begin{align*}
 \Exp[\| \nabla P(w_{t-1})\|^2]
   &\leq  \frac{2}{\alpha_{\min}^k} [\Exp[P(w_{t-1})]-\Exp[P(w_{t})] ] + \frac{\alpha_{\max}^k}{\alpha_{\min}^k} \Exp[\| \nabla P(w_{t-1})-v_{t-1}\|^2] -  \left(1 - \bar{L}_k \alpha_{\max}^k \right) \Exp[\|v_{t-1}\|^2]\\
   &\leq  \frac{2}{\alpha_{\min}^k} [\Exp[P(w_{t-1})]-\Exp[P(w_{t})] ] + \frac{\alpha_{\max}^k}{\alpha_{\min}^k}\Exp[\| \nabla P(w_{t-1})-v_{t-1}\|^2],
\end{align*}
where the last inequality holds as $\alpha^k_{\max} \leq \frac{1}{\bar{L}_k}$. Summing over $t=1,2,\dots, m+1$, we have
\begin{eqnarray*}
 \sum_{t=1}^{m+1}\Exp[\| \nabla P(w_{t-1})\|^2]
   &\leq&  \frac{2}{\alpha_{\min}^k}  \sum_{t=1}^{m+1} \Exp[P(w_{t-1})-P(w_{t})] + \frac{\alpha_{\max}^k}{\alpha_{\min}^k}  \sum_{t=1}^{m+1}\Exp[\| \nabla P(w_{t-1})-v_{t-1}\|^2]\notag\\
   &=&  \frac{2}{\alpha_{\min}^k} \Exp[P(w_{0})-P(w_{m+1})] + \frac{\alpha_{\max}^k}{\alpha_{\min}^k}  \sum_{t=1}^{m+1} \Exp[\| \nabla P(w_{t-1})-v_{t-1}\|^2\notag\\
    &\leq &  \frac{2}{\alpha_{\min}^k} \Exp[P(w_{0})-P(w_{*})] + \frac{\alpha_{\max}^k}{\alpha_{\min}^k}  \sum_{t=1}^{m+1}\Exp[\| \nabla P(w_{t-1})-v_{t-1}\|^2],
\end{eqnarray*}
where the last inequality holds since $w^*$ is the global minimizer of $P.$

The last expression can be equivalently written as:
\begin{eqnarray*}
 \sum_{t=0}^{m}\Exp[\| \nabla P(w_{t})\|^2]
    &\leq &  \frac{2}{\alpha_{\min}^k} \Exp[P(w_{0})-P(w_{*})] + \frac{\alpha_{\max}^k}{\alpha_{\min}^k}  \sum_{t=0}^{m}\Exp[\| \nabla P(w_{t})-v_{t}\|^2],
 \end{eqnarray*}
 which completes the proof.

\subsubsection{Proof of Lemma~\ref{Lemma_boundonDifference2}}
\begin{eqnarray*}
 \Exp_j\left[\|v_j\|^2\right] &\leq&   \Exp_j\left[\|v_{j-1} - \left(\nabla f_{i_j}(w_{j-1}) - \nabla f_{i_j}(w_j) \right)\|^2\right]\notag\\
   &=& \|v_{j-1}\|^2 +  \Exp_j\left[\|\nabla f_{i_j}(w_{j-1}) - \nabla f_{i_j}(w_j)\|^2\right] 
   \\&&\quad-  \Exp_j\left[\frac{2}{\alpha_{j-1}} \left\langle\nabla f_{i_t}(w_{j-1}) - \nabla f_{i_j}(w_j), w_{j-1} - w_j\right\rangle\right]\notag\\
   &\overset{(\ref{najkna})}{\leq} & \|v_{j-1}\|^2 +  \Exp_j\left[\|\nabla f_{i_j}(w_{j-1}) - \nabla f_{i_j}(w_j)\|^2\right] -  \Exp_j\left[\frac{2}{\alpha_{j-1} L^{i_j}_k} \|\nabla f_{i_j}(w_{j-1}) - \nabla f_{i_j}(w_j)\|^2\right].\notag\\
\end{eqnarray*}
For each outer loop $k$, it holds that $\alpha_{j-1}\leq \alpha_{\max}^k$ and  $L^{i}_k \leq L_k^{\max}$. Thus,
\begin{eqnarray*}
 \Exp_j[\|v_j\|^2] &\leq&  \|v_{j-1}\|^2 +  \Exp_j\left[\|\nabla f_{i_j}(w_{j-1}) - \nabla f_{i_j}(w_j)\|^2\right] - \frac{2}{\alpha_{\max}^k L_k^{\max}}  \Exp_j\left[\|\nabla f_{i_j}(w_{j-1}) - \nabla f_{i_j}(w_j)\|^2\right]\notag\\
&=&  \|v_{j-1}\|^2 + \left(1- \frac{2}{\alpha_{\max}^k L_k^{\max}}\right)  \Exp_j\left[\|\nabla f_{i_j}(w_{j-1}) - \nabla f_{i_j}(w_j)\|^2\right]\notag\\
&=&  \|v_{j-1}\|^2 + \left(1- \frac{2}{\alpha_{\max}^k L_k^{\max}}\right)  \Exp_j\left[\|v_j - v_{j-1}\|^2\right].\notag\\
\end{eqnarray*}

By rearranging, taking expectations again, and assuming that $\alpha_{\max}^k < 2/ L_k^{\max}$:
\begin{eqnarray*}
 \Exp[\|v_j - v_{j-1}\|^2] &\leq& \left(\frac{\alpha_{\max}^k L_k^{\max}}{2-\alpha_{\max}^k L_k^{\max}}\right) \left[\Exp[\|v_{j-1}\|^2] -  \Exp[\|v_j\|^2]\right].  \notag\\
\end{eqnarray*}

By summing the above inequality over $j=1,\dots, t$ $(t\geq1)$, we have:
\begin{eqnarray}
\label{ajskdnak}
 \sum_{j=1}^t  \Exp[\|v_j - v_{j-1}\|^2] &\leq& \left(\frac{\alpha_{\max}^k L_k^{\max}}{2-\alpha_{\max}^k L_k^{\max}}\right)  \sum_{j=1}^t\left[\|v_{j-1}\|^2 - \|v_j\|^2 \right]  \notag\\
 &\leq& \left(\frac{\alpha_{\max}^k L_k^{\max}}{2-\alpha_{\max}^k L_k^{\max}}\right) \left[\Exp[\|v_{0}\|^2] - \Exp[\|v_t\|^2]\right].
\end{eqnarray}

Now, by using Lemma~\ref{lemmaTechnicalLam}, we obtain:
\begin{eqnarray}
\Exp[\|\nabla P(w_t) - v_t\|^2] &\overset{(\ref{lakjsndal})}{\leq}& \sum_{j=1}^t \Exp\left[\|v_j - v_{j-1}\|^2\right]\notag\\& \overset{(\ref{ajskdnak})}{\leq} &\left(\frac{\alpha_{\max}^k L_k^{\max}}{2-\alpha_{\max}^k L_k^{\max}}\right) \left[\Exp[\|v_{0}\|^2] - \Exp[\|v_t\|^2] \right]\notag\\&\leq &\left(\frac{\alpha_{\max}^k L_k^{\max}}{2-\alpha_{\max}^k L_k^{\max}}\right) \Exp[\|v_{0}\|^2].
\end{eqnarray}

\subsubsection{Proof of Theorem~\ref{the:converge_1}}
\begin{proof}
Since $v_0=\nabla P(w_0)$ implies $\|\nabla P(w_0)-v_0 \|^2=0$, then by Lemma~\ref{Lemma_boundonDifference2}, we obtain:
\begin{equation}
\label{cnalkdnalkd}
\sum_{t=0}^m \Exp[\|\nabla P(w_t) - v_t\|^2]  \leq \left ( \frac{m \alpha_{\max}^k L_k^{\max}}{2-\alpha_{\max}^k L_k^{\max}}\right) \Exp[\|v_{0}\|^2].
\end{equation}
Combine this with Lemma~\ref{Lemma_BoundGradient}, we have that:
\begin{eqnarray}
\label{cnaoisal}
 \sum_{t=0}^{m}\Exp[\| \nabla P(w_{t})\|^2]
    &\leq &  \frac{2}{\alpha_{\min}^k} \Exp[P(w_{0})-P(w_{*})] + \frac{\alpha_{\max}^k}{\alpha_{\min}^k}  \sum_{t=0}^{m}\Exp[\| \nabla P(w_{t})-v_{t}\|^2]\notag\\
    &\overset{(\ref{cnalkdnalkd})}{\leq} &  \frac{2}{\alpha_{\min}^k} \Exp[P(w_{0})-P(w_{*})] + \frac{\alpha_{\max}^k}{\alpha_{\min}^k}  \left( \frac{m \alpha_{\max}^k L_k^{\max}}{2-\alpha_{\max}^k L_k^{\max}}\right) \Exp[\|v_{0}\|^2].
\end{eqnarray}
Since we are considering one outer iteration, with $k\geq1$, we have $v_0=\nabla P(w_0)=\nabla P(\tilde{w}_{k-1})$ and $\tilde{w}_{k}=w_t$, where $t$ is drawn uniformly at random from $\{0,1,\dots, m\}$. Therefore, the following holds,
\begin{eqnarray*}
 \Exp[\| \nabla P(\tilde{w}_{k})\|^2]&=&\frac{1}{m+1}\sum_{t=0}^{m}\Exp[\| \nabla P(w_{t})\|^2]\notag\\
    &\overset{(\ref{cnaoisal})}{\leq} &  \frac{2}{\alpha_{\min}^k (m+1)} \Exp[P(\tilde{w}_{k-1})-P(w_{*})] + \frac{\alpha_{\max}^k}{\alpha_{\min}^k}  \left( \frac{\alpha_{\max}^k L_k^{\max}}{2-\alpha_{\max}^k L_k^{\max}}\right) \Exp[\|\nabla P(\tilde{w}_{k-1})\|^2]\notag\\
    &\leq & \left(\frac{1}{\mu \alpha_{\min}^k (m+1)} + \frac{\alpha_{\max}^k}{\alpha_{\min}^k}  \left( \frac{\alpha_{\max}^k L_k^{\max}}{2-\alpha_{\max}^k L_k^{\max}}\right) \right)\Exp[\|\nabla P(\tilde{w}_{k-1})\|^2].\notag\\
\end{eqnarray*}
Let us use $\sigma_m^k=\frac{1}{\mu \alpha_{\min}^k (m+1)} + \frac{\alpha_{\max}^k}{\alpha_{\min}^k} \cdot \frac{\alpha_{\max}^k L_k^{\max}}{2-\alpha_{\max}^k L_k^{\max}}$,
then the above expression can be written as:
\begin{eqnarray*}
 \Exp[\| \nabla P(\tilde{w}_{k})\|^2]
    &\leq & \sigma_m^k \Exp[\|\nabla P(\tilde{w}_{k-1})\|^2].\notag\\
\end{eqnarray*}
By expanding the recurrence, we obtain:
\begin{eqnarray*}
 \Exp[\| \nabla P(\tilde{w}_{k})\|^2]
    &\leq &  \left(\prod_{\ell=1}^k \sigma_m^\ell\right) \|\nabla P(\tilde{w}_{0})\|^2.\notag\\
\end{eqnarray*}
This completes the proof.
\end{proof}

\subsection{On working set and iterates}
In Chapter \ref{app:alternative}, we propose the definition of working set
$$\mathcal W_k := \{w \in \mathcal R^d \; | \; \|\tilde w_{k-1} - w\| \leq m \cdot \alpha_{max}^k \|v_0\|\}.$$ 
We claim that all the iterates of the $k$th outer loop lie within the set.
\begin{proof}
$\textbf{(1)}$. First of all, we can show that, as long as $w_{t-1}, w_t \in \mathcal W_k$, $\|v_t\| \leq \|v_{t-1}\|$ deterministically.  By definition of $v_t$ and the assumption of the main theorem that $\alpha_{t-1} \leq \alpha_{max}^k \leq \frac{2}{L_{max}^k}$. Given $w_t, w_{t-1} \in \mathcal W_k$,
\begin{align*}
||v_t||^2 &= || v_{t-1} - (\nabla f_{i_t} (w_{t-1}) - \nabla f_{i_t} (w_t))||^2 \\
&= || v_{t-1} ||^2 + || \nabla f_{i_t} (w_{t-1}) - \nabla f_{i_t} (w_t) ||^2 - \frac{2}{\alpha_{t-1}} \langle \nabla f_{i_t}(w_{t-1}) - \nabla f_{i_t}(w_t), w_{t-1} - w_t \rangle\\
&\leq || v_{t-1} ||^2 + || \nabla f_{i_t} (w_{t-1}) - \nabla f_{i_t} (w_t) ||^2 - L_{max}^k \langle \nabla f_{i_t}(w_{t-1}) - \nabla f_{i_t}(w_t), w_{t-1} - w_t \rangle\\
& \leq || v_{t-1} ||^2. 
\end{align*}
The last inequality holds as $f_{i_t}$ is convex and smooth with parameter $L_{max}^k$ on $\mathcal W_k$.

$\textbf{(2)}$. Now, we can show the main results and note that $\tilde w_{k-1} = w_0$.

By induction, at $j=1$, $\|w_0 - w_1\| = \alpha_0 \|v_0\| \leq m \cdot \alpha_{max}^k \|v_0\|$, and thus $w_1 \in \mathcal W_k$. By $\textbf{(1)}$, $\|v_1\| \leq \|v_0\|$.

Assume for $1<j\leq t$, $w_j \in \mathcal W_k$, then by $\textbf{(1)}$, $\|v_j\| \leq \|v_{j-1}\|$ and thus $\|v_j\| \leq \|v_0\|$. 

Then, at $j=t+1$, we have
$\|w_0 - w_{t+1}\| = \|\sum_{i=1}^{t+1} \alpha_{i-1} v_{i-1}\| \leq \alpha_{max}^k \sum_{i=1}^{t+1}\|v_{i-1}\| \leq m \cdot \alpha_{max}^k \|v_0\| \Rightarrow w_{t+1} \in \mathcal W_k \text{ and } \|v_{t+1}\| \leq \|v_0\|.$

Therefore, by $\textbf{(1)}$ and $\textbf{(2)}$, we show the desired results.
\end{proof}

\end{document}